\documentclass{article}

\PassOptionsToPackage{numbers, sort&compress}{natbib}
  \usepackage[preprint]{neurips_2026}

\usepackage[utf8]{inputenc} % allow utf-8 input
\usepackage[T1]{fontenc}    % use 8-bit T1 fonts
\usepackage{hyperref}       % hyperlinks
\usepackage{url}            % simple URL typesetting
\usepackage{booktabs}       % professional-quality tables
\usepackage{amsfonts}       % blackboard math symbols
\usepackage{nicefrac}       % compact symbols for 1/2, etc.
\usepackage{microtype}      % microtypography
\usepackage{xcolor}         % colors
\usepackage[pdftex]{graphicx}

\usepackage{algorithm}
\usepackage{algpseudocode}
\usepackage{subcaption}
\usepackage{graphicx}

\algrenewcommand{\algorithmiccomment}[1]{\hfill$\triangleright$ #1}
\usepackage{amsmath}
\usepackage{amssymb}
\usepackage{bm}
\usepackage{amsthm}
\usepackage{wrapfig}
\usepackage{needspace}
\usepackage{placeins}
\usepackage{multirow}
\usepackage{multicol}
\usepackage{enumitem}
\usepackage[table]{xcolor}
\usepackage{booktabs}
\usepackage{multirow}
\usepackage{threeparttable}
\usepackage[symbol*]{footmisc}
\usepackage[skins,breakable]{tcolorbox}
\usepackage{fontawesome5}
\usepackage{pifont}
\newtcolorbox{takeawayA}{
    enhanced,
    arc=0pt, 
    boxrule=0pt, 
    borderline west={3pt}{0pt}{black!60}, 
    colback=black!4, 
    left=10pt, right=10pt, top=8pt, bottom=8pt,
    before skip=12pt, after skip=12pt 
}

\newtcolorbox{takeawayB}{
    enhanced,
    arc=3pt, 
    boxrule=0.5pt, 
    colframe=black!50,
    colback=black!2,
    left=10pt, right=10pt, top=8pt, bottom=8pt,
    before skip=12pt, after skip=12pt
}

\newtcolorbox{introquestion}{
    enhanced,
    arc=6pt, 
    boxrule=1.5pt, 
    colframe=black!65, 
    colback=black!3, 
    left=10pt, right=10pt, top=2pt, bottom=2pt, 
    halign=center, 
    fontupper=\itshape, 
    before skip=3pt, after skip=3pt
}

\newtcolorbox{fancyTakeawayA}{
    enhanced,
    arc=2pt, 
    boxrule=0.5pt,
    colframe=blue!30!black, 
    colback=blue!2, 
    borderline west={3pt}{0pt}{blue!40!black}, 
    left=12pt, right=12pt, top=2pt, bottom=2pt,
    before skip=5pt, after skip=5pt
}

\newtcolorbox{fancyTakeawayB}{
    enhanced,
    arc=4pt,
    boxrule=0.2pt, 
    colframe=black!15, 
    colback=white, 
    drop fuzzy shadow=black!10, 
    left=12pt, right=12pt, top=10pt, bottom=10pt,
    before skip=15pt, after skip=15pt
}

\newtcolorbox{fancyTakeawayC}{
    enhanced,
    frame hidden, 
    boxrule=0pt,
    colback=black!4, 
    borderline top={1.2pt}{0pt}{black!70}, 
    borderline bottom={1.2pt}{0pt}{black!70}, 
    left=12pt, right=12pt, top=10pt, bottom=10pt,
    before skip=15pt, after skip=15pt
}

\newtheorem{lemma}{Lemma}
\newtheorem{theorem}{Theorem}
\newtheorem{proposition}{Proposition}
\newtheorem{definition}{Definition}

\newtheorem{remark}{Remark}

\ifx\assumption\undefined
\newtheorem{assumption}{Assumption}

\fi
\theoremstyle{remark}

\newcommand{\M}{\mathcal{M}}
\renewcommand{\L}{\mathcal{L}}

\newcommand{\bW}{\mathbf{W}}
\newcommand{\bU}{\mathbf{U}}
\newcommand{\bV}{\mathbf{V}}

\newcommand{\bx}{\mathbf{x}}
\newcommand{\bX}{\mathbf{X}}
\newcommand{\bY}{\mathbf{Y}}

\newcommand{\bv}{\mathbf{v}}
\newcommand{\bu}{\mathbf{u}}

\newcommand{\bI}{\mathbf{I}}

\newcommand{\bG}{\mathbf{G}}

\newcommand{\bT}{\mathbf{T}}

\newcommand{\EE}{\mathbb{E}}

\newcommand{\cR}{\mathcal{R}}
\newcommand{\cM}{\mathcal{M}}

\newcommand{\ovec}{\operatorname{vec}}
\newcommand{\Tr}{\operatorname{Trace}}
\newcommand{\din}{D_{\textrm{in}}}
\newcommand{\dout}{D_{\textrm{out}}}
\newcommand{\diag}{\mathrm{diag}}
\newcommand{\proj}{\mathrm{proj}}

\newcommand{\T}{\textrm{T}}

\newcommand{\ourname}{\texttt{MACRO}}

\title{Demystifying Manifold Constraints in LLM Pre-training}

\author{
{Kang An$^{1}$\thanks{Equal Contribution.},
Jiaxiang Li$^{2}$\footnotemark[1],
Donald Goldfarb$^{3}$,
Shiqian Ma$^{4}$}
\\ \\
{\normalsize $^{1}$Rice University 
$^{2}$Virginia Tech $^{3}$Columbia University $^{4}$Johns Hopkins University} \\
\texttt{kang.an@rice.edu, jasonljx@vt.edu, goldfarb@columbia.edu, sqma@jhu.edu}
}

\begin{document}

\maketitle

\begin{abstract}
The recent success of matrix optimizers (e.g., Muon) suggests that specific normalization of momentum—such as orthogonalization and row-wise normalization—benefits both the stability and acceleration of LLM training. Consequently, several recent studies have suggested that weights should also be normalized, leading to a Riemannian optimization problem. While such constrained training frameworks demonstrate superior performance, the effects of explicitly constraining weights, and their interaction with existing stabilization mechanisms, remain less understood. To bridge this gap, we study manifold constrained training dynamics through activation scales, rotational dynamics, and the update-to-weight ratio. We propose a Riemannian spectral steepest descent optimizer called \ourname{}, alongside a radius selection principle to serve as our testbed. Our analysis and numerical experiments reveal that RMSNorm and manifold constraints serve overlapping roles, and that weight decay can be completely eliminated when manifold constraints are applied. By controlling the update-to-weight ratio, constrained training significantly alleviates update cancellation, empirically demonstrating that MACRO is robust to low-precision computation and competitive with existing algorithms for standard LLM pre-training. 
\end{abstract}

\section{Introduction}

Large language model (LLM) pre-training has benefited from matrix
optimizers that apply structured normalization to gradients or momentum,
including orthogonalization~\citep{jordan2024muon} and row-wise normalization
~\citep{ma2025swansgdnormalizationwhitening,
glentis2025minimalistoptimizerdesignllm}.
These have motivated a growing line of research work that restrict the weights on manifolds that serve the purpose of orthogonalization/normalization on weights ~\citep{zhao2026principledmuonmumathsfpensuring,jiang2026enhancingllmtrainingspectral,arefin2026learning,gu2026manorestrikingmanifoldoptimization,newhouse2025trainingtransformersenforcedlipschitz,xie2026controlledllmtrainingspectral,yang2026manifoldconstrainedsteepestdescent,dolatabadi2026numuonnuclearnormconstrainedmuoncompressible,kexuefm-11241,Su_Muon_Orthogonal,wen2025hyperball,bernstein2025manifolds}, leading to optimization problems over Riemannian manifolds. Promising numerical results have been reported for solving this Riemannian optimization problem. 

However, the underlying mechanisms driving the empirical gains of Riemannian optimization are still not well understood, and existing algorithms often suffer from at least one of three critical drawbacks. 
First, some algorithms \citep{wen2025hyperball} lack rigorous theoretical guarantees. Second, some other algorithms relying on double-loop formulations~\citep{xie2026controlledllmtrainingspectral} incur substantial computational overhead that is prohibitive for large-scale pre-training. Finally, existing Riemannian optimization algorithms with convergence guarantees, such as MCSD~\citep{yang2026manifoldconstrainedsteepestdescent}, lack empirical validation on modern LLM architectures, and it thus remains unclear whether they are suitable for large-scale pre-training.
A comprehensive discussion of additional related works is in Appendix~\ref{app:related_work}.

The role of weight constraints is not immediate in modern LLM
architectures. Some popular heuristics such as 
RMSNorm explicitly rescales activations
~\citep{zhang2019rootmeansquarelayer}, while weight decay changes
weight norms and rotational dynamics
~\citep{loshchilov2019decoupledweightdecayregularization,
kosson2024rotationalequilibriumweightdecay,
li2020reconcilingmoderndeeplearning}.
A change in weight scale can therefore be partly offset by normalization,
and some effects of explicit constraints may overlap with those of
weight decay.
At the same time, the size of an update relative to its weight matters
for whether it survives low-precision rounding
~\citep{micikevicius2018mixed,zamirai2020revisiting}.
These interactions make it difficult to infer the role of a weight constraint
from final validation loss alone, motivating our central question:

\begin{introquestion}
    How do manifold constraints shape the training dynamics of modern LLMs?
\end{introquestion}

In this paper, we consider constraining the weights on Frobenius spheres, spectral spheres, and oblique manifolds, 
and examine their connections to RMSNorm, weight decay, and
FP32/BF16 master-weight training. We inspect the training dynamics through three complementary
aspects: activation scales, rotation between consecutive weight iterates,
and the update-to-weight ratio.

Existing approaches for solving manifold constrained LLM pre-training include heuristic normalization
~\citep{wen2025hyperball}, double-loop constrained solvers
~\citep{xie2026controlledllmtrainingspectral}, and general Riemannian
solvers ~\citep{bernstein2025manifolds,
yang2026manifoldconstrainedsteepestdescent}.
To study different geometries within a common practical framework,
we introduce the \textbf{M}sign-\textbf{A}ligned
\textbf{C}onstrained \textbf{R}iemannian \textbf{O}ptimizer
(\ourname{}), a single-loop optimizer based on Riemannian spectral steepest
descent, and establish its convergence under standard assumptions.
We also derive norm-specific radius scalings under normalized Gaussian
initialization, accounting for the input and output dimensions of
each weight matrix.
Together, the optimizer and radius selection serve as our testbed.
Our main contributions are summarized below.

\begin{enumerate}[leftmargin=*]
    \item \textbf{Activation Dynamics Effect (Section~\ref{sec:rmsnorm_interplay}).} 
    We identify overlaps of functions between manifold constraints and
    RMSNorm.
    In an ablation architecture that removes the standard learnable
    RMSNorm layers following all linear layers, \ourname{} is more stable than Muon~\citep{jordan2024muon}. Under the standard architecture with learnable RMSNorm layers, the performance of \ourname{} is more robust to changes in the constraint radius because of the recaling from RMSNorm.
    \item \textbf{Rotational Dynamics Effect (Section \ref{sec:effects_on_wd}).} We theoretically and empirically analyze how manifold constraints alter rotational dynamics during LLM pre-training. Specifically, we demonstrate that manifold constraints intrinsically govern the rotation angle between consecutive updates. The Frobenius constraint imposes global static rotation angles, while adaptive subspace rotations are imposed under spectral constraints. This demonstrates that when manifold constraints are present, weight decay is not necessary.
    \item \textbf{Weight-Update-Ratio Regulation Effect (Section \ref{sec:training_precision}).} 
    When switching master-weight precision from FP32 to BF16 during training, the validation loss naturally increases. Our experiments find that the loss increase incurred by \ourname{} is substantially smaller than the one incurred by Muon. We connect the controlled relative learning rate to update cancellation through a Gaussian coordinate-threshold model to provide an explanation for such a difference.
    
    \item \textbf{Standard Pretraining Results (Section \ref{sec:justification}).} Across standard Qwen3-style architectures ranging from 240M to 1.2B parameters, we demonstrate that \ourname{}, which solves the manifold-constrained problem, consistently outperforms Muon, which solves the unconstrained problem. \ourname{} is competitive with other constrained optimizers, such as MuonH. \ourname{} also has substantially lower per-step costs than double-loop constrained optimizers like SSO.
    
\end{enumerate}

\section{Notation and Preliminaries}
\label{sec:background}
\vspace{-0.2cm}
{\bf Notation.} 
Throughout this paper, we use capital bold letters such as $\bW$ to represent matrices and $\bW_{i,:}$ ($\bW_{:,j}$) to denote the $i$-th row ($j$-th column) of $\bW$. Moreover,
$\|\bW\|_F$ denotes the Frobenius norm of $\bW$; 
$\|\bW\|_2$ denotes the spectral norm of a matrix $\bW$, and  $\|\bW\|_{*}$ denotes the nuclear norm of $\bW$, which is the dual norm of $\|\bW\|_2$; 
$\|\bx\|_2$ denotes the Euclidean norm of a vector $\bx$;
$\operatorname{msign}(\bW)$ denotes the Matrix Sign Operator of matrix $\bW$, i.e. $\operatorname{msign}(\bW)=\mathbf{U}\mathbf{V}^\top$, where $\mathbf{U}(\mathbf{V})$ is the left(right) singular vectors of $\bW$. 
For manifold $\M$, we use $\T_{\M}\bW$ to denote its tangent space at point $\bW$. Further, we denote by $\proj_{\T_{\M}\bW}(G)$ the projection of $G$ onto the tangent space to the point $\bW$ on the manifold, and use $\mathcal{P}_{\M}(X)$ to denote the projection (retraction) of $X$ onto the manifold $\M$. We use $\nabla\L$ to denote the Euclidean gradient of $\L$ and $\nabla_{\M}\L(\bW)$ to denote the Riemannian gradient of $\L$. For embedded submanifolds considered in this paper, Riemannian gradient is the projection of the Euclidean gradient to the tangent space, i.e., $\nabla_{\M}\L(\bW) = \proj_{\T_{\M}\bW}(\nabla\L(\bW) )$. 
{\bf Manifold Constraints.}
We consider four constraint choices for the weight matrix $\bW \in \mathbb{R}^{\dout \times \din}$:
\vspace{-0.2cm}
\begin{itemize}[leftmargin=*]
    \item \textbf{Frobenius Sphere~\citep{wen2025hyperball}}: $\mathcal{M}_F(R) = \{\bW: \|\bW\|_F = R\}$.
    \item \textbf{Spectral Sphere~\citep{xie2026controlledllmtrainingspectral}:} $\mathcal{M}_S(R) = \{\bW: \|\bW\|_2 = R\}$ \footnote{See the detailed discussions of spectral sphere in Appendix \ref{app:convergence_analysis} \& \ref{app:projection_retraction}.}
    \item \textbf{Input (resp. Output) Oblique Manifold~\citep{gu2026manorestrikingmanifoldoptimization}:} $\mathcal{M}_{O_{\textrm{in}}}(R_{\textrm{in}}) = \{\bW: \|\bW_{:, j}\|_2 = R_{\textrm{in}}, \forall j\}$ (resp. $\mathcal{M}_{O_{\textrm{out}}}(R_{\textrm{out}}) = \{\bW: \|\bW_{i, :}\|_2 = R_{\textrm{out}}, \forall i\}$.). 
\end{itemize}

For each of these manifolds, we use $\|\cdot\|_{\M}$ to denote the norm in their definition, i.e. for Frobenius sphere $\|\cdot\|_{\M}=\|\cdot\|_{F}$, for spectral sphere $\|\cdot\|_{\M}=\|\cdot\|_{2}$, for input oblique manifold, $\|\cdot\|_{\M}=\max_j\|\cdot_{:, j}\|=\|\cdot\|_{1\rightarrow 2}$ and for output oblique manifold $\|\cdot\|_{\M}=\max_i\|\cdot_{i, :}\|=\|\cdot\|_{2\rightarrow \infty}$.

\section{Constrained Training Algorithm and Radius Selection}
\label{sec:algorithm}

\subsection{\ourname{} Algorithm and Convergence}

We consider the following manifold constrained LLM training problem:
\begin{equation}
\label{eq:originalOptimizationProblem}
    \min_{\bW} \mathcal{L}(\bW) = \mathbb{E}_{\xi \sim \mathcal{D}}\left[\ell(\bW, \xi)\right] \text{ s.t. } \bW \in \mathcal{M},
\end{equation} 
where $\mathcal{M}$ is a manifold discussed in Sec \ref{sec:background}. To solve \eqref{eq:originalOptimizationProblem}, we propose a single-loop optimization framework, named \textbf{Msign-Aligned Constrained Riemannian Optimizer (\ourname{})}. \ourname{} is a provably convergent Riemannian algorithm
that maintains the spectral preconditioning benefits of the Muon optimizer. The complete description of \ourname{} is given in Alg \ref{alg:macro} and its 
main update is:
\begin{equation}
\label{eq:update}
O_t = \operatorname{msign}(\nabla_{\mathcal{M}}\mathcal{L}(\bW_t)), \quad \bW_{t+1} = \mathcal{P}_{\mathcal{M}}\left(\bW_t - \eta_t R\cdot O_t / \left(\|O_t\|_{\M}+\epsilon\right)\right),
\end{equation}
where $\eta_t$ is the learning rate and $\epsilon$ is a numerical stabilizer. This update scheme fundamentally differs from standard optimizers and the original Muon in the following key aspects:

\begin{itemize}[leftmargin=*]
    
    \item \textbf{Tangent Space Projection:} 
    We project it onto the tangent space 
    at the current weight $\bW_t$ to obtain the Riemannian gradient $\nabla_{\mathcal{M}}\mathcal{L}(\bW_t)$.  
    The exact computations are detailed in Appendix~\ref{app:projection_retraction}.
    
    \item \textbf{Riemannian Spectral Steepest Descent:} 
    \ourname{} evaluates the Linear Minimization Oracle (LMO) on the \textbf{Riemannian gradient}. 
    This extracts the optimal steepest descent direction reflecting the local geometry of the manifold.
    
    \item \textbf{Update-Weight Ratio Alignment:} After obtaining the update direction $O_t$, we rescale it by the constraint radius $R$. Together with the fixed weight norm, this controls the proposed update-weight ratio through $\eta_t$. 
    Section~\ref{sec:training_precision} analyzes this relative step.
    \item \textbf{Manifold Projection/Retraction:}  
    We apply a projection $\mathcal{P}_{\mathcal{M}}$ (acting as a retraction) at every iteration to map the weights back onto the target constraint manifold $\mathcal{M}$. 
    The exact computations are detailed in Appendix~\ref{app:projection_retraction}.
\end{itemize}

\begin{algorithm}[th]
\caption{Msign-Aligned Constrained Riemannian Optimizer (\ourname{}) % \jl{todo: wait for experiments then update the normalization and scaling step, also include the definition of $\|\cdot\|_{\M}$}
}
\label{alg:macro}
\begin{algorithmic}[1] 
    \State \textbf{Input:} Learning rate $\eta_t$, Manifold Constraint $\mathcal{M}$, Hyperparameters $\beta, R, T$
    %\State \textbf{Output:} Updated Weights $\bW_{t+1}$
    \For{$t = 0, 1, \ldots, T$}
        \State $G_t = \nabla \ell(\bW_t,\xi_t)$
        \vspace{0.05cm}
        \State $M_t = \beta M_{t-1} + (1-\beta) G_t$
        \vspace{0.05cm}
        \State $\Phi_t = \proj_{\T_{\M} \bW_t}(M_t)$ \Comment{Projection to the Tangent Space}\label{alg:macro_tangent_projection}
        \vspace{0.1cm}
        \State $O_t = \operatorname{msign}(\Phi_t)$ \Comment{Compute the steepest descent direction via LMO}
        \vspace{0.05cm}
        \State % $\widetilde{\nabla}_t = c \cdot \mathcal{P}_{\mathcal{M}}\left(O_t\right)$ 
        $\widetilde{\nabla}_t = R\cdot O_t / (\|O_t\|_{\M}+\epsilon)$ 
        \Comment{Normalization and Scaling}
        \vspace{0.05cm}
        \State $\bW_{t+1} = \mathcal{P}_{\mathcal{M}} \left( \bW_t - \eta_t \cdot \widetilde{\nabla}_t \right)$ \Comment{Descent Step and Manifold Projection}
    \EndFor
\end{algorithmic}
% \vspace{-0.2cm}
\end{algorithm}

{\bf Convergence analysis.} Two assumptions are needed for the convergence analysis of \ourname{}.
%\label{sec:convergence_analysis}
%We now analyze the convergence of \ourname{} (Algorithm \ref{alg:macro}). To this end, we utilize a similar framework as \citet{yang2026manifoldconstrainedsteepestdescent} and adopt the following assumption.

\begin{assumption}\label{assumption1}
    We assume that $\mathcal{M}$ is a compact $C^3$ manifold. %, also the projection to the manifold operator $\mathcal{P}_{\mathcal{M}}$ is $2$-Lipschitz continuous and $C^1$ in $B(\mathcal{M},r)$ for a certain radius $r$.
\end{assumption}
% Note that this assumption does not hold for the spectral sphere $\mathcal{M}_S(R)$,
% % $ = \{\bW \in \mathbb{R}^{\dout \times \din} : \|\bW\|_2 = R > 0\}$
% which is not even a manifold, but it holds for the set $\tilde{\mathcal{M}}_S(R)=\{\bW: \|\bW\|_2 = R,\sigma_1(\bW)\geq \sigma_2(\bW)+\epsilon\}$, Frobenius sphere and the two oblique manifolds discussed in Section \ref{sec:background} (see Appendix \ref{app:convergence_analysis} and \citet{yang2026manifoldconstrainedsteepestdescent} for a more detailed discussion). For the loss function $\mathcal{L}$, we also have the following assumption.

\begin{assumption}\label{assumption2}
    We assume the loss function $\mathcal{L}$ is lower bounded by $\mathcal{L}^*$, and has $L$-Lipschitz continuous gradient, and % in $B(\mathcal{M},r)$ for a certain radius $r$. 
    the stochastic loss $\ell(W,\xi)$ satisfies: (i) the stochastic gradient is unbiased, i.e. $\mathbb{E}_{\xi}\nabla\ell(W,\xi)=\nabla\mathcal{L}(W)$; (ii) the variance is bounded, i.e.,  $\mathbb{E}_{\xi}\|\nabla\ell(W,\xi) - \nabla\mathcal{L}(W)\|_F^2\leq\sigma^2$.
\end{assumption}

Under these two assumptions, we have the following theorem.
\begin{theorem}\label{thm:convergence}
    Suppose the manifold $\mathcal{M}$ and  the loss function $\mathcal{L}$ satisfy Assumptions \ref{assumption1} and \ref{assumption2}. 
    Denote $\Delta=\mathcal{L}(\bW_0)-\mathcal{L}^*$. By setting 
    $
    \beta=1-1/\sqrt{T},\ \eta_t=\eta=\Omega(\sqrt{{\Delta}/{(L T^{3/2})}}),
    $
    the sequence $\{\bW_t\}$ generated by Algorithm \ref{alg:macro} satisfies (constants related to dimension are omitted in $\mathcal{O}$):
    $$
    \min_{t=0,...,T-1} \mathbb{E}\|\nabla_{\mathcal{M}} \mathcal{L}(\bW_t)\|_{*}\leq \mathcal{O}\left({\sqrt{L\Delta}\sigma}{T^{-1/4}}\right).
    $$
\end{theorem}
This matches the $\mathcal{O}(T^{-1/4})$ rate for general nonconvex smooth stochastic optimization~\cite{arjevani2023lower}.

\subsection{Initialization-Based Radius Selection}
\vspace{-0.2cm}
\label{sec:activation_control}
\label{sec:radius_selection}

We now address the choice of radius for each weight constraint in \ourname{}. We consider a linear layer $\bY=\bX\bW^\top$, where $\bX\in\cR^{T\times n}$ and $\bW\in\cR^{p\times n}$, with $n:=\din$ and $p:=\dout$. Following \citet{yang2023spectral} and \citet{Su_MuP2}, our goal is measuring the activation of a liner layer given the initialized Gaussian initialization, and the input RMSNorm is $O(1)$.
% i.e.   
The following proposition summarizes the initialization result in Theorem~\ref{thm:gaussian_radius_choices} of Appendix~\ref{app:gaussian_radius_choices}.
\begin{proposition}[Initialization-Based Radius Selection]
\label{prop:radius_selection}
    Define $r_{\bX}:=\frac{\|\bX\|_F}{\sqrt{Tn}} = O(1)$ and $r_{\bY}:=\frac{\|\bX\bW^\top\|_F}{\sqrt{Tp}}$. Let $T,n,p\geq1$,  
    and $\bG\in\cR^{p\times n}$ be independent of $\bX \in\cR^{T\times n}$ with i.i.d.\ $\mathcal{N}(0,1)$ entries: 
    \vspace{-0.4cm}
    \begin{itemize}[leftmargin=*]
        \item \textbf{Frobenius Sphere:} If $\bW_0=R_{\textrm{fro}}\bG/\|\bG\|_F$ and $R_{\textrm{fro}}=r\sqrt{p}$, then
        $\mathbb{E}_{\bG}[r_{\bY}^2\mid\bX]=r^2r_{\bX}^2.$
        \item \textbf{Spectral Sphere:} If $\bW_0=R_{\textrm{spec}}\bG/\|\bG\|_2$ and $R_{\textrm{spec}}=r\sqrt{p/n}$, then $r_{\bY}\leq r r_{\bX}$ deterministically. For every $\varepsilon\in(0,1)$ and $a>0$, the lower bound $r_{\bY}\geq \frac{r\sqrt{1-\varepsilon}}{\sqrt{n/p}+1+a}\,r_{\bX}$        
        holds with probability at least $1-e^{-p\varepsilon^2/4}-e^{-a^2p/2}$ over $\bG$, conditional on the fixed $\bX$.
    \end{itemize}
    \vspace{-0.1cm}
    For $r,r_{\bX}=\Theta(1)$, the concentration bounds in Theorem~\ref{thm:gaussian_radius_choices} give $r_{\bY}=\Theta(1)$ with high probability. 
\end{proposition}
\begin{wraptable}{r}{0.38\textwidth}
  \vspace{-0.3cm}
  \centering
  \label{tab:radius_summary}
  \renewcommand{\arraystretch}{1.45}
  \begin{tabular}{lc}
    \toprule
    \textbf{Manifold} & \textbf{Radius Setting} \\
    \midrule
    Spectral Sphere & $r\sqrt{p / n}$ \\
    \rowcolor{black!5}
    Frobenius Sphere & $r\sqrt{p}$ \\
    Output Oblique & $r$ \\
    \rowcolor{black!5}
    Input Oblique & $r\sqrt{p / n}$ \\
    \bottomrule
  \end{tabular}
\end{wraptable}
According to the above proposition, the right table summarizes the radii chosen based on the Gaussian initialization. The hyperparameter $r=\Theta(1)$ is separate from the dimension-dependent scaling derived above; Table~\ref{tab:macro_radius_sensitivity} studies its empirical sensitivity. 

\textbf{Oblique vs. Frobenius Constraints.} 
For the oblique constraints, we match the total Frobenius norm to $R_{\textrm{fro}}$, using $\|\bW\|_F^2=p R_{\textrm{out}}^2$ or $\|\bW\|_F^2=n R_{\textrm{in}}^2$ on the respective manifolds. In the comparison in Appendix~\ref{app:oblique_redundancy}, the Frobenius constraint achieves lower validation loss than the oblique constraints across the tested learning rates, while their spectral and maximum row/column norm trajectories are similar (Figure~\ref{fig:frobvsbblique_ablation}). We therefore focus on the Frobenius and spectral spheres in the remaining experiments.

\noindent\textbf{Empirical Norm Geometry during the Training Process.}
The proposed Frobenius radius ensures that the single-layer amplification upper bound is preserved beyond initialization. Define 
\begin{equation}
\label{eq:empirical_kappa}
    \kappa(\bW)
    :=
    d\,\frac{\|\bW\|_2^2}{\|\bW\|_F^2},
    \qquad
    d := \min\{n,p\}.
\end{equation}
For any input $\bX$ and any weight matrix satisfying
$\|\bW\|_F = r_{\mathrm{fro}}\sqrt{p}$,
the output $\bY=\bX\bW^\top$ satisfies
$\label{eq:empirical_fro_amplification}
    r_{\bY}
    \leq
    \sqrt{\frac{n}{p}}\,\|\bW\|_2\,r_{\bX}
    =
    r_{\mathrm{fro}}
    \sqrt{\frac{n}{d}\,\kappa(\bW)}\,r_{\bX}.$
Figure~\ref{fig:frob_worstcase_stable_rank} in Appendix~\ref{app:practical_radius} shows that the reported values of \(\kappa(\bW)\) remain \(\Theta(1)\) during training. Combined with \(n/d = \Theta(1)\), we get \(r_{\bY} = \mathcal{O}(r_{\bX})\), i.e. the proposed Frobenius radius ensures that the single-layer amplification is bounded.

\section{Effects of Manifold Constraints on Training Dynamics}
So far we have discussed an explicit manifold-constrained optimizer and radius selection for LLM training. We now examine how manifold constraints interact with commonly used training techniques such as RMSNorm layers, weight decay, and low-precision training.
Additional experiments on intersection with learning-rate warmup are reported in the appendix (Table~\ref{tab:warmup_constraint_comparison}).
\subsection{Activation Dynamics Effect: Interplay with RMSNorm}
\label{sec:rmsnorm_interplay}
\begin{wrapfigure}{r}{0.35\textwidth}
\centering
\includegraphics[width=\linewidth]{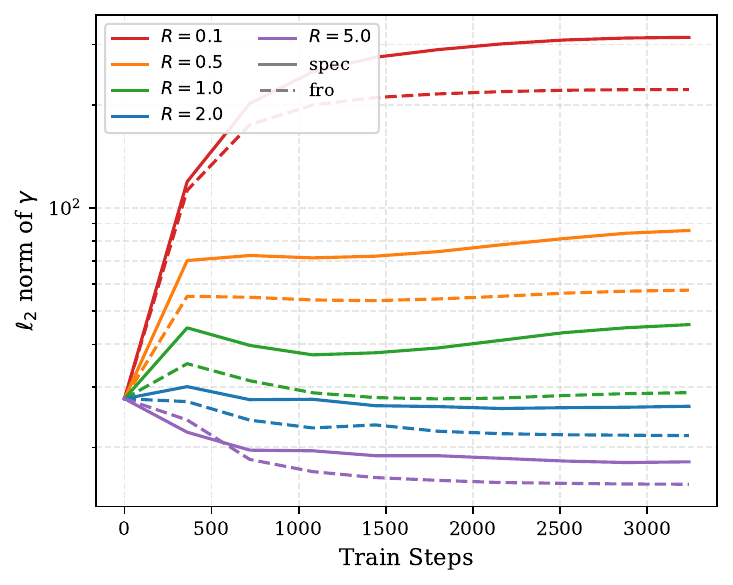}
\caption{\small $\ell_2$ norm of $\gamma$ during training.}
\label{fig:norm_vs_radius}
\end{wrapfigure}
In standard LLM architectures, normalization layers~\citep{ioffe2015batchnormalizationacceleratingdeep,Wu_2018_ECCV,ba2016layernormalization,salimans2016weightnormalizationsimplereparameterization,qiao2020microbatchtrainingbatchchannelnormalization} explicitly control the forward activation scale. For example, the RMSNorm layer~\citep{zhang2019rootmeansquarelayer} uses a learnable affine parameter, $\gamma$, to scale activations. 
However, Proposition~\ref{prop:radius_selection} shows that manifold constraints also determine the output scale at Gaussian initialization, suggesting an interaction with normalization during training. Figure~\ref{fig:norm_vs_radius} illustrates this interaction: larger constraint radii $R$ are associated with smaller learned norms of $\gamma$. This compensatory relation can buffer changes in weight scale. 
To examine the respective effects of RMSNorm and manifold constraints on training dynamics, we remove the standard learnable RMSNorm layers. The modified model retains a parameter-free RMSNorm ($\gamma = 1$) immediately after the attention block, QK-norm, and a parameter-free normalization after the initial embedding layer. It also uses Norm-Gated SwiGLU, which normalizes the Swish branch before the element-wise multiplication:
\begin{equation} \label{eq:norm-gated}
\setlength{\abovedisplayskip}{8pt}
\setlength{\belowdisplayskip}{8pt}
    \bY = \left(\operatorname{Swish}(\bX\bW_1^\top) /{\|\operatorname{Swish}(\bX\bW_1^\top)\|_{\textrm{RMS}}} \right)\odot (\bX \bW_2^\top) .
\end{equation}
\noindent
\begin{wraptable}{r}{0.4\textwidth}
    \vspace{-0.35cm}
    \centering
    \caption{Architecture Comparison}
    \label{tab:arc_changes}
    \renewcommand{\arraystretch}{1.25}
    \setlength{\tabcolsep}{3pt}
    \small
    \begin{tabular}{lcc}
    \toprule
    & \textbf{Standard} & \textbf{Ablation} \\
    \midrule
    Activation & SwiGLU & Eq.~\ref{eq:norm-gated} \\
    \rowcolor{black!5}
    Attn: Norm & \texttt{False} & \texttt{True} \\
    RMSNorm Layer & \texttt{True} & \texttt{False} \\
    \bottomrule
    \end{tabular}
\end{wraptable}
By training a 330M model with the ablation setting in Table~\ref{tab:arc_changes}, we evaluate how the optimizer stabilizes training after the standard learnable RMSNorm layers are removed. 
The empirical results using Muon, \ourname{}-fro, and \ourname{}-spec are summarized in Table~\ref{tab:ablation_qwen}. 

The results show that Muon diverges at $3\times10^{-2}$, while both \ourname{} variants remain stable across the tested learning rates. \ourname{}-spec achieves the lowest best-tuned validation loss (2.739), compared with 2.758 for \ourname{}-fro and 2.813 for Muon. This result highlights the normalizing effect of manifold constraints: while \ourname{} maintains robust stability at the tested learning rates, Muon may fail once the safeguard of learnable RMSNorm layers is removed.
\vspace{-0.3cm}
\begin{table}[htbp]
  \renewcommand{\arraystretch}{1.00} 
  \caption{Validation loss for QWEN3-like 330M pre-training in the ablation architecture without standard learnable RMSNorm layers.}
  \label{tab:ablation_qwen}
  \centering
  \begin{tabular}{lccccc}
    \toprule
    \multirow{2}{*}{\textbf{Optimizer}} & \multicolumn{5}{c}{\textbf{Learning Rate}} \\
    \cmidrule(r){2-6}
    & $3 \times 10^{-3}$ & $5 \times 10^{-3}$ & $7 \times 10^{-3}$ & $1 \times 10^{-2}$ & $3 \times 10^{-2}$ \\
    \midrule
    Muon (baseline)  & 2.912 & 2.849 & 2.823 & 2.813 & \texttt{NaN} \\
    \rowcolor{black!5} 
    \ourname{}-fro (Ours) & 2.901 & 2.825 & 2.787 & \textbf{2.758} & 2.781 \\
    \rowcolor{black!5} 
    \ourname{}-spec (Ours)& 2.820 & 2.768 & 2.747 & \textbf{2.739} & 2.819 \\
    \bottomrule
  \end{tabular}
\end{table}
\vspace{-0.15cm}
\textbf{Radius Sensitivity with Standard RMSNorm.} We now return to the standard RMSNorm architecture, distinct from the 330M ablation above. Table~\ref{tab:macro_radius_sensitivity} evaluates the radius multiplier $r$ at the 240M and 600M scales, using the dimension-dependent radii in Section~\ref{sec:radius_selection}. \ourname{} is relatively insensitive to the radius multiplier $r$ over the practical range $r\in[1,5]$, particularly at the 600M scale. At 600M, validation perplexity stays within 15.296--15.329 for \ourname{}-fro and 15.277--15.312 for \ourname{}-spec over this range. RMSNorm can buffer changes in weight scale, providing an explanation for this hyperparameter robustness consistent with the learned normalization rescales with radius in Figure~\ref{fig:norm_vs_radius}. 
\begin{table}[t]
    \renewcommand{\arraystretch}{1.00}
    \setlength{\tabcolsep}{11pt}
    \centering
    \small
    \caption{
      Radius sensitivity of \ourname{} under Frobenius and spectral
      constraints in the standard RMSNorm architecture at the 240M and 600M model scales.
      We report final validation perplexity.
    }
    \label{tab:macro_radius_sensitivity}

    \begin{threeparttable}
      \begin{tabular}{llcccc}
        \toprule
        \textbf{Constraint}
        & \textbf{Learning Rate}
        & $r=0.5$
        & $r=1$
        & $r=2$
        & $r=5$ \\
        \midrule

        \multicolumn{6}{c}{\textbf{240M Model}} \\
        \midrule

        \multirow{3}{*}{\ourname{}-fro}
        & 0.008 & 20.911 & 20.816 & 20.781 & 20.772 \\
        & 0.016 & 20.669 & 20.542 & 20.526 & 20.518 \\
        & 0.032 & 21.739 & 21.381 & 21.286 & 21.288 \\
        \midrule

        \multirow{3}{*}{\ourname{}-spec}
        & 0.004 & 21.500 & 21.129 & 21.062 & 21.044 \\
        & 0.008 & 21.103 & 20.678 & 20.535 & 20.457 \\
        & 0.016 & 21.385 & 20.981 & 20.833 & 20.738 \\
        \midrule

        \multicolumn{6}{c}{\textbf{600M Model}} \\
        \midrule

        \ourname{}-fro
        & 0.012 & 15.329 & 15.296 & 15.297 & 15.329 \\

        \ourname{}-spec
        & 0.008 & 15.434 & 15.312 & 15.277 & 15.288 \\
        \bottomrule
      \end{tabular}
    \end{threeparttable}
  \end{table}

\begin{fancyTakeawayA}
\textit{\textbf{\ding{43}\ Key Takeaway:}
Manifold constraints and RMSNorm have aligned roles: constraints improve stability in the ablation architecture, while performance with standard RMSNorm is relatively insensitive to the radius multiplier.}
\end{fancyTakeawayA}

\subsection{Rotational Dynamics Effect: Interplay with Weight Decay}
\label{sec:effects_on_wd}

Decoupled weight decay~\citep{loshchilov2019decoupledweightdecayregularization}
is standard practice in model training.
Recent works~\citep{dangelo2024needweightdecaymodern, xie2024implicitbiasadamwellinfty, kosson2024rotationalequilibriumweightdecay, defazio2025gradientsrapidlyincreasenear,li2020reconcilingmoderndeeplearning,chen2025muon} have studied its effects on weight scale and optimization
geometry.
Here, we focus on the rotational dynamics of the weights.
Weight decay balances radial shrinkage against the growth induced by
optimization updates, allowing weight norms and angular updates to
approach a rotational equilibrium~\citep{kosson2024rotationalequilibriumweightdecay,
li2020reconcilingmoderndeeplearning}. In other word, the late stage training with weight decay only changes the \textit{rotational angle} of consecutive updates, defined as $\theta_t = \angle(\bW_{t+1}, \bW_t)$~\citep{kosson2024rotationalequilibriumweightdecay}.
Such equilibrium is achieved by weight decay empirically as analyzed in \cite{kosson2024rotationalequilibriumweightdecay,li2020reconcilingmoderndeeplearning}. Because of this gradual process, standard optimizers undergo a long \textit{transient phase}. During this early phase, parameter scales and update angles remain unregulated. In contrast, explicit manifold constraints offer an explicit solution. By restricting the optimization trajectory, the rotation angle $\theta_t$ exactly from the first iteration. 

\begin{wrapfigure}{r}{0.38\textwidth}
\vspace{-0.65cm}
\centering
\includegraphics[width=\linewidth]{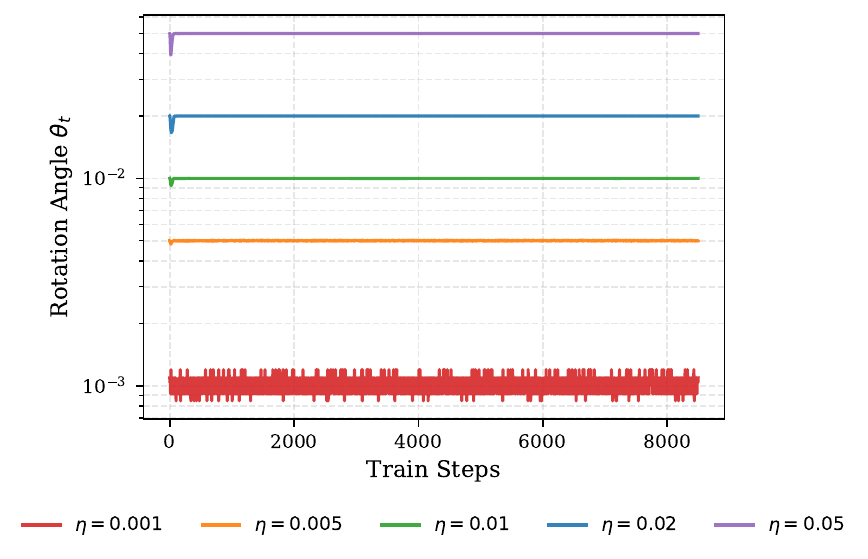}
\caption{\small $\theta_t$ under Frobenius Sphere.}
\label{fig:frob_dynamics}
\end{wrapfigure}
\paragraph{Static Rotational Equilibrium under the Frobenius Sphere.}
Manifold constraints also fundamentally alter the angular trajectory. Traditional weight decay requires a transient phase to reach rotational 
equilibrium~\citep{kosson2024rotationalequilibriumweightdecay}. Manifold constraints bypass this phase entirely and enforce a pure rotational state from the very first step. 
We explicitly quantify the rotational angle for the Frobenius sphere $\mathcal{M}_F(R)$. Define $\alpha_t := \cos\angle(\bW_t,\,O_t)$, substituting $\widetilde{\nabla}_t=(R/\|O_t\|_F)O_t$ in the update formula we have:
\vspace{0.18cm}
\begin{equation*}
    \langle \bW_t-\eta_t\widetilde{\nabla}_t,\bW_t\rangle_F = R^2(1-\eta_t \alpha_t),\qquad
    \|\bW_t-\eta_t\widetilde{\nabla}_t\|_F = R\sqrt{1-2\eta_t\alpha_t+\eta_t^2}.
\end{equation*}
Using this inner product constraint, we explicitly compute the rotational angle $\theta_t$ between consecutive steps:
\begin{equation*}
    \theta_t = \arccos\left(\frac{\langle \bW_{t+1}, \bW_t \rangle_F}{\|\bW_{t+1}\|_F\|\bW_t\|_F}\right) = \arccos \left(\frac{1-\eta_t \alpha_t}{\sqrt{1-2\eta_t \alpha_t+\eta_t^2}}\right)
\end{equation*}
Applying the second-order Taylor expansion for the cosine function, we obtain a simple geometric identity: $\theta_t \approx \eta_t\,\sqrt{1-\alpha_t^2} \approx \eta_t$ as $\alpha_t \approx 0$ (details in Appendix~\ref{app:fro-rotation}). This indicate that the Frobenius sphere constraint directly enforces a steady-state rotation while standard weight decay penalizes the norm to force a rotation. The full algebraic derivation is in Appendix~\ref{app:fro-rotation}.
 
To empirically validate this dynamic, we train a 330M Qwen3-like architecture using a constant learning rate 0.01 and various alignment ratios $c$ ranging from $[0.1, 5.0]$. As Figure~\ref{fig:frob_dynamics} shows, the 
Frobenius rotation angle remains flat,
confirming our derived identity $\theta_t^{\text{fro}} \approx \eta_t$. 

\paragraph{Adaptive Rotational Equilibrium under the Spectral Sphere.}
We now analyze the rotation angle $\theta_t$ for spectral sphere, which requires a different approach because spectral norm is not induced by an inner product. Instead of a single global rotation angle, the rotation is anisotropic: different singular subspaces rotate by different amounts. 
We focus on the maximum singular subspace, which is a natural choice because the spectral sphere explicitly regulates the maximum singular value. Let $\bu_t$ and $\bv_t$ denote the leading left and right singular vectors of $\bW_t$. We define the effective rotational angle for this primary direction as $\theta_t= \max\{\theta_u, \theta_v\}$, where $\theta_u = \arccos(|\bu_{t+1}^\top\bu_t|)$ and $\theta_v = \arccos(|\bv_{t+1}^\top\bv_t|)$. 
We treat our unprojected update step $- \eta_t \cdot \widetilde{\nabla}_t$ 
as the perturbation matrix, with magnitude $\|\Delta\bW_t\|_2 = \eta_t R$. By Wedin $\sin \Theta$ theorem~\citep{davis1970rotation,wedin1972perturbation}, we provide the following bound (see Appendix~\ref{app:spec-rotation} for proofs):
\begin{equation*}
\max\{\sin \theta_u, \sin \theta_v\} \le {\|\Delta \bW_t\|_2}/{\Delta_1} \approx {\eta_t R}/{\Delta_1}.
\end{equation*} 
Here we assume the leading singular value is unique, so the spectral gap $\Delta_1 := R - \sigma_2(\bW_{t+1}) > 0$. Substituting this definition into the above bound yields: $ \sin \theta_t \lesssim \frac{\eta_t R}{R - \sigma_2(\bW_{t+1})}$.
Applying the small-angle approximation ($\sin \theta \approx \theta$), we obtain:
\begin{equation}\label{eq:thetat_spec}
    \theta_t \lesssim \eta_t \left( \frac{R}{R - \sigma_2(\bW_{t+1})} \right).
\end{equation}

\begin{figure}[htbp]
    \centering
    \begin{minipage}{0.49\textwidth}
        \centering
        \includegraphics[width=\linewidth]{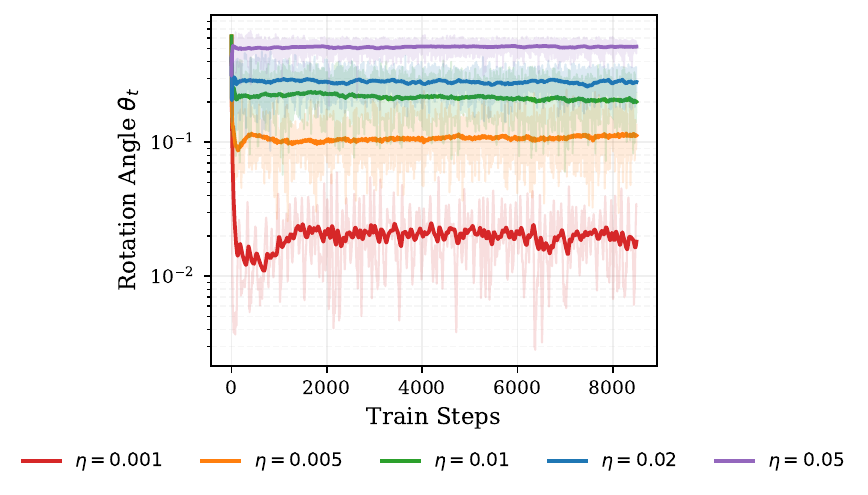}
        {\small (a) Principal rotation angle of the left leading singular vector $u_1$.}
    \end{minipage}\hfill
    \begin{minipage}{0.49\textwidth}
        \centering
        \includegraphics[width=\linewidth]{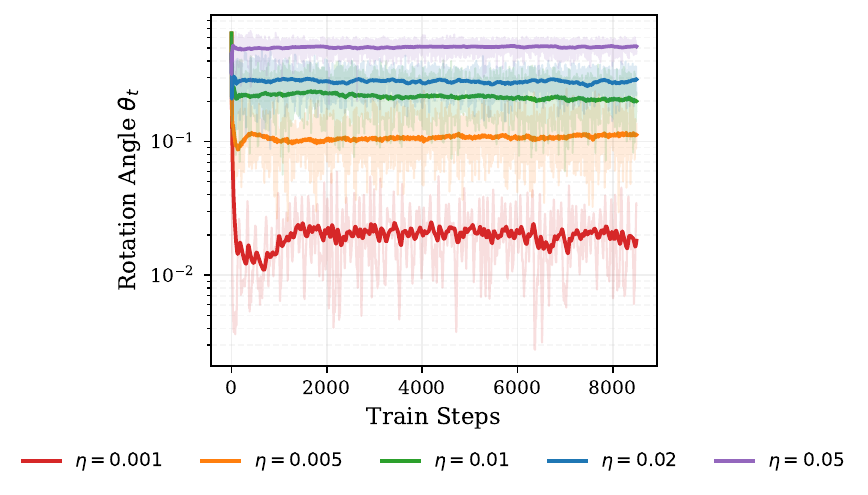}
        {\small (b) Principal rotation angle of the right leading singular vector $v_1$.}
    \end{minipage}
    \caption{Rotational angle under the spectral constraint: 
    The spectral rotation angles display high-frequency variance. Despite this,
    both the $u_1$ and $v_1$ trajectories remain well-balanced and stable.
    This dynamic behavior supports our Wedin bound analysis: the local spectral gap continuously modulates these adaptive rotations.}
    \label{fig:spec_rotation}
\end{figure}

Figure~\ref{fig:spec_rotation} visualizes the principal rotation angles under the spectral constraint. We used the same setting as
used in the Frobenius Sphere experiment 
depicted 
in Figure \ref{fig:frob_dynamics}.
Geometrically, the left and right singular vectors $u_1$ and $v_1$ rotate the output and input features, respectively. Despite this difference, both subspaces exhibit highly similar rotational trajectories, which indicates that the input and output feature rotations remain well-balanced during training. Furthermore, comparing the rotation angle under the spectral constraint (Eq. \eqref{eq:thetat_spec}) to the Frobenius case ($\theta_t \approx \eta_t$, Figure~\ref{fig:frob_dynamics}) reveals a key difference. The Frobenius constraint imposes a static, global rotation determined purely by constant hyperparameters. In contrast, the spectral constraint induces an \textit{adaptive} rotation. For the leading feature dimension, the training-dependent spectral gap $R - \sigma_2(\bW_{t+1})$ explicitly modulates the rotation angle. These distinct behaviors support our theoretical analysis. Overall, manifold constraints provide a more principled approach for rotational regulation than weight decay. We further study the effect of weight decay on \ourname{} and Muon. Table~\ref{tab:muon_lr_wd_coupling} in Appendix~\ref{app:weight_decay_experiments} shows that Muon is relatively sensitive to weight decay, and the optimal weight decay is also coupled with the learning rate. In contrast, Table~\ref{tab:wd_ablation} in Appendix~\ref{app:weight_decay_experiments} shows that adding weight decay to \ourname{} changes the validation loss by at most $0.001$ across both constraints and learning rates.

\begin{fancyTakeawayA}
\textit{\textbf{\ding{43}\ Key Takeaway:} (i) Weight decay is not needed if we use a manifold constraint optimizer. (ii) Unlike weight decay, manifold constraints regulate rotation angle from the very first iteration; while the Frobenius sphere enforces a static global rotation, the Spectral sphere induces an adaptive, anisotropic rotation.} 
\end{fancyTakeawayA}

\subsection{Weight-Update-Ratio Regulation Effect: Interplay with Training Precision}
\label{sec:training_precision}

Low-precision master weight training can lose small increments during rounding.
Standard remedies include FP32 master weights~\citep{micikevicius2018mixed},
stochastic rounding, and compensated
summation~\citep{zamirai2020revisiting,ozkara2025stochastic}.
Here, we examine how manifold constraints affect training under
reduced master weight precision. We train the 240M and 600M Qwen3 models using
TorchTitan~\citep{liang2024torchtitan}, comparing the \texttt{float32}
and \texttt{bfloat16} master weight precision settings.
The BF16 configuration reduces the storage precision of parameters,
gradients, and optimizer state. Figure~\ref{fig:master_weight_precision_comparison} shows the
validation trajectories.
Muon's final validation-loss increases
are $1.1470$ and $0.9394$ at the two scales, compared with
$0.0172$ and $0.0239$ for \ourname{}-fro and
$0.0317$ and $0.0476$ for \ourname{}-spec.
\ourname{} has substantially less performance degradation than Muon.
The experimental setting and tuned results can be found in Appendix~\ref{app:precision_experiments}.
\begin{figure*}[htbp]
    \centering
    \includegraphics[width=0.85\textwidth]{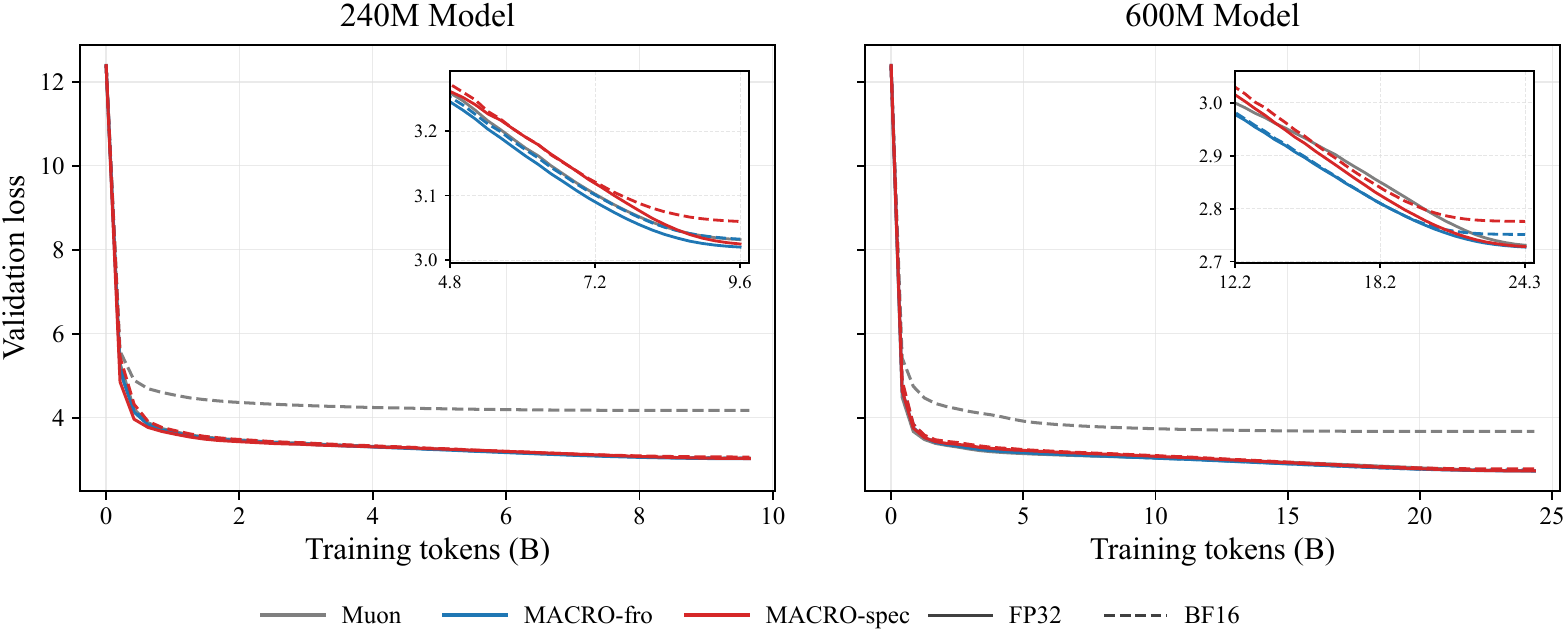}
    \vspace{-0.2cm}
    \caption{Validation-loss trajectories in the FP32 and BF16
    configurations for the 240M and 600M models.
    Insets enlarge differences during the late stage of training.}
    \vspace{-0.3cm}
    \label{fig:master_weight_precision_comparison}
\end{figure*}

\textbf{Update-Weight Ratio Alignment and update cancellation.}
The update-to-weight ratio~\citep{wen2025hyperball,xie2026controlledllmtrainingspectral} plays a crucial role here. To see this, we define $\rho_t:=\|\Delta\bW_t\|_F/{\|\bW_t\|_F}=\|\Delta\bW_t\|_{\mathrm{RMS}}/{\|\bW_t\|_{\mathrm{RMS}}}$.
For full-rank matrix-sign directions
, the three optimizers give
\begin{equation}
\label{eq:bf16_update_scale_comparison}
    \rho_t=
    \begin{cases}
       \eta_t\sqrt{d}/\|\bW_t\|_F,
        & \text{Muon},\\[3pt]
        \eta_t,
        & \text{\ourname{}-fro},\\[3pt]
        \eta_t\sqrt{\kappa(\bW_t)},
        & \text{\ourname{}-spec},
    \end{cases}
\end{equation}
where $d$ and $\kappa$ are defined in~\eqref{eq:empirical_kappa}.
The Frobenius ratio follows the learning rate directly; the spectral
ratio is of the same order when $\kappa$ remains bounded.

We now use these relative-update scales to analyze coordinate cancellation under BF16. For a scalar $w$ and an increment $\delta$, nearest rounding produces $Q_{\mathrm{BF16}}(w+\delta)$.
The increment is canceled if $w+\delta$ remains within the rounding
interval of $w$~\citep{zamirai2020revisiting}.
A normal BF16 value,
$|w|=(1.m)_22^E=(1.m_1,\dots,m_7)_22^E$, has a unit of least precision (ULP) of $2^{E-7}$ and
a half-ULP of $2^{E-8}$. Any increment $\delta$ will be canceled under nearest rounding if $|\delta| < 2^{E-8}$, i.e., the increment will be canceled if $\frac{|\delta|}{|w|}
< \frac{2^{E-8}}{(1.m)_2 2^E}=\frac{1}{(1.m)_2 2^8}$. Note that the quantity $\frac{1}{(1.m)_2 2^8}$ is always between $1/510$ and $1/256$.

\begin{wrapfigure}{r}{0.35\textwidth}
\vspace{-0.45cm}
\centering
\includegraphics[width=\linewidth]{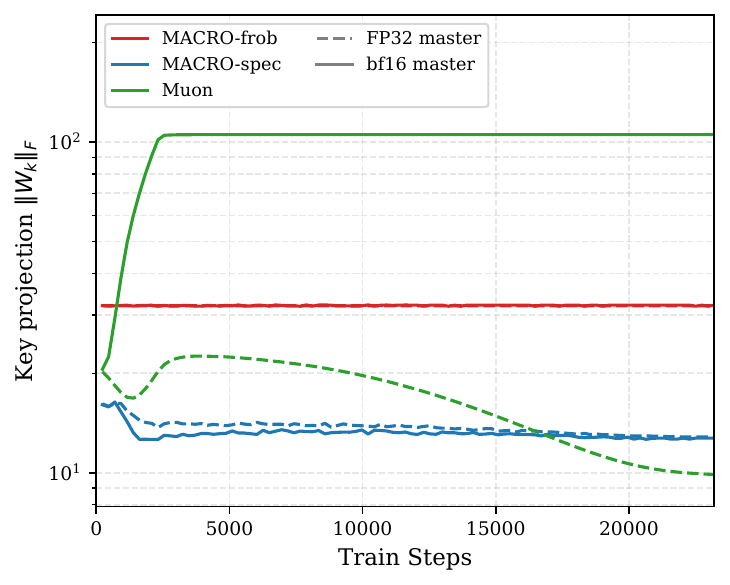}
\caption{\small The $\|W_t\|_{\text{F}}$ of Key projection parameter}
\label{fig:bf16_norm_dynamics}
\end{wrapfigure}
Now we inspect optimizer update step. Although the parameter is a matrix rather than a scalar, we can still use the update-weight ratio $\rho_t$ to approximate the average relative scale between the update and the weight. Without a manifold constraint, Muon may keep increasing the weight RMS norm during the early stage of training while the update RMS norm remains constant~\citep{liu2025muon}. Such misalignment will introduce update cancellation in Muon, leading to a much larger performance gap relative to full precision. 
Fig~\ref{fig:bf16_norm_dynamics} empirically shows that Muon fails to apply the weight-decay update because $\eta\lambda \ll 1/510$, which leads to weight accumulation from the very first step~\citep{yu2024collagelightweightlowprecisionstrategy}. After about 2500 steps, the Frobenius norm of the key-projection weight under Muon increases to approximately $129$. The update scale is $\eta |O_t|_{\mathrm{F}} = \Theta(0.01 \sqrt{\min\{n, p\}}) = \Theta(0.32)$, so the update-weight ratio is approximately $\Theta(2 \times 10^{-3}) \approx 1/ 510$. This tiny update-weight-ratio can easily lead to update cancellation, which corresponds to the performance degradation from a very early stage observed in Fig~\ref{fig:master_weight_precision_comparison}. On the other hand, both \ourname{} variants control the weight Frobenius norm and incur loss degradation only when $\eta$ decays below a certain threshold during the late phase of training. 

To theoretically explain the loss degradation during the late stage, we build on the observation that weight matrices can be modeled as zero-mean Gaussian random matrices with i.i.d. entries~\citep{huang2024billm, kedia2024transformers, babaoglu2026multi}. We then use two independent random variables, $w\sim\mathcal N(0,\|\bW_t\|_{\text{RMS}}^2)$ and
$\delta\sim\mathcal N(0,\|\Delta \bW_t\|_{\text{RMS}}^2)$, to simulate update cancellation. Under this assumption, we can estimate the retained fraction of update magnitude $R^{\text{fro}}_t$ under Frobenius Sphere as 
$$
R^{\text{fro}}_t:=\frac{\mathbb E|\delta\,\mathbf 1\{|\delta|>\tau|w|\}|}{\mathbb E|\delta|}
      =\frac{1}{\sqrt{1+(\tau/\rho)^2}} \overset{\rho_t = \eta_t}{=} \left[1+\left(\frac{\tau}{\eta_t}\right)^2\right]^{-1/2}.
$$
Under BF16 precision, let $\tau = 1/510 \approx 1.96 \times 10^{-3}$. The above equation shows that the retained fraction of update magnitude depends only on the relation between $\tau$ and $\eta_t$. Under the cosine-decay learning-rate schedule, once the learning rate decays to near or even below $\tau$, the retained fraction $R^{\text{fro}}_t$ also decreases, leading to update cancellation during the late stage of training in Fig~\ref{fig:master_weight_precision_comparison}. \ourname{}-spec exhibits a similar effect, as it also controls the Frobenius norm during training, as shown in Fig~\ref{fig:bf16_norm_dynamics}. The detailed proposition and discussion are provided in
Appendix~\ref{app:bf16_update_retention} and
Appendix~\ref{app:precision_experiments}.

\begin{fancyTakeawayA}
\textit{\textbf{\ding{43}\ Key Takeaway:}
\ourname{} exhibits less performance degradation than Muon in BF16 master-weight training. Controlling the update-to-weight ratio provides a scale-based interpretation by analyzing update cancellation as the weight norm and learning rate change.}
\end{fancyTakeawayA}

\section{Numerical Results}
\label{sec:justification}
To evaluate the empirical performance of \ourname{}, we compare optimizers on standard Qwen3 architectures at the 240M, 600M, and 1.2B scales using TorchTitan~\citep{liang2024torchtitan} training pipeline. Across all model sizes, we train on the token budget using 40 Token Per Parameters(TPP) and 5 seeds. Detailed experimental settings are reported in Table~\ref{tab:section5_setup} in Appendix~\ref{app:pretraining_experiments}. Alongside Muon~\citep{jordan2024muon}, we evaluate MuonH and the Spectral Sphere Optimization (SSO) algorithm. To compare both sphere constraints, we extend SSO to the Frobenius sphere and MuonH to the spectral sphere, yielding FSO and MuonH-spec, respectively (Appendix~\ref{app:constrained_baselines}). Table~\ref{tab:optimizer_comparison} summarizes training loss and validation perplexity (PPL) and Table~\ref{tab:optimizer_efficiency} in Appendix~\ref{app:precision_experiments} shows the optimizer efficiency. 
\begin{table}[htbp]
  \renewcommand{\arraystretch}{1.10} 
  \setlength{\tabcolsep}{8pt} 
  \caption{Training loss and validation perplexity (PPL) for Qwen3 models from 240M to 1.2B.}
  \centering
  \small
  \begin{threeparttable}
    \begin{tabular}{lcccccc}
      \toprule
      \multirow{2}{*}{\textbf{Optimizers}} 
      & \multicolumn{2}{c}{\textbf{240M}} 
      & \multicolumn{2}{c}{\textbf{600M}} 
      & \multicolumn{2}{c}{\textbf{1.2B}} \\
      \cmidrule(lr){2-3} \cmidrule(lr){4-5} \cmidrule(lr){6-7}
      & Train & Val PPL & Train & Val PPL & Train & Val PPL \\
      \midrule
      
      Muon (Baseline) & 3.042 & 20.72 ± 0.010 & 2.745 & 15.37 ± 0.006 & 2.577 & 12.99 ± 0.004 \\
      
      \rowcolor{black!5}
      MuonH-fro      & 3.033 & 20.55 ± 0.007 & \textbf{2.735} & \textbf{15.23 ± 0.006} & 2.569 & 12.92 ± 0.005 \\
      \rowcolor{black!5}
      MuonH-spec     & 3.037 & 20.63 ± 0.014 & 2.742 & 15.33 ± 0.006 & 2.566 & 12.88 ± 0.006 \\
      
      FSO\tnote{$\dagger$} & 3.040 & 20.98 ± 0.013 & 2.744 & 15.33 ± 0.014 & - & - \\
      SSO\tnote{$\dagger$}            & 3.035 & 20.59 ± 0.021 & 2.742 & 15.32 ± 0.007 & - & - \\
      
      \rowcolor{black!5}
      \ourname{}-fro (Ours) & 3.034 & 20.58 ± 0.011 & 2.739 & 15.29 ± 0.002 & \textbf{2.563} & \textbf{12.85 ± 0.006} \\
      \rowcolor{black!5}
      \ourname{}-spec (Ours)& \textbf{3.031} & \textbf{20.50 ± 0.011} & 2.740 & 15.30 ± 0.005 & 2.568 & 12.90 ± 0.009 \\
      \bottomrule
    \end{tabular}
    \begin{tablenotes}[flushleft]
    \scriptsize
    \linespread{0.9}\selectfont
    \item[$\dagger$] Scaling these double-loop algorithms requires custom CUDA kernels. Due to limited computational resources and the absence of such operator-level optimizations, we omit SSO and FSO for the 1.2B model.
    \end{tablenotes}
  \end{threeparttable}
  \label{tab:optimizer_comparison}
\end{table}

We draw conclusions from these results:
Firstly, both \ourname{} variants improve validation PPL over the Muon baseline at all three model scales. These results support the use of weight-norm constraints in standard architectures, consistent with the findings in~\cite{wen2025hyperball}. From the efficiency perspective (Table~\ref{tab:optimizer_efficiency}), MACRO-fro shows only a small efficiency gap compared with Muon. Its model FLOPs utilization (MFU) and end-to-end (E2E) time decrease by 0.62 percentage points and increase by 1.6\%, respectively.
Secondly, two constraints exhibit similar overall performance. For \ourname{}, the validation PPL gap between them is at most 0.08 across the three scales. The Frobenius constraint performs better among all three scales. On the other hand, the Frobenius constraint is consistently more efficient than the spectral constraint. Finally, Compared to MuonH, \ourname{} achieves comparable performance: \ourname{} attains the lowest validation PPL at 240M and 1.2B. In contrast of MuonH, \ourname{} retains the convergence guarantee established in Section~\ref{sec:algorithm}. 
Moreover, \ourname{} single-loop update achieves lower validation PPL than SSO and FSO at both evaluated scales, with the computational savings quantified in Table~\ref{tab:optimizer_efficiency}. Without specialized kernel optimizations, the MFU of \ourname{}-spec (35.15\%) is 2.45 times that of SSO (14.35\%), while \ourname{}-fro (39.06\%) is 1.89 times that of FSO (20.72\%).

\section{Conclusions}
In this paper, we study manifold constraints in LLM pre-training
and their interplay with normalization layers, weight decay,
and training precision.
Our analysis motivates norm-specific radii under Gaussian initialization
and reveals functional overlap between manifold constraints and RMSNorm. Our analysis also explains imply the regulation of weight rotation when using manifold constraints, and demonstrates that weight decay can be removed when manifold constraints are present. 
Relative-learning-rate control helps explain the smaller BF16
validation-loss penalty observed for \ourname{} compared with Muon. Furthermore, \ourname{} achieves validation losses comparable to, or better than, existing heuristic methods
while still maintains convergence guarantee from Riemannian optimization principles. 
\label{sec:conclusions}

\clearpage
\begin{ack}
Kang and Jiaxiang thank Anna Cai, Hao-Jun Michael Shi, Jihao Andreas Lin, Jialun Zhang, and Shagun Gupta for helpful discussions and valuable feedback.

\end{ack}

\bibliographystyle{unsrtnat} % abbrvnat or unsrtnat
\bibliography{reference}

%%%%%%%%%%%%%%%%%%%%%%%%%%%%%%%%%%%%%%%%%%%%%%%%%%%%%%%%%%%%

\appendix

\section{Related Works}
\label{app:related_work}

\paragraph{Muon optimizers and manifold-constrained variants.}
Muon~\citep{jordan2024muon,pethick2025trainingdeeplearningmodels} introduced a spectral-norm steepest descent update for weight matrices, achieving strong empirical performance in LLM pre-training. It is a parallel  work of matrix structured preconditioning methods such as Shampoo~\citep{eschenhagen2025purifyingshampooinvestigatingshampoos,gupta2018shampoo,10678134,shi2023distributeddataparallelpytorchimplementation}. \citet{bernstein2024old} provided theoretical grounding by interpreting Muon as an approximate Newton's method on the Stiefel manifold. Subsequent work has extended Muon in several directions: \citet{zhao2026principledmuonmumathsfpensuring} ensured spectral conditions under the $\mu$P parameterization; \citet{an2025asgo} proposed ASGO, which generalizes structured gradient preconditioning across different matrix decompositions; and \citet{dolatabadi2026numuonnuclearnormconstrainedmuoncompressible} introduced NuMuon, which replaces the spectral constraint with a nuclear norm constraint to encourage low-rank structure during training. Some recent studies on structure-aware optimizers \cite{ma2025swansgdnormalizationwhitening, glentis2025minimalistoptimizerdesignllm,11312186, 11312220} also consider row/column-wise normalization for the gradient. In parallel, several methods have explored explicit manifold constraints for LLM training. SSO~\citep{xie2026controlledllmtrainingspectral} formulates training on the spectral sphere with provable double-loop projections, while Mano~\citep{gu2026manorestrikingmanifoldoptimization} targets oblique manifolds and demonstrates competitive pre-training results. MuonH~\citep{wen2025hyperball} constrains weights to a Frobenius sphere via heuristic normalization after each Muon step. \citet{yang2026manifoldconstrainedsteepestdescent} propose Manifold Constrained
Steepest Descent (MCSD) from a general Riemannian optimization perspective.

\paragraph{Training deep learning models with constraints.}
Imposing structural constraints on neural network weights has a long history. Spectral normalization~\citep{miyato2018spectral} was originally proposed to stabilize GAN training by controlling the Lipschitz constant of discriminator layers. \citet{newhouse2025trainingtransformersenforcedlipschitz} extended this idea to enforce Lipschitz bounds in transformers. \citet{pethick2025trainingdeeplearningmodels} trained deep networks with norm-constrained linear minimization oracles (LMOs), offering a Frank-Wolfe perspective on constrained optimization. \citet{liu2021learningbyturning} proposed optimization directly on the Stiefel manifold by decomposing updates into rotations, and \citet{arefin2026learning} studied spectral constraints for feature learning in transformers. At the architecture level, nGPT~\citep{loshchilov2024ngpt} normalizes all representations to lie on a hypersphere, and Nemotron-Flash~\citep{fu2025nemotronflash} builds on this design for latency-efficient models. \citet{franke2025compactspaces} further explore approximately normalized transformers that learn in compact parameter spaces. \citet{bernstein2025manifolds} provides a conceptual framework connecting modular network design to manifold structure. While these methods each address specific constraint types, none systematically investigates how different manifold geometries interact with normalization layers and weight decay during LLM pre-training.

\paragraph{Normalization layers and their variants.}
Batch Normalization~\citep{ioffe2015batchnormalizationacceleratingdeep} introduced per-channel normalization of activations and was followed by Layer Normalization~\citep{ba2016layernormalization}, Group Normalization~\citep{Wu_2018_ECCV}, and RMSNorm~\citep{zhang2019rootmeansquarelayer}. In particular, RMSNorm is now standard in transformer-based LLMs. In the weight space, Weight Normalization~\citep{salimans2016weightnormalizationsimplereparameterization} decouples the magnitude and direction of weight vectors, and Weight Standardization~\citep{qiao2020microbatchtrainingbatchchannelnormalization} normalizes the rows of weight matrices to improve micro-batch training. \citet{wan2021spherical} analyze how the combination of batch normalization and weight decay induces spherical motion dynamics, revealing that the effective learning rate is governed by the angular update rather than the Euclidean step size. \citet{liu2018decoupled} propose Decoupled Networks, which separate the angular and radial components of inner products in linear layers. \citet{miyato2024akorn} introduce Artificial Kuramoto Oscillatory Neurons, whose phase-based representations naturally reside on the unit circle. \citet{karras2024analyzing} demonstrate that weight normalization of convolution layers stabilizes diffusion model training by controlling magnitude growth. \citet{owen2025variancecontrol} propose weight rescaling to control activation variance during LLM pre-training. Our work demonstrates that manifold constraints on weight matrices can formally subsume the stabilizing role of explicit normalization layers, enabling normalization-free linear blocks without sacrificing training stability.

\paragraph{Weight decay and its dynamics.}
Weight decay was originally introduced as $\ell_2$ regularization to improve generalization~\citep{zhang2018mechanismsweightdecayregularization}, and \citet{loshchilov2019decoupledweightdecayregularization} showed that decoupling it from the gradient (as in AdamW) leads to substantially different behavior than classical $\ell_2$ penalties. \citet{zhuang2022understanding} further clarify this distinction through proximal operator interpretations. A key insight from \citet{li2020reconcilingmoderndeeplearning} is that for scale-invariant networks, weight decay controls the \emph{intrinsic learning rate}---the ratio of the update norm to the weight norm---rather than acting as a direct regularizer.  \citet{li2019exponentiallearningrateschedule} further show that this mechanism induces an implicit exponential learning rate schedule. \citet{kosson2024rotationalequilibriumweightdecay} formalize this as \emph{rotational equilibrium}, showing that weight decay balances the angular velocity of weight updates across layers. \citet{heo2021adamp} propose AdamP, which explicitly removes the radial component of momentum updates to prevent the slowdown effect on scale-invariant weights. \citet{vanlaarhoven2017l2reg} analyze the interaction between $\ell_2$ regularization and batch/weight normalization, showing that their combined effect reduces to a modulation of the effective step size. \citet{kosson2024warmup} connect weight decay dynamics to learning rate warmup, demonstrating that careful weight initialization can reduce the need for warmup in GPT training. \citet{xie2024overlookedpitfallsweightdecay} identify gradient-norm pathologies caused by weight decay, and \citet{dangelo2024needweightdecaymodern} revisit why weight decay remains necessary in modern architectures, attributing its benefit primarily to scale invariance. \citet{xu2026widthscalingneuraloptimizers}, \citet{Su_MuP4} ,\citet{chen2025lionsecretlysolvesconstrained} and \citet{chen2025muon} show that weight decay together with bounded update
directions keeps the parameters in a uniformly bounded ball. Finally, \citet{defazio2025gradientsrapidlyincreasenear} show that gradient norms increase rapidly near the end of training, a phenomenon linked to the decay of weight norms. Our analysis shows that manifold constraints intrinsically govern the quantities that weight decay controls heuristically---locking the relative learning rate and rotation angle---and thus provide a principled geometric alternative to weight decay tuning.

\paragraph{Low-precision training and update cancellation.}
Low-precision training requires preserving optimizer updates that can be lost during parameter rounding. \citet{micikevicius2018mixed} retain FP32 master weights for mixed-precision training, while \citet{gupta2015limited} demonstrate the importance of stochastic rounding in low-precision fixed-point training. For BF16, \citet{zamirai2020revisiting} identify update cancellation as a key limitation and study stochastic rounding and Kahan summation. Gopher~\citep{rae2021gopher} uses stochastic rounding for BF16 parameters but reports a remaining gap relative to mixed precision, and \citet{ozkara2025stochastic} provide further theoretical and empirical analysis of stochastic rounding for LLMs. Collage~\citep{yu2024collagelightweightlowprecisionstrategy} compensates numerical errors using multi-component low-precision representations and explicitly discusses the loss of small weight-decay updates. More recently, ECO~\citep{nikdan2026eco} incorporates weight-quantization errors into optimizer momentum, whereas M+Adam~\citep{liang2026madam} combines additive and multiplicative updates to address precision-dependent stagnation. \citet{shang2026auditing} examines coordinatewise update invisibility and the limitations of forecasts based on high-precision reference traces. Related BF16 update sparsity is exploited for distributed RL communication by PULSE~\citep{miahi2026pulse}. Our work focuses on how manifold constraints regulate the update-to-weight ratio and thereby affect susceptibility to BF16 update cancellation. We use a Gaussian coordinate-threshold model to interpret the observed precision gap, without introducing stochastic rounding or explicit error compensation.

\section{Algorithm Supplementary Materials}
\subsection{Tangent Space Projection and Projectional Retraction}
\label{app:projection_retraction}
\paragraph{Tangent Space Projection.}
A critical component of our single-loop algorithm is the orthogonal projection of the momentum $M_t$ (or gradient) onto the tangent space of the constraint manifold. We define the projected vector $\Phi^*_t$ as the solution to the following proximity problem:
$$
\Phi^*_t = \arg\min_{\Phi \in \T_\mathcal{M}{\bW_t}} \|\Phi - M_t\|_F^2.
$$
Given that the constraints we consider can be locally characterized by a scalar function $h(W) = c$, the tangent space is a linear subspace. This yields a simple closed-form solution:
$$
\Phi^*_t = M_t - \frac{\langle M_t, \Theta \rangle}{\|\Theta\|_F^2}\Theta,
$$
where $\Theta = \nabla_W h(W)$ represents the normal vector to the constraint surface. This projection is highly computationally efficient, as it typically involves only fundamental matrix-vector or inner products. The specific projection operators for our four constraints are detailed below:
\begin{itemize}[leftmargin=*]
    \item Frobenius Sphere Constraint ($\mathcal{M}_{F}$): The normal vector is the weight matrix itself, $\Theta = W$. The projection is given by:
    $$
    \Phi^*_t = M_t - \frac{\langle M_t, W\rangle}{\|W\|_F^2}W.
    $$
    \item Spectral Sphere Constraint ($\mathcal{M}_{S}$): The normal vector is $\Theta = u_1v_1^\top$, where $u_1$ and $v_1$ are the leading singular vectors of weight matrices $\bW$. Since $\|\Theta\|_F^2 = \operatorname{Trace}(v_1u_1^\top u_1v_1^\top) = 1$, the projection simplifies to:
    $$
    \Phi^*_t = M_t - \operatorname{Trace}(M_t^\top u_1v_1^\top)u_1v_1^\top = M_t - (u_1^\top M_t v_1) u_1v_1^\top.
    $$
    \item Output Oblique Manifold Constraint ($\mathcal{M}_{O_{\textrm{out}}}$): Since the constraint is applied to each row independently, we project each row vector $M_t[i, :]$ onto the orthogonal complement of the corresponding weight row $W_{i, :}$:
    $$
    \Phi_t^*[i, :] = M_t[i, :] - \frac{M_t[i, :] W_{i, :}^\top}{\|W_{i, :}\|_2^2} W_{i, :}.
    $$
    \item Input Oblique Manifold Constraint ($\mathcal{M}_{O_{\textrm{in}}}$): Symmetrically, the constraint applies to the columns. We project each column vector $M_t[:, j]$ independently:
    $$
    \Phi_t^*[:, j] = M_t[:, j] - \frac{W_{:, j}^\top M_t[:, j]}{\|W_{:, j}\|_2^2} W_{:, j}.
    $$
\end{itemize}

\paragraph{Projectional Retraction.}
The second critical component of our single-loop algorithm is the retraction of the update $\tilde{\bW}_t:=\bW_t - \eta_t\tilde{\nabla}_t$ to the constraint manifold. Again we use the projectional retraction:
$$
\bW_{t+1} = \mathcal{P}_{\M}(\tilde{\bW}_t) := \arg\min_{\bW \in \mathcal{M}} \|\bW - \tilde{\bW}_t \|_{F}^2.
$$
The specific projectional retraction operators for our four constraints are detailed below:
\begin{itemize}[leftmargin=*]
    \item Frobenius Sphere Constraint ($\mathcal{M}_{F}(R)$): The projection is given by:
    $$
    \bW_{t+1} = \mathcal{P}_{\M}(\tilde{\bW}_t) = R\cdot \tilde{\bW}_t / \|\tilde{\bW}_t\|_F.
    $$
    \item Spectral Sphere Constraint ($\mathcal{M}_{S}(R)$): The projection is given by:
    $$
    \bW_{t+1} = \mathcal{P}_{\M}(\tilde{\bW}_t) = U \Sigma' V^\top,
    $$
    where $U \Sigma V^\top = \tilde{\bW}_t$ is the singular value decomposition of $\tilde{\bW}_t$ and $\Sigma'$ is in the following structure: if $\Sigma$ has elements that are greater than $R$, then clip them to $R$; otherwise if all elements in $\Sigma$ are smaller than $R$, then turn the largest element of $\Sigma$ to $R$. However, this operation is quite expensive to compute in practice, and we \textbf{use the following approximate operation in all our experiments}:
    $$
    \bW_{t+1} = R \cdot \tilde{\bW}_t / \|\tilde{\bW}_t\|_2
    $$
    to approximately project the matrix $\tilde{\bW}_t$ back to the spectral manifold.
    
    \item Output Oblique Manifold Constraint ($\mathcal{M}_{O_{\textrm{out}}}$): Since the constraint is applied to each row independently, the projectional retraction is simply divide each row with its own norm, and multiply with $R$.
    
    \item Input Oblique Manifold Constraint ($\mathcal{M}_{O_{\textrm{in}}}$): Symmetrically, the constraint applies to the columns. The projectional retraction is simply divide each column with its own norm, and multiply with $R$.
\end{itemize}

\subsection{Frobenius Sphere Optimization and MuonH-spec}
\label{app:constrained_baselines}
Inspired by \citep{xie2026controlledllmtrainingspectral}, we introduce the Frobenius Sphere Optimization(FSO) by considering the following subproblem: 
\begin{equation}\label{sso-subproblem}
\begin{array}{cl}
\max_{\boldsymbol{\Phi}} & \langle\bG, \boldsymbol{\Phi}\rangle \\
\text{s.t.} & \|\boldsymbol{\Phi}\|_2=1 , \quad \|\bW-\eta R \boldsymbol{\Phi}\|_F = \|\bW\|_F = R, 
\end{array}
\end{equation}
where $\bG$ is the stochastic gradient and $\bW$ is the weight matrix. Using the first-order Taylor's expansion of the Frobenius norm, we have:
$$
\|\bW-\eta R\boldsymbol{\Phi}\|_F = \|\bW\|_F - \eta \langle \bW, \boldsymbol{\Phi}\rangle + \mathcal{O}(\eta^2).
$$

Using the first order approximation, problem \eqref{sso-subproblem} reduces to:
\begin{equation}
\begin{array}{cl}
    \max_{\boldsymbol{\Phi}} \langle \bG, \boldsymbol{\Phi} \rangle \quad \text{s.t.} \quad \|\boldsymbol{\Phi}\|_2 = 1, \quad \langle \boldsymbol{\bW}, \boldsymbol{\Phi} \rangle = 0.
\end{array}
\label{eq:subproblem}
\end{equation}

To solve Eq.~\ref{eq:subproblem}, we consider its Lagrangian relaxation:
$$
\begin{array}{cl}
\max_{\boldsymbol{\Phi}} \langle \bG + \lambda \bW, \boldsymbol{\Phi} \rangle \quad \text{s.t.} \quad \|\boldsymbol{\Phi}\|_2 = 1.
\end{array}
$$
Following Theorem A.1 in \citep{xie2026controlledllmtrainingspectral}, the closed-form solution to this relaxed problem is:
$$
\boldsymbol{\Phi}(\lambda) = \arg\max_{\|\boldsymbol{\Phi}\|_2 = 1} \langle \bG + \lambda \bW , \boldsymbol{\Phi}\rangle = \operatorname{msign}(\bG + \lambda \bW).
$$

To satisfy the tangent space feasibility condition in Eq.~\ref{eq:subproblem}, we require $\langle \bW, \boldsymbol{\Phi}(\lambda^*) \rangle = 0$. Theorem A.2 in \cite{xie2026controlledllmtrainingspectral} establishes that the multiplier $\lambda$ is monotonic under the spectral sphere constraint. We show that this monotonicity also holds for the Frobenius sphere constraint:

\begin{lemma}
\label{lemma:frob_monotonicity}
The function
$$h(\lambda) = \langle \bW, \boldsymbol{\Phi}(\lambda)\rangle = \langle \bW, \operatorname{msign}(\bG+\lambda \bW)\rangle$$
is monotonically non-decreasing with respect to $\lambda$.
Moreover, there exists a root $\lambda^* \in \mathbb{R}$ such that $h(\lambda^*) = 0$, and any such root rigorously satisfies the bound $|\lambda^*| \leq \frac{2 \|\bG\|_*}{\|\bW\|_*}$.
\end{lemma}

\begin{proof}
We first prove that $h(\lambda)$ is a monotonic non-decreasing function. By the Theorem A.1 in \cite{xie2026controlledllmtrainingspectral}, we have:
\begin{equation}
\operatorname{msign}(\bX) = \operatorname{argmax}_{\|\mathbf{T}\|_2 = 1} \langle \bX, \mathbf{T} \rangle. \label{eq:max_def}
\end{equation}
Consider two arbitrary scalar values $\lambda_2 > \lambda_1$. By the definition of $\boldsymbol{\Phi}(\lambda) = \operatorname{msign}(\bG + \lambda \bW)$ as the precise maximizer of the inner product $\langle \bG + \lambda \bW, \cdot \rangle$ over the unit spectral ball, we establish the following two cross-inequalities:
\begin{align}
\langle \bG + \lambda_1 \bW, \boldsymbol{\Phi}(\lambda_1) \rangle &\geq \langle \bG + \lambda_1 \bW, \boldsymbol{\Phi}(\lambda_2) \rangle \nonumber \\
\Rightarrow \quad \langle \bG, \boldsymbol{\Phi}(\lambda_1) \rangle + \lambda_1 \langle \bW, \boldsymbol{\Phi}(\lambda_1) \rangle &\geq \langle \bG, \boldsymbol{\Phi}(\lambda_2) \rangle + \lambda_1 \langle \bW, \boldsymbol{\Phi}(\lambda_2) \rangle, \label{eq:ineq1} \\
\langle \bG + \lambda_2 \bW, \boldsymbol{\Phi}(\lambda_2) \rangle &\geq \langle \bG + \lambda_2 \bW, \boldsymbol{\Phi}(\lambda_1) \rangle \nonumber \\
\Rightarrow \quad \langle \bG, \boldsymbol{\Phi}(\lambda_2) \rangle + \lambda_2 \langle \bW, \boldsymbol{\Phi}(\lambda_2) \rangle &\geq \langle \bG, \boldsymbol{\Phi}(\lambda_1) \rangle + \lambda_2 \langle \bW, \boldsymbol{\Phi}(\lambda_1) \rangle. \label{eq:ineq2}
\end{align}
Summing Equation (\ref{eq:ineq1}) and Equation (\ref{eq:ineq2}), the cross-terms involving $\bG$ naturally cancel out, allowing us to derive:
\begin{align}
\lambda_1 \langle \bW, \boldsymbol{\Phi}(\lambda_1) \rangle + \lambda_2 \langle \bW, \boldsymbol{\Phi}(\lambda_2) \rangle &\geq \lambda_1 \langle \bW, \boldsymbol{\Phi}(\lambda_2) \rangle + \lambda_2 \langle \bW, \boldsymbol{\Phi}(\lambda_1) \rangle \nonumber \\
\Rightarrow \quad (\lambda_2 - \lambda_1) \langle \bW, \boldsymbol{\Phi}(\lambda_2) \rangle &\geq (\lambda_2 - \lambda_1) \langle \bW, \boldsymbol{\Phi}(\lambda_1) \rangle. \label{eq:monofinal}
\end{align}
Given that $\lambda_2 - \lambda_1 > 0$, we can divide both sides by $\lambda_2 - \lambda_1$ to obtain $\langle \bW, \boldsymbol{\Phi}(\lambda_2) \rangle \geq \langle \bW, \boldsymbol{\Phi}(\lambda_1) \rangle$. This algebraically confirms that $h(\lambda_2) \geq h(\lambda_1)$, thus $h(\lambda)$ is strictly monotonic non-decreasing in $\lambda$.

We next localize the root $\lambda^*$. For any $\lambda > \frac{2\|\bG\|_*}{\|\bW\|_*} > 0$ and setting $\bT = \boldsymbol{\Phi}(\lambda) = \operatorname{msign}(\bG + \lambda \bW)$, the definition in Equation (\ref{eq:max_def}) ensures $\|\bT\|_2 \leq 1$. We have the relation:
\begin{equation}
\lambda h(\lambda) = \lambda \langle \bW, \bT \rangle = \langle \bG + \lambda \bW, \bT \rangle - \langle \bG, \bT \rangle. \label{eq:decomp}
\end{equation}
By the duality of norms, the inner product is bounded by $\langle \bG, \bT \rangle \leq \|\bG\|_* \|\bT\|_2 \leq \|\bG\|_*$. Simultaneously, because $\bT$ maximizes the inner product with $\bG + \lambda \bW$, we have $\langle \bG + \lambda \bW, \bT \rangle = \|\bG + \lambda \bW\|_*$. Applying the reverse triangle inequality:
\begin{equation}
\|\bG + \lambda \bW\|_* \geq \|\lambda \bW\|_* - \|\bG\|_* = \lambda \|\bW\|_* - \|\bG\|_* \label{eq:triangle}
\end{equation}
Substituting Equation (\ref{eq:triangle}) and the dual norm bound into Equation (\ref{eq:decomp}), we derive:
\begin{equation}
\lambda h(\lambda) \geq (\lambda\|W\|_* - \|\bG\|_*) - \|\bG\|_* = \lambda\|W\|_* - 2\|\bG\|_*.
\end{equation}
Since we assumed $\lambda > \frac{2\|\bG\|_*}{\|\bW\|_*}$, it follows that $\lambda h(\lambda) > 0$. Because $\lambda$ is positive, this implies $h(\lambda) > 0$. Through symmetrical logical deduction, if $\lambda < \frac{-2\|\bG\|_*}{\|\bW\|_*}$, we obtain $\lambda h(\lambda) > 0$, which inherently forces $h(\lambda) < 0$ since $\lambda$ is strictly negative. 

Because $h(\lambda)$ represents a monotonically non-decreasing mapping (as established above) that transitions from a negative state at $\lambda = \frac{-2\|\bG\|_*}{\|W\|_*}$ to a positive state at $\lambda = \frac{2\|\bG\|_*}{\|W\|_*}$, the Intermediate Value Property mathematically guarantees the existence of at least one root $\lambda^*$ satisfying $h(\lambda^*) = 0$. Consequently, any valid root must reside within the bounded interval $[\frac{-2\|\bG\|_*}{\|W\|_*}, \frac{2\|\bG\|_*}{\|W\|_*}]$, proving $|\lambda^*| \leq \frac{2\|\bG\|_*}{\|W\|_*}$.
\end{proof}

Leveraging Lemma \ref{lemma:frob_monotonicity}, we can solve the subproblem in Eq.~\ref{eq:subproblem} via bisection search. This forms the basis of their double-loop algorithms: Frobenius Sphere Optimization (FSO), as summarized in Algorithm \ref{alg:SSO}.

For MuonH, \citet{wen2025hyperball} proposed Muon-Hyperball method under Frobenius Sphere based on the update step: 
\begin{equation*}
\bW_{t+1}= \mathcal{P}_{\mathcal{M}}\left(\bW_t-\eta R \cdot \operatorname{Normalize}\left(\operatorname{msign}(M_t\right)\right).
\end{equation*}
According to the projection Retraction we mentioned in \ref{app:projection_retraction} we can get the MuonH-spec using the following update step:
\begin{equation*}
    \bW_{t+1} = \frac{R \cdot \left(\bW_t - \eta R \operatorname{msign}(M_t)\right)}{\|\bW_t - \eta R \operatorname{msign}(M_t)\|_2}.
\end{equation*}

\begin{algorithm}[th]
\caption{Frobenius Sphere Optimization (FSO)}
\label{alg:SSO}
\begin{algorithmic}[1] 
    \State \textbf{Input:} Weights $\bW_{t}$, Learning rate $\eta_t$, Hyperparameters $\beta, R$
    \State \textbf{Output:} Updated Weights $\bW_{t+1}$
    \For{$t = 0, 1, \dots$}
        \State $G_t = \nabla \ell(\bW_t,\xi_t)$
        \vspace{0.05cm}
        \State $M_t = \beta M_{t-1} + (1-\beta) G_t$
        \vspace{0.05cm}
        \State Define $h(\lambda) = \langle \bW_t, \operatorname{msign}(M_t + \lambda \bW_t)\rangle$\;
        \State $\lambda_t^* \leftarrow \operatorname{Bisection}(h, \epsilon)$ \Comment{Find the multiplier $\lambda$}
        \vspace{0.05cm}
        \State $O_t \leftarrow \operatorname{msign}(M_t + \lambda_t^* \bW_t)$ \Comment{Compute the update direction}
        \vspace{0.05cm}
        \State % $\widetilde{\nabla}_t = c \cdot \mathcal{P}_{\mathcal{M}}\left(O_t\right)$ 
        $\widetilde{\nabla}_t = R \cdot O_t / (\|O_t\|_{F}+\epsilon)$ 
        \Comment{Normalization and Scaling}
        \vspace{0.05cm}
        \State $\bW_{t+1} = \mathcal{P}_{\mathcal{M}} \left( \bW_t - \eta_t \cdot \widetilde{\nabla}_t \right)$ \Comment{Descent Step and Manifold Projection}
    \EndFor
\end{algorithmic}
\end{algorithm}

\section{Theoretical Supplementary Materials}
\subsection{Assumption and Proof for Theorem \ref{thm:convergence}}\label{app:convergence_analysis}
In this subsection we provide theoretical proof of the \ourname{} (Algorithm \ref{alg:macro}) convergence. We firstly listed two standard assumption:
\begin{assumption}\label{assumption1}
    We assume that $\mathcal{M}$ is a compact $C^3$ manifold.\footnote{See Appendix \ref{app:convergence_analysis} and \citet{yang2026manifoldconstrainedsteepestdescent} for more discussion. }
\end{assumption}
\vspace{-0.05cm}

\begin{assumption}\label{assumption2}
    We assume the loss function $\mathcal{L}$ is lower bounded by $\mathcal{L}^*$, and has $L$-Lipschitz continuous gradient, and  
    the stochastic loss $\ell(W,\xi)$ satisfies: (i) the stochastic gradient is unbiased, i.e. $\mathbb{E}_{\xi}\nabla\ell(W,\xi)=\nabla\mathcal{L}(W)$; (ii) the variance is bounded, i.e.,  $\mathbb{E}_{\xi}\|\nabla\ell(W,\xi) - \nabla\mathcal{L}(W)\|_F^2\leq\sigma^2$.
\end{assumption}

Then, we prove the convergence result in Theorem \ref{thm:convergence}, which holds for Frobenius, input and output Oblique manifolds.\footnote{The analysis follows MCSD on Stiefel manifold in ~\citep{yang2026manifoldconstrainedsteepestdescent}. We also point out that the theoretical analysis in \cite{gu2026manorestrikingmanifoldoptimization} explicitly assumes that the angles between the weights and the stochastic gradients are bounded away from zero, which is a very restrictive and non-verifiable assumption.}

\begin{proof}[Proof of Theorem \ref{thm:convergence}]
    This proof follows \cite{yang2026manifoldconstrainedsteepestdescent}. By Assumption \ref{assumption1} and \ref{assumption2}, for any $\bW\in\M$ and $G$ we have\footnote{Note that we do not require $G$ to be on the tangent space, as proved in \citet[Lemma 4.4]{yang2026manifoldconstrainedsteepestdescent}}
    \begin{align}
        \L\circ \mathcal{P}_{\M}(\bW+G) \leq \L(\bW) + \langle\nabla_{\M}\L(\bW), G  \rangle + \frac{L}{2}\|G\|_F^2,
    \end{align}
    where $\nabla_{\M}$ denotes the Riemannian gradient. 
    Now from the algorithm we get
    \begin{equation*}
    \begin{aligned}
        \L(\bW_{t+1})=\L\circ \mathcal{P}_{\M}(\bW-\eta_t\tilde{\nabla}_t) \leq & \L(\bW_t) - \eta_t\langle\nabla_{\M}\L(\bW_t), \tilde{\nabla}_t  \rangle + \frac{L\eta_t^2}{2}\|\tilde{\nabla}_t\|_F^2 \\
        = & \L(\bW_t) - \langle\nabla_{\M}\L(\bW_t), \tilde{\nabla}_t \rangle + \frac{R^2 L\eta_t^2}{2}\left(\frac{\|O_t\|_F}{\|O_t\|_{\M}+\epsilon}\right)^2.
    \end{aligned}
    \end{equation*}

    Note that $\|O_t\|_2=1$, therefore $1 \leq \|O_t\|_F^2\leq d$ with $d:=\min\{\din, \dout\}$. Moreover, we may assume a uniform bound $c_0 \|A\|_{F}\leq \|A\|_{\M} \leq c_1 \|A\|_{F}$. Therefore,
    \begin{equation*}
    \begin{aligned}
        \L(\bW_{t+1}) \leq & \L(\bW_t) - \tilde{c}\eta_t\langle\nabla_{\M}\L(\bW_t), O_t  \rangle + \frac{\tilde{c}^2L\eta_t^2}{2}d,
    \end{aligned}
    \end{equation*}
    where we denote $\tilde{c}:=\frac{R}{c_0+\epsilon}$. Further, we have
    \begin{equation*}
    \begin{aligned}
        & \L(\bW_{t+1}) \\ 
        \leq & \L(\bW_t) - \tilde{c}\eta_t\langle\nabla_{\M}\L(\bW_t), O_t  \rangle + \frac{\tilde{c}^2L\eta_t^2}{2}d \\
        = & \L(\bW_t) - \tilde{c}\eta_t \left(\langle\nabla_{\M}\L(\bW_t) - \proj_{\T_{\M} \bW_t}(M_t), O_t \rangle + \langle\proj_{\T_{\M} \bW_t}(M_t), O_t \rangle \right) + \frac{\tilde{c}^2L\eta_t^2}{2}d.
    \end{aligned}
    \end{equation*}
    Since $O_t=\mathrm{msign}(\Phi_t)$ with $\Phi_t=\proj_{\T_{\M} \bW_t}(M_t)$, we know that $\langle\proj_{\T_{\M} \bW_t}(M_t), O_t \rangle = \|\Phi_t\|_{*}$, therefore (Denote the dual norm of $\|\cdot\|_{\M}$ as $\|\cdot\|_{*,\M}$)
    \begin{equation*}
    \begin{aligned}
        & \L(\bW_{t+1}) \\ 
        \leq & \L(\bW_t) + \tilde{c}\eta_t \|\nabla_{\M}\L(\bW_t) - \proj_{\T_{\M} \bW_t}(M_t)\|_{*} - \tilde{c}\eta_t\|\Phi_t\|_{*}  + \frac{\tilde{c}^2 L\eta_t^2}{2}d. 
    \end{aligned}
    \end{equation*}
    By Triangle inequality, $-\|\Phi_t\|_{*}\leq -\|\nabla_{\M}\L(\bW_t)\|_{*} + \|\nabla_{\M}\L(\bW_t) - \proj_{\T_{\M} \bW_t}(M_t)\|_{*}$, we have
    \begin{equation*}
    \begin{aligned}
        & \L(\bW_{t+1}) \\ 
        \leq & \L(\bW_t) + \tilde{c}\eta_t \|\nabla_{\M}\L(\bW_t) - \proj_{\T_{\M} \bW_t}(M_t)\|_{*} - \tilde{c}\eta_t\|\nabla_{\M}\L(\bW_t)\|_{*}  + \frac{\tilde{c}^2 L\eta_t^2}{2}d,
    \end{aligned}
    \end{equation*}
    i.e.,
    \begin{equation}\label{eq:temp1}
    \begin{aligned}
        & \tilde{c}\eta_t\|\nabla_{\M}\L(\bW_t)\|_{*} \\ 
        \leq & (\L(\bW_t) - \L(\bW_{t+1})) + \tilde{c}\eta_t \|\nabla_{\M}\L(\bW_t) - \proj_{\T_{\M} \bW_t}(M_t)\|_{*}  + \frac{\tilde{c}^2 L\eta_t^2}{2}d.
    \end{aligned}
    \end{equation}

    It remains to bound $\|\nabla_{\M}\L(\bW_t) - \proj_{\T_{\M} \bW_t}(M_t)\|_{*}$. Since $\nabla_{\M}\L(\bW_t) = \proj_{\T_{\M} \bW_t}(\nabla\L(\bW_t))$, we have
    \begin{equation*}
    \begin{aligned}
        &\|\nabla_{\M}\L(\bW_t) - \proj_{\T_{\M} \bW_t}(M_t)\|_{*} = \|\proj_{\T_{\M} \bW_t}(\nabla\L(\bW_t) - M_t)\|_{*} \\
        \leq & \sqrt{d} \|\proj_{\T_{\M} \bW_t}(\nabla\L(\bW_t) - M_t)\|_{F}\qquad\text{(WLOG assume full-rankness)} \\
        \leq & \sqrt{d} \|\nabla\L(\bW_t) - M_t\|_{F}.
    \end{aligned}
    \end{equation*}
    Hence, we just need to bound $\|\nabla\L(\bW_t) - M_t\|_{F}$. Denote $C_t:=\beta C_{t-1}+(1-\beta)\nabla \L(\bW_t)$ with $C_0:= \L(\bW_0)$, we have:
    \begin{equation*}
        \|\nabla\L(\bW_t) - M_t\|_{F} \leq \|\nabla\L(\bW_t) - C_t\|_{F} + \|C_t - M_t\|_{F}.
    \end{equation*}
    For the first term, we have
    \begin{equation*}
    \begin{aligned}
        &\|\nabla\L(\bW_t) - C_t\|_{F}\\ = & \|\nabla\L(\bW_t) - (\beta C_{t-1}+(1-\beta)\nabla \L(\bW_t))\|_{F} \\
        = & \beta \|\nabla\L(\bW_t) - C_{t-1}\|_{F} \\
        \leq & \beta \|\nabla\L(\bW_t) - \nabla\L(\bW_{t-1})\|_{F} + \beta \|\nabla\L(\bW_{t-1}) - C_{t-1}\|_{F} \\
        \leq & \beta L \|\bW_t - \bW_{t-1}\|_{F} + \beta \|\nabla\L(\bW_{t-1}) - C_{t-1}\|_{F} \\
        \leq & \beta L \eta_t \tilde{c}+ \beta \|\nabla\L(\bW_{t-1}) - C_{t-1}\|_{F}.
    \end{aligned}
    \end{equation*}
    Applying recursively we get 
    $$
    \|\nabla\L(\bW_t) - C_t\|_{F}\leq L\eta_t \tilde{c}\sum_{i=0}^{t-1}\beta^i \leq \frac{2L\eta_t \tilde{c}}{1-\beta}.
    $$
    For the second term $\|C_t - M_t\|_{F}$, we have
    \begin{equation*}
    \begin{aligned}
        &\EE \|C_t - M_t\|_{F}\\ \leq & (1-\beta) \EE \left\| \sum_{i=1}^{t}\beta^{t-i}(G_i - \nabla \L(\bW_i)) \right\|_{F} + \beta^t \EE \|G_0 - \nabla \L(\bW_0)\|_{F} \\
        \leq & (1-\beta) \sqrt{\EE \left\| \sum_{i=1}^{t}\beta^{t-i}(G_i - \nabla \L(\bW_i)) \right\|_{F}^{2}} + \beta^t \sqrt{\EE \|G_0 - \nabla \L(\bW_0)\|_{F}^{2}} \\
        \leq & (1-\beta) \sqrt{\sum_{i=1}^{t}\beta^{2(t-i)}\sigma^2} + \beta^t \sigma \leq \left(\sqrt{\frac{1-\beta}{1+\beta}}+\beta^t\right)\sigma.
    \end{aligned}
    \end{equation*}

    Now plug in everything back to \eqref{eq:temp1}, we get
    \begin{equation*}
    \begin{aligned}
        & \tilde{c}\eta_t\EE\|\nabla_{\M}\L(\bW_t)\|_{*} \\ 
        \leq & \EE(\L(\bW_t) - \L(\bW_{t+1})) + \tilde{c}\eta_t\sqrt{d} \EE\|\nabla_{\M}\L(\bW_t) - \proj_{\T_{\M} \bW_t}(M_t)\|_{F}  + \frac{\tilde{c}^2 L\eta_t^2}{2}d \\
        \leq & \EE(\L(\bW_t) - \L(\bW_{t+1})) + \tilde{c}(c_0+1)\eta_t\sqrt{d} \left( \frac{2L\eta_t \tilde{c}}{1-\beta} + \left(\sqrt{\frac{1-\beta}{1+\beta}}+\beta^t\right)\sigma \right)  + \frac{\tilde{c}^2 L\eta_t^2}{2}d.
    \end{aligned}
    \end{equation*}
    Summing up the above inequality from $t=0$ to $T-1$, we get
    \begin{equation*}
        \sum_{t=0}^{T-1}\tilde{c}\eta_t\mathbb{E}\|\nabla_{\mathcal{M}} \mathcal{L}(\bW_t)\|_{*}\lesssim \mathbb{E}[\mathcal{L}(\bW_0)-\mathcal{L}^*] + \tilde{c}\sqrt{d}\sum_{t=0}^{T-1}\left( \frac{2\beta L{\eta_t}}{1-\beta} + (\sqrt{\frac{1-\beta}{1+\beta}}+\beta^t)\sigma \right)\eta_t + \frac{\tilde{c}^2 d L}{2}\sum_{t=0}^{T-1}\eta_t^2.
    \end{equation*}
The final result is obtained by plugging in 
    $$
    \beta=1-\frac{1}{\sqrt{T}},\ \eta_t=\eta=\Omega\left(\sqrt{\frac{\Delta}{L T^{3/2}}}\right)
    $$
    to the above inequality.
\end{proof}

\textbf{Discussion on the spectral sphere}. The constraint $\mathcal{M}_S(R) = \{\bW \in \mathbb{R}^{\dout \times \din} : \|\bW\|_2 = R > 0\}$ is actually not a manifold, due to the possible existence of multiple top singular values. This can be proved by the fact that the largest singular value is differentiable only at matrices where it is simple; see \citet{lewis2005nonsmooth}.

If we restrict the constraint to the subset where the top singular value is unique (and the first and second singular values have a small gap)
$$
\tilde{\mathcal{M}}_S(R)=\{\bW: \|\bW\|_2 = R,\sigma_1(\bW)\geq \sigma_2(\bW)+\epsilon\}
$$
then $\tilde{\mathcal{M}}_S(R)$ is shown to be a smooth embedded compact submanifold. A simple argument is that, on the closed set $\Omega=\{X:\sigma_1(X)\geq\sigma_2(X)+\epsilon\}$, the largest singular value $X\rightarrow\sigma_1(X)$ is Fréchet differentiable; again see \citet{lewis2005nonsmooth}. Also the spectral sphere constraint $\mathcal{M}_S(R) = \{\bW \in \mathbb{R}^{\dout \times \din} : \|\bW\|_2 = R > 0\}$ is a regular level set, therefore their intersection is a smooth embedded submanifold by regular level set theorem; see \citet{lee2003smooth}. We can naturally consider this submanifold to apply Algorithm \ref{alg:macro} and conduct a convergence analysis exactly the same as Theorem \ref{thm:convergence}, due to the following reasons.
\begin{itemize}[leftmargin=*]
    \item The projection onto the tangent space of $\tilde{\mathcal{M}}_S(R)$ is exactly defined in Appendix \ref{app:projection_retraction};
    \item It is hard globally compute the projectional retraction onto the set $\tilde{\mathcal{M}}_S(R)$ (The $\mathrm{msign}$ projection may not necessarily keep the top singular value unique). However we argue that this is not a problem in practice: due to stochasticity, the gradients $G_t$ and $\Phi_t$ and the updates have unique top singular values with probability one, and $\mathrm{msign}$ could serve as the approximate projection in practice. 
\end{itemize}

Therefore, despite the fact that ${\mathcal{M}}_S(R)$ is not a smooth manifold, the proposed Algorithm \ref{alg:macro} and the convergence analysis in Theorem \ref{thm:convergence} can be applied for the spectral sphere constraint \textbf{as a manifold} almost surely.

\subsection{Deriving the Radius Choices under Gaussian Initialization}
\label{app:gaussian_radius_choices}

In this section, we derive the radius bound when the weight direction is sampled independently from a normalized Gaussian distribution. At a layerwise-independent initialization, the incoming activation $\bX$ depends only on the preceding layers and is independent of the current layer's Gaussian weight matrix, so the conditional formulation below applies directly.

For brevity, we still write
\begin{equation}\label{eq:gaussian_radius_notation}
\begin{aligned}
    n&:=\din,\qquad p:=\dout,\\
    r_{\bX}&:=\frac{\|\bX\|_F}{\sqrt{Tn}},\qquad
    r_{\bY}:=\frac{\|\bY\|_F}{\sqrt{Tp}},\qquad
    \Sigma_{\bX}:=\frac{1}{T}\bX^\top\bX.
\end{aligned}
\end{equation}
All probabilities and expectations below are conditional on the realized input $\bX$; equivalently, $\bX$ may be any deterministic matrix.

\begin{lemma}[Gaussian sketch and norm concentration]
\label{lemma:gaussian_sketch_concentration}
Let $\bG\in\mathbb{R}^{p\times n}$ have i.i.d. $\mathcal{N}(0,1)$ entries and let $\bX\in\mathbb{R}^{T\times n}$ be fixed. Define
\[
    Q_{\bX}(\bG):=\|\bX\bG^\top\|_F^2.
\]
For every $\varepsilon\in(0,1)$,
\begin{align}
    \mathbb{P}\!\left(
        Q_{\bX}(\bG)
        \leq (1-\varepsilon)p\|\bX\|_F^2
    \right)
    &\leq
    \exp\!\left(-\frac{p\varepsilon^2}{4}\right),
    \label{eq:gaussian_sketch_lower_tail}\\
    \mathbb{P}\!\left(
        Q_{\bX}(\bG)
        \geq (1+\varepsilon)p\|\bX\|_F^2
    \right)
    &\leq
    \exp\!\left(-\frac{p\varepsilon^2}{8}\right).
    \label{eq:gaussian_sketch_upper_tail}
\end{align}
Moreover, for every $a>0$ and $\delta\in(0,1)$,
\begin{align}
    \mathbb{P}\!\left(
        \|\bG\|_2>\sqrt{n}+(1+a)\sqrt{p}
    \right)
    &\leq
    \exp\!\left(-\frac{a^2p}{2}\right),
    \label{eq:gaussian_operator_tail}\\
    \mathbb{P}\!\left(
        \|\bG\|_F^2
        \notin[(1-\delta)pn,(1+\delta)pn]
    \right)
    &\leq
    2\exp\!\left(-\frac{pn\delta^2}{8}\right).
    \label{eq:gaussian_fro_tail}
\end{align}
\end{lemma}

\begin{proof}
Let $\lambda_1,\ldots,\lambda_n\geq 0$ be the eigenvalues of
$\bX^\top\bX$. By rotational invariance of a Gaussian matrix,
\begin{equation}\label{eq:gaussian_sketch_eigendecomp}
    Q_{\bX}(\bG)
    \ \stackrel{d}{=}\
    \sum_{j=1}^n\lambda_j\sum_{i=1}^p Z_{ij}^2,
    \qquad
    Z_{ij}\stackrel{\mathrm{i.i.d.}}{\sim}\mathcal{N}(0,1),
\end{equation}
with mean
\[
    \mu
    =p\sum_{j=1}^n\lambda_j
    =p\|\bX\|_F^2.
\]
If $\bX=0$, the first two claims are immediate, so suppose $\mu>0$.

For the lower tail, for any $\theta>0$,
\begin{align*}
    &\log\mathbb{E}\exp\!\left(
        -\theta(Q_{\bX}(\bG)-\mu)
    \right)\\
    &\qquad=
    \frac{p}{2}\sum_{j=1}^n
    \left(
        2\theta\lambda_j-\log(1+2\theta\lambda_j)
    \right)\\
    &\qquad\leq
    p\theta^2\sum_{j=1}^n\lambda_j^2,
\end{align*}
where we used
\[
    u-\log(1+u)\leq \frac{u^2}{2},
    \qquad u\geq0.
\]
Exponential Markov inequality and the choice
\[
    \theta
    =
    \frac{\varepsilon\sum_j\lambda_j}
         {2\sum_j\lambda_j^2}
\]
give
\[
\begin{aligned}
    &\mathbb{P}\!\left(
        Q_{\bX}(\bG)\leq(1-\varepsilon)\mu
    \right)\\
    &\qquad\leq
    \exp\!\left(
       -\frac{
          p\varepsilon^2(\sum_j\lambda_j)^2
       }{
          4\sum_j\lambda_j^2
       }
    \right)\\
    &\qquad\leq
    \exp\!\left(-\frac{p\varepsilon^2}{4}\right),
\end{aligned}
\]
which proves~\eqref{eq:gaussian_sketch_lower_tail}.

For the upper tail, set
\[
    \theta
    =
    \frac{\varepsilon}{4\sum_j\lambda_j}.
\]
Then
$0\leq2\theta\lambda_j\leq\varepsilon/2\leq1/2$, and
\[
    -\log(1-u)-u\leq u^2,
    \qquad u\in[0,1/2].
\]
Hence
\begin{align*}
    &\log\mathbb{E}\exp\!\left(
        \theta(Q_{\bX}(\bG)-\mu)
    \right)\\
    &\qquad\leq
    2p\theta^2\sum_{j=1}^n\lambda_j^2\\
    &\qquad\leq
    2p\theta^2
    \left(\sum_{j=1}^n\lambda_j\right)^2
    =\frac{p\varepsilon^2}{8}.
\end{align*}
Therefore, exponential Markov inequality gives
\[
\begin{aligned}
    &\mathbb{P}\!\left(
        Q_{\bX}(\bG)\geq(1+\varepsilon)\mu
    \right)\\
    &\qquad\leq
    \exp\!\left(
        -\theta\varepsilon\mu
        +\frac{p\varepsilon^2}{8}
    \right)
    =
    \exp\!\left(-\frac{p\varepsilon^2}{8}\right),
\end{aligned}
\]
which proves~\eqref{eq:gaussian_sketch_upper_tail}.

For~\eqref{eq:gaussian_operator_tail}, the classical Gaussian
comparison bound gives
\[
    \mathbb{E}\|\bG\|_2
    \leq
    \sqrt{n}+\sqrt{p},
\]
and the map $\bG\mapsto\|\bG\|_2$ is one-Lipschitz with respect to
the Frobenius norm. Gaussian concentration with deviation
$a\sqrt p$ proves the claim.

Finally, $\|\bG\|_F^2\sim\chi^2_{pn}$. The standard Chernoff
bounds for a chi-square variable give
\[
    \mathbb{P}\!\left(
        \|\bG\|_F^2\leq(1-\delta)pn
    \right)
    \leq e^{-pn\delta^2/4},
\]
and
\[
    \mathbb{P}\!\left(
        \|\bG\|_F^2\geq(1+\delta)pn
    \right)
    \leq e^{-pn\delta^2/8}.
\]
Their union implies~\eqref{eq:gaussian_fro_tail}.
\end{proof}

\begin{theorem}[Radius choices under normalized Gaussian initialization]
\label{thm:gaussian_radius_choices}
Let $T,n,p\geq1$. Fix any
$\bX\in\mathbb{R}^{T\times n}$, and let
$\bG\in\mathbb{R}^{p\times n}$ be independent of $\bX$ with
i.i.d.\ $\mathcal{N}(0,1)$ entries.

\textnormal{(i) Spectral normalization.}
Let
\begin{equation}\label{eq:normalized_gaussian_spec}
    \bW_{\mathrm{spec}}
    :=
    R_{\mathrm{spec}}
    \frac{\bG}{\|\bG\|_2},
    \qquad
    R_{\mathrm{spec}}
    :=
    r_{\mathrm{spec}}\sqrt{\frac{p}{n}},
\end{equation}
where $r_{\mathrm{spec}}>0$, and set
\[
    r_{\bY,\mathrm{spec}}
    :=
    \frac{
        \|\bX\bW_{\mathrm{spec}}^\top\|_F
    }{
        \sqrt{Tp}
    }.
\]
Then
\begin{equation}\label{eq:spec_gaussian_deterministic_upper}
    r_{\bY,\mathrm{spec}}
    \leq
    r_{\mathrm{spec}}r_{\bX}
    \qquad\text{deterministically}.
\end{equation}
For every $\varepsilon\in(0,1)$ and $a>0$,
\begin{equation}\label{eq:spec_gaussian_two_sided}
\mathbb{P}\!\left(
    \frac{
        r_{\mathrm{spec}}\sqrt{1-\varepsilon}
    }{
        \sqrt{n/p}+1+a
    }\,r_{\bX}
    \leq
    r_{\bY,\mathrm{spec}}
    \leq
    r_{\mathrm{spec}}r_{\bX}
    \;\middle|\;\bX
\right)
\geq
1-e^{-p\varepsilon^2/4}-e^{-a^2p/2}.
\end{equation}
In particular, if $n/p\leq C$ for a dimension-independent constant $C$, then $\sqrt{n/p}$ in \eqref{eq:spec_gaussian_two_sided} may be replaced by $\sqrt C$.

\textnormal{(ii) Frobenius normalization.}
Let
\begin{equation}\label{eq:normalized_gaussian_fro}
    \bW_{\mathrm{fro}}
    :=
    R_{\mathrm{fro}}
    \frac{\bG}{\|\bG\|_F},
    \qquad
    R_{\mathrm{fro}}
    :=
    r_{\mathrm{fro}}\sqrt{p},
\end{equation}
where $r_{\mathrm{fro}}>0$, and set
\[
    r_{\bY,\mathrm{fro}}
    :=
    \frac{
        \|\bX\bW_{\mathrm{fro}}^\top\|_F
    }{
        \sqrt{Tp}
    }.
\]
For every finite $n,p$,
\begin{equation}\label{eq:fro_gaussian_exact_second_moment}
    \mathbb{E}\!\left[
        r_{\bY,\mathrm{fro}}^2
        \mid\bX
    \right]
    =
    r_{\mathrm{fro}}^2r_{\bX}^2.
\end{equation}
More strongly, for every $\varepsilon,\delta\in(0,1)$,
\begin{equation}\label{eq:fro_gaussian_two_sided}
\mathbb{P}\!\left(
    r_{\mathrm{fro}}
    \sqrt{\frac{1-\varepsilon}{1+\delta}}\,
    r_{\bX}
    \leq
    r_{\bY,\mathrm{fro}}
    \leq
    r_{\mathrm{fro}}
    \sqrt{\frac{1+\varepsilon}{1-\delta}}\,
    r_{\bX}
    \;\middle|\;\bX
\right)
\geq
1-2e^{-p\varepsilon^2/8}
 -2e^{-pn\delta^2/8}.
\end{equation}
\end{theorem}

\begin{proof}
\textbf{Spectral normalization.}
For any deterministic pair $(\bX,\bW)$,
\[
\begin{aligned}
    \frac{\|\bX\bW^\top\|_F}{\sqrt{Tp}}
    &\leq
    \frac{\|\bX\|_F\|\bW\|_2}{\sqrt{Tp}}\\
    &=
    \sqrt{\frac{n}{p}}\,
    \|\bW\|_2r_{\bX}.
\end{aligned}
\]
Substituting
$\|\bW_{\mathrm{spec}}\|_2=R_{\mathrm{spec}}$
proves~\eqref{eq:spec_gaussian_deterministic_upper}.

On the intersection of the events in
\eqref{eq:gaussian_sketch_lower_tail} and
\eqref{eq:gaussian_operator_tail},
\begin{align*}
    r_{\bY,\mathrm{spec}}
    &=
    \frac{R_{\mathrm{spec}}}{\|\bG\|_2}
    \frac{\|\bX\bG^\top\|_F}{\sqrt{Tp}}\\
    &\geq
    \frac{
        r_{\mathrm{spec}}\sqrt{p/n}
    }{
        \sqrt n+(1+a)\sqrt p
    }\,
    \frac{
        \sqrt{(1-\varepsilon)p}\,\|\bX\|_F
    }{
        \sqrt{Tp}
    }\\
    &=
    \frac{
        r_{\mathrm{spec}}\sqrt{1-\varepsilon}
    }{
        \sqrt{n/p}+1+a
    }\,
    r_{\bX}.
\end{align*}
The union bound gives~\eqref{eq:spec_gaussian_two_sided}.

\textbf{Frobenius normalization.}
Let
\[
    \bU:=\frac{\bG}{\|\bG\|_F}.
\]
Vectorizing $\bU$ gives a random vector uniformly distributed on the unit sphere in $\mathbb{R}^{pn}$. Rotational symmetry therefore implies
\[
    \mathbb{E}[U_{ij}U_{k\ell}]
    =
    \frac{1}{pn}
    \mathbf{1}\{(i,j)=(k,\ell)\},
\]
and hence
\[
    \mathbb{E}\!\left[
        \bW_{\mathrm{fro}}^\top
        \bW_{\mathrm{fro}}
    \right]
    =
    \frac{R_{\mathrm{fro}}^2}{n}\bI_n.
\]
Consequently,
\begin{align*}
    \mathbb{E}\!\left[
        r_{\bY,\mathrm{fro}}^2
        \mid\bX
    \right]
    &=
    \frac{1}{Tp}
    \Tr\!\left(
        \bX\,
        \mathbb{E}\!\left[
            \bW_{\mathrm{fro}}^\top
            \bW_{\mathrm{fro}}
        \right]
        \bX^\top
    \right)\\
    &=
    \frac{
        R_{\mathrm{fro}}^2
    }{
        Tpn
    }
    \|\bX\|_F^2\\
    &=
    \frac{R_{\mathrm{fro}}^2}{p}r_{\bX}^2\\
    &=
    r_{\mathrm{fro}}^2r_{\bX}^2,
\end{align*}
which proves~\eqref{eq:fro_gaussian_exact_second_moment}.

For the high-probability statement, write
\begin{equation}\label{eq:fro_gaussian_ratio_identity}
    r_{\bY,\mathrm{fro}}^2
    =
    \frac{R_{\mathrm{fro}}^2}{Tp}
    \frac{
        Q_{\bX}(\bG)
    }{
        \|\bG\|_F^2
    }.
\end{equation}
On the event where
\[
    (1-\varepsilon)p\|\bX\|_F^2
    \leq
    Q_{\bX}(\bG)
    \leq
    (1+\varepsilon)p\|\bX\|_F^2
\]
and
\[
    (1-\delta)pn
    \leq
    \|\bG\|_F^2
    \leq
    (1+\delta)pn,
\]
identity~\eqref{eq:fro_gaussian_ratio_identity}, together with
\[
    R_{\mathrm{fro}}^2=r_{\mathrm{fro}}^2p,
    \qquad
    \|\bX\|_F^2=Tnr_{\bX}^2,
\]
gives
\[
    r_{\mathrm{fro}}^2
    \frac{1-\varepsilon}{1+\delta}
    r_{\bX}^2
    \leq
    r_{\bY,\mathrm{fro}}^2
    \leq
    r_{\mathrm{fro}}^2
    \frac{1+\varepsilon}{1-\delta}
    r_{\bX}^2.
\]
Taking square roots and applying Lemma~\ref{lemma:gaussian_sketch_concentration} with a union bound proves~\eqref{eq:fro_gaussian_two_sided}.
\end{proof}

\begin{remark}[Implication for radius selection]
\label{rem:gaussian_radius_sharpness}
Theorem~\ref{thm:gaussian_radius_choices} directly implies the radius scalings used in the main text. In particular, under normalized Gaussian initialization, if $r_{\bX}=\Theta(1)$ and $n/p=\mathcal{O}(1)$, then maintaining $r_{\bY}=\Theta(1)$ with high probability requires, up to dimension-independent constants,
\[
    R_{\mathrm{spec}}
    =
    \Theta\!\left(\sqrt{\frac{p}{n}}\right),
    \qquad
    R_{\mathrm{fro}}
    =
    \Theta\!\left(\sqrt{p}\right).
\]
Thus the Gaussian initialization analysis recovers exactly the spectral and Frobenius radius choices adopted in Proposition~\ref{prop:radius_selection}.
\end{remark}

\subsubsection{Practical Radius Selection}
\label{app:practical_radius}

Combining Theorem \ref{thm:gaussian_radius_choices} with the norm-alignment relations of Oblique Manifolds definition in Section~\ref{sec:background} yields a single tunable hyperparameter $r = \Theta(1)$ that fixes all four radii simultaneously:
\begin{equation}\label{eq:radius_choices}
    R_{\mathrm{spec}} = r\sqrt{\tfrac{\dout}{\din}},\quad
    R_{\mathrm{fro}} = r\sqrt{\dout},\quad
    R_{\mathrm{out}} = \tfrac{R_{\mathrm{fro}}}{\sqrt{\dout}} = r,\quad
    R_{\mathrm{in}} = \tfrac{R_{\mathrm{fro}}}{\sqrt{\din}} = r\sqrt{\tfrac{\dout}{\din}}.
\end{equation}
The mapping $R_{\mathrm{out}} = \|\bW\|_F/\sqrt{\dout}$ and $R_{\mathrm{in}} = \|\bW\|_F/\sqrt{\din}$ follows directly from $\|\bW\|_F^2 = \dout R_{\mathrm{out}}^2 = \din R_{\mathrm{in}}^2$ on the corresponding manifolds; these are the choices summarized in Section~\ref{sec:radius_selection}.

\noindent\textbf{Why the Radius Depends on the Norm Choice}.
The worst-case estimate
$r_{\bY} \leq \sqrt{n/p}\,\|\bW\|_F\,r_{\bX}$
suggests using the spectral radius scaling for the Frobenius constraint.
However, this choice is too conservative at Gaussian initialization.
For $\bW_{\mathrm{fro}} = R_{\mathrm{fro}}\bG/\|\bG\|_F$,
with $\bG$ sampled independently of $\bX$ as in
Appendix~\ref{app:gaussian_radius_choices}, the second-moment calculation gives
\begin{equation}
\label{eq:radius_norm_comparison}
    \mathbb{E}\!\left[r_{\bY,\mathrm{fro}}^2 \mid \bX\right]
    =
    \frac{R_{\mathrm{fro}}^2}{p}\,r_{\bX}^2
    =
    \begin{cases}
        \displaystyle \frac{r_{\mathrm{spec}}^2}{n}\,r_{\bX}^2,
        & R_{\mathrm{fro}} = r_{\mathrm{spec}}\sqrt{p/n}, \\[4pt]
        r_{\mathrm{fro}}^2 r_{\bX}^2,
        & R_{\mathrm{fro}} = r_{\mathrm{fro}}\sqrt{p}.
    \end{cases}
\end{equation}
For order-one radius multipliers, directly using the spectral radius
therefore introduces a width-dependent attenuation in the initial
mean squared output RMS. In contrast, the proposed Frobenius radius
keeps this quantity comparable to $r_{\bX}^2$.
Figure~\ref{fig:frob_worstcase_stable_rank} (a) illustrates the practical issue:
\ourname{}-fro with $R_{\mathrm{fro}}=\sqrt{p/n}$ attains higher
training loss than Muon in the reported experiment.

\paragraph{Consistency with prior initialization schemes.}
The radii in~\eqref{eq:radius_choices} agree with the initialization conventions of \citet{yang2023spectral} and \citet{wen2025hyperball}. If the model is initialized as $\bW_0 \sim \mathcal{N}(0,\sigma^2)$ with $\sigma = \min\!\left\{\sqrt{\dout/\din},\,1\right\}/\sqrt{\din}$, then
\begin{equation}\label{eq:init_norm}
    \mathbb{E}\!\left[\|\bW_0\|_F\right]
    \;\approx\; \sqrt{\dout\din\,\mathrm{Var}(\bW_{ij})}
    \;=\; \min\!\left\{\sqrt{\tfrac{\dout}{\din}},\,1\right\}\sqrt{\dout},
\end{equation}
which matches our Frobenius radius $R_{\mathrm{fro}} = r\sqrt{\dout}$ for $r = \min\{\sqrt{\dout/\din},1\} = \Theta(1)$. Likewise, the data-dependent strategy of~\citet{wen2025hyperball} fixes the radius to the realized energy of an $\mathcal{N}(0, 1/\dout)$ initialization, which corresponds to $r = 1$ in~\eqref{eq:radius_choices}.

\begin{figure}[htbp]
    \centering
    \begin{minipage}{0.45\textwidth}
        \centering
        \includegraphics[width=\linewidth]{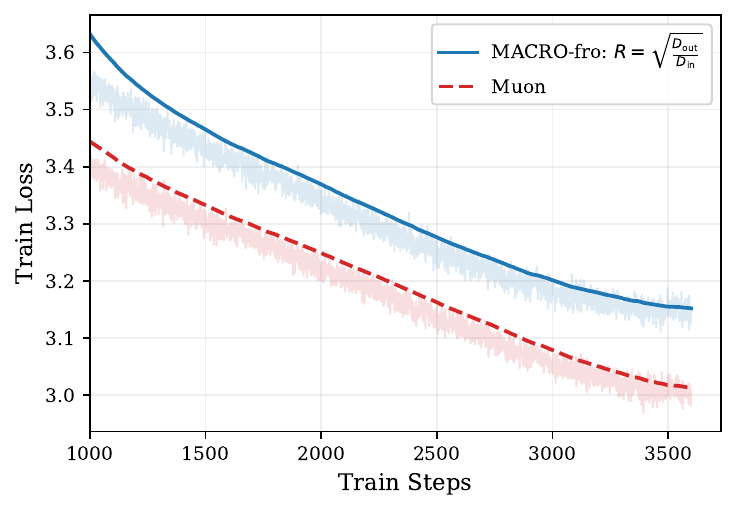}
        \vspace{0.1cm}
        {\small (a) Training Loss.}
    \end{minipage}\hfill
    \begin{minipage}{0.45\textwidth}
        \centering
        \includegraphics[width=\linewidth]{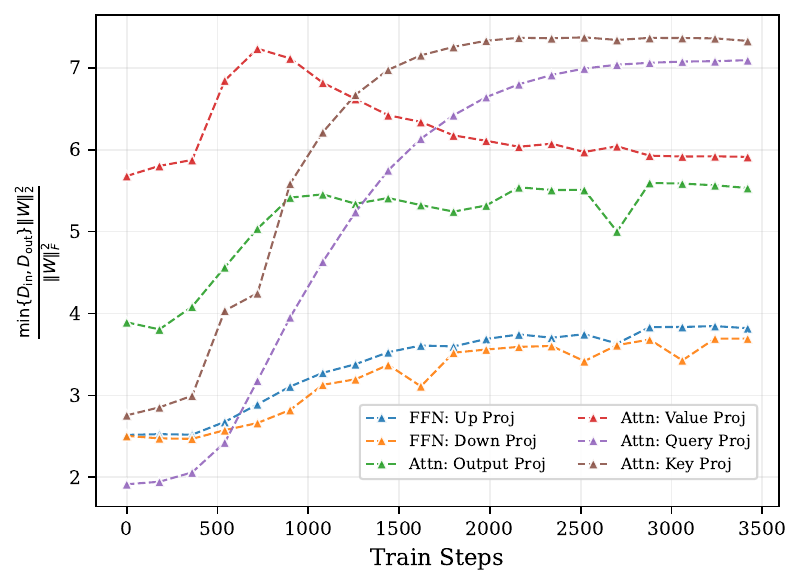}
        \vspace{0.1cm}
        {\small (b) Normalized inverse stable rank $\kappa$.}
    \end{minipage}
    \vspace{0.2cm} 
    \caption{(a) Train loss for 120M QWEN3-like model under Frobenius Sphere with different radius. (b) Normalized inverse stable rank during training.}
    \label{fig:frob_worstcase_stable_rank}
\end{figure}

% \begin{wrapfigure}{r}{0.38\textwidth}
%     \vspace{-0.3cm}
%     \centering
%     \includegraphics[width=\linewidth]{Plots/kappa_plot_updated.pdf}
%     \vspace{-0.3cm}
%     \caption{}
%     \label{fig:stable_rank}
%     \vspace{-0.3cm}
% \end{wrapfigure}
% \begin{wrapfigure}{r}{0.38\textwidth}
% \vspace{-0.3cm}
% \centering
% \includegraphics[width=\linewidth]{Plots/frob_worstcase.pdf}
% \vspace{-0.4cm}
% \caption{\small Train loss for 120M QWEN3-like model.}
% \label{fig:frob_worstcase}
% \vspace{-0.3cm}
% \end{wrapfigure}
\subsection{Rotational Equilibrium under Frobenius Sphere in Section \ref{sec:effects_on_wd}}
\label{app:fro-rotation}

This subsection supplies a formal derivation of the Frobenius rotation identity stated in the main body (Section~\ref{sec:effects_on_wd}, paragraph on static rotational equilibrium under the Frobenius sphere) that does not rely on the heuristic Frobenius-norm-preservation argument used there. The argument is parametrized by the geometric \emph{tangency cosine}
\begin{equation}\label{eq:alpha_t_def_app}
    \alpha_t \;:=\; \frac{\langle \bW_t,\,O_t\rangle_F}{R\,\|O_t\|_F} \;=\; \cos\angle(\bW_t,\,O_t)\;\in\;[-1,\,1],
\end{equation}
which is independent of $\eta_t$ because $O_t=\operatorname{msign}(\Phi_t)$ depends on the gradient geometry alone.

\paragraph{Setup.} On the Frobenius sphere $\mathcal{M}_F(R)=\{\bW:\|\bW\|_F=R\}$, the projection of any non-zero matrix $A$ is $\mathcal{P}_{\mathcal{M}_F}(A)=R\,A/\|A\|_F$. With the normalized update direction $\widetilde{\nabla}_t=(R/\|O_t\|_F)\,O_t$ used by \ourname{},
\begin{equation}\label{eq:fro_iterate_app}
    \bW_{t+1} \;=\; \mathcal{P}_{\mathcal{M}_F}\bigl(\bW_t - \eta_t\widetilde{\nabla}_t\bigr) \;=\; \frac{R}{\|\bW_t-\eta_t\widetilde{\nabla}_t\|_F}\bigl(\bW_t-\eta_t\widetilde{\nabla}_t\bigr),
\end{equation}
and $\|\bW_t\|_F=\|\bW_{t+1}\|_F=R$ by construction.

\paragraph{Exact projected angle.} Direct computation yields
\begin{align}
    \langle \bW_t-\eta_t\widetilde{\nabla}_t,\,\bW_t\rangle_F &\;=\; R^2 - \eta_t R^2\,\alpha_t \;=\; R^2\bigl(1 - \eta_t\,\alpha_t\bigr),\label{eq:num_exact_fro}\\
    \|\bW_t-\eta_t\widetilde{\nabla}_t\|_F^2 &\;=\; R^2 - 2\eta_t R^2\,\alpha_t + \eta_t^2 R^2 \;=\; R^2\bigl(1 - 2\eta_t\,\alpha_t + \eta_t^2\bigr).\label{eq:den_exact_fro}
\end{align}
Combining~\eqref{eq:fro_iterate_app}--\eqref{eq:den_exact_fro},
\begin{equation}\label{eq:exact_costheta_fro}
    \cos\theta_t \;=\; \frac{\langle \bW_{t+1},\bW_t\rangle_F}{\|\bW_{t+1}\|_F\,\|\bW_t\|_F} \;=\; \frac{1-\eta_t\,\alpha_t}{\sqrt{1-2\eta_t\,\alpha_t+\eta_t^2}}.
\end{equation}
Equation~\eqref{eq:exact_costheta_fro} is exact in $\eta_t$ and $\alpha_t$. No norm-preservation heuristic has been used; in particular, the inner product $\langle \bW_t,O_t\rangle_F$ enters only through the $\eta_t$-independent quantity $\alpha_t$.

\paragraph{Taylor expansion.} Expanding the denominator of~\eqref{eq:exact_costheta_fro} via the binomial series $(1+x)^{-1/2}=1-\tfrac{x}{2}+\tfrac{3x^2}{8}+O(x^3)$ with $x=-2\eta_t\,\alpha_t+\eta_t^2$ and retaining terms through $O(\eta_t^2)$,
\begin{equation*}
    (1-2\eta_t\,\alpha_t+\eta_t^2)^{-1/2} \;=\; 1 + \eta_t\,\alpha_t - \tfrac{1}{2}\eta_t^2 + \tfrac{3}{2}\eta_t^2\,\alpha_t^2 + O(\eta_t^3).
\end{equation*}
Multiplying by the numerator $1-\eta_t\,\alpha_t$, the $O(\eta_t)$ terms cancel and we obtain
\begin{equation}\label{eq:cos_theta_taylor}
    \cos\theta_t \;=\; 1 \;-\; \tfrac{1}{2}\,\eta_t^2\,\bigl(1 - \alpha_t^2\bigr) + O(\eta_t^3).
\end{equation}
Inverting via $\arccos(1-z)\approx\sqrt{2z}$ for small $z$ yields the rigorous Frobenius rotation identity
\begin{equation}\label{eq:fro_rotation_rigorous}
    \boxed{\;\theta_t \;\approx\; \eta_t\,\sqrt{1-\alpha_t^2} \;=\; \eta_t\,\sin\angle(\bW_t,\,O_t).\;}
\end{equation}

% \paragraph{Geometric interpretation.} Decompose $O_t=O_t^{\parallel}+O_t^{\perp}$ into components tangent and normal to $\mathcal{M}_F$ at $\bW_t$, where $O_t^{\perp}=(\langle \bW_t,O_t\rangle_F/R^2)\,\bW_t$. Then $\|O_t^{\parallel}\|_F^2=\|O_t\|_F^2(1-\alpha_t^2)$, so~\eqref{eq:fro_rotation_rigorous} states that only the tangential mass of $O_t$ produces rotation, while the radial mass is annihilated by the projection step~\eqref{eq:fro_iterate_app}. The perfectly-tangent special case $\alpha_t=0$ recovers the classical Riemannian-update angle $\theta_t\approx c\eta_t$; the more general formula~\eqref{eq:fro_rotation_rigorous} accommodates the generic $\operatorname{msign}$-induced tangency violation $\alpha_t\neq 0$ without forcing it to scale with $\eta_t$.

Across our experiments the tangent-space residual for $\mathcal{M}_F$ $\mathrm{vio}(\bW_t):=|\langle \bW_t, \widetilde{\nabla}_t\rangle|$ stabilizes at $10^{-2}$--$10^{-3}$ for every linear-layer type throughout training; see Figure~\ref{fig:section5_violation} (Frobenius panel, right). Since $|\alpha_t|=\mathrm{vio}(W_t) / R^2$, the same regime gives $|\alpha_t|\approx 0$, so
\begin{equation*}
    \sqrt{1-\alpha_t^2} \;=\; 1 - \tfrac{1}{2}\alpha_t^2 + O(\alpha_t^4) \;\approx\; 1,
\end{equation*}
and~\eqref{eq:fro_rotation_rigorous} reduces to
\begin{equation}\label{eq:fro_rotation_simple_app}
    \theta_t \;\approx\; \eta_t, 
\end{equation}
which is the identity used in the main body and confirmed empirically by the flat curves of Figure~\ref{fig:frob_dynamics}.

\paragraph{Contrast with the spectral case.} The same setup under the spectral norm requires Wedin's $\sin\Theta$ theorem and produces an \emph{adaptive} rotation angle modulated by the training-dependent spectral gap; see Appendix~\ref{app:spec-rotation}. The contrast is geometric: the Frobenius-side rigorous identity~\eqref{eq:fro_rotation_rigorous} depends only on $\eta_t$ and a small empirically-measured residual $\alpha_t$, whereas the spectral-side bound additionally depends on the iterate-dependent quantity $R-\sigma_2(\bW_{t+1})$, producing the adaptive, anisotropic dynamics visualized in Figure~\ref{fig:spec_rotation}.

%%%%%%%%%%%%%%%%%%%%%%%%%%%%%%%%%%%%%%%%%%%%%%%%%%%%%%%%%%%%
\subsection{Rotational Equilibrium under Spectral Sphere in Section \ref{sec:effects_on_wd}}
\label{app:spec-rotation}

This section supplies the formal proof of the spectral rotation bound stated in the main body (Section~\ref{sec:effects_on_wd}, paragraph ``Adaptive Rotational Equilibrium under Spectral Sphere''). Concretely, we show that under the normalized update rule of \ourname{}, the leading left and right singular subspaces of $\bW_t$ rotate by an angle whose magnitude is governed by the local spectral gap of $\bW_t$. The argument relies on (i) the standard equivalence between subspace distance and principal angles (Appendix~\ref{app:subspace_distance}), and (ii) the Wedin $\sin\Theta$ theorem for singular subspaces (Appendix~\ref{app:wedin}). The two ingredients are then combined in Appendix~\ref{app:spec_rotation_proof}.

\subsubsection{Setup: Normalized Update on the Spectral Sphere}
\label{app:spec_rotation_setup}

We start from the update step without projection step such as
\begin{equation}\label{eq:normalized_update_app}
    \widetilde{\bW}_{t+1} \;=\; \bW_t \;-\; \eta_t\,\frac{R\,O_t}{\|O_t\|_{\mathcal{M}}},\text{ and } \bW_{t+1} = \mathcal{P}_{\mathcal{M_S}}(\widetilde{\bW}_{t+1})
\end{equation}
where $O_t$ is the (un-projected) update direction returned by the optimizer, $R$ is the constraint radius, and $\|\cdot\|_{\mathcal{M}}$ denotes the manifold-specific norm. For the spectral sphere $\mathcal{M} = \mathcal{M}_{\mathrm{spec}}(R)$ we set $\|\cdot\|_{\mathcal{M}} = \|\cdot\|_2$, and we denote the spectral perturbation by
\begin{equation}\label{eq:spec_perturbation}
    \mathbf{E}_t \;:=\; \widetilde{\bW}_{t+1} - \bW_t \;=\; -\,\eta_t\,\frac{R\,O_t}{\|O_t\|_2}, \qquad \|\mathbf{E}_t\|_2 \;=\; \eta_t\,R.
\end{equation}
% Both $\bW_t$ and $\bW_{t+1}$ lie on $\mathcal{M}_{\mathrm{spec}}(R)$, so $\|\bW_t\|_2 = \|\bW_{t+1}\|_2 = R$.
Below, Wedin's theorem is applied to the pair $(\bW_t, \widetilde{\bW}_{t+1})$, and then we transfer the resulting principal-angle bound from $(\bW_t, \widetilde{\bW}_{t+1})$ to $(\bW_t, \bW_{t+1})$.

\begin{remark}[Update activation control]\label{rem:update_activation}
The same normalized update~\eqref{eq:normalized_update_app}, combined with the input scale assumption $\|\ovec(\bX)\|_{\mathrm{RMS}} = \Theta(1)$, also secures the second $\mu$P desideratum (Figure \ref{fig:muP}) on the activation update,
\begin{equation}
    \|\ovec(\Delta \bY_t)\|_{\mathrm{RMS}} \;=\; \eta_t\,\Bigl\|\ovec\!\Bigl(\tfrac{R\,O_t}{\|O_t\|_2}\,\bX\Bigr)\Bigr\|_{\mathrm{RMS}} \;=\; \Theta(1),
\end{equation}
since $R\eta_t = \Theta(1)$ and the bracketed factor obeys exactly the bound established for $\bY_t$ in Lemma~\ref{lemma:spectral_control}. 
\end{remark}

% \ka{\begin{lemma}[Spectral-sphere projection preserves singular subspaces]\label{lemma:spec_proj_subspaces}
% Let $A \in \cR^{m \times n}$ have singular value decomposition $A = U \Sigma V^\top$, with $\Sigma = \operatorname{diag}(\sigma_1, \sigma_2, \ldots)$ and $\sigma_1 \geq \sigma_2 \geq \cdots \geq 0$. The Frobenius-closest-point projection of $A$ onto $\mathcal{M}_{\mathrm{spec}}(R)$ is
% \begin{equation}\label{eq:spec_projection_formula}
%     \mathcal{P}_{\mathcal{M}_{\mathrm{spec}}(R)}(A) \;=\; U\,\operatorname{diag}\bigl(R,\,\sigma_2,\,\sigma_3,\,\ldots\bigr)\,V^\top.
% \end{equation}
% The projection thus replaces the largest singular value of $A$ by $R$, and leaves both the orthonormal factors $U, V$ and the remaining singular values $\sigma_2, \sigma_3, \ldots$ unchanged.
% \end{lemma}}

\subsubsection{Subspace Distance and Principal Angles}
\label{app:subspace_distance}

The Frobenius rotation angle of the main body is well defined because $\|\cdot\|_F$ is induced by an inner product. The spectral norm is not, and the natural object of study is instead the rotation of \emph{singular subspaces}, measured through principal angles.

\begin{definition}[Subspace distance and principal angles]\label{def:subspace_effectivelr}
Let $X, Z \in \mathbb{R}^{m \times r}$ have orthonormal columns ($X^\top X = Z^\top Z = \bI_r$). The \emph{subspace distance} between $\mathrm{range}(X)$ and $\mathrm{range}(Z)$ is
\begin{equation}\label{eq:subspace_distance}
    \operatorname{dist}(X, Z) \;:=\; \bigl\|XX^\top - ZZ^\top\bigr\|_2.
\end{equation}
Let $X^\top Z = U\,\Sigma\,V^\top$ be the singular value decomposition of $X^\top Z$, and write $\Sigma = \diag(\sigma_1,\dots,\sigma_r)$ with $1 \geq \sigma_1 \geq \dots \geq \sigma_r \geq 0$. The \emph{principal angles} $\theta_1,\dots,\theta_r \in [0,\pi/2]$ between $\mathrm{range}(X)$ and $\mathrm{range}(Z)$ are defined by $\cos\theta_i = \sigma_i$.
\end{definition}

\begin{lemma}[Subspace distance and the largest principal angle]\label{lemma:subspace_effectivelr}
For $X, Z \in \mathbb{R}^{m \times r}$ with orthonormal columns,
\begin{equation}\label{eq:subspace_effectivelr}
    \operatorname{dist}(X, Z) \;=\; \max_{1 \leq i \leq r}\,\sin\theta_i \;=\; \sin\theta_{\max}.
\end{equation}
\end{lemma}

\begin{proof}
This is~\citet[Theorem~I.5.5]{stewart1990matrix}. We give the short argument for completeness. Complete $X$ and $Z$ to orthogonal matrices $[X\ X_\perp]$ and $[Z\ Z_\perp]$ in $\mathbb{R}^{m\times m}$, so that
\begin{equation}
    XX^\top - ZZ^\top \;=\; X X^\top (I - ZZ^\top) - (I - XX^\top) ZZ^\top \;=\; X(X^\top Z_\perp) Z_\perp^\top \;-\; X_\perp (X_\perp^\top Z) Z^\top.
\end{equation}
The two terms have orthogonal column spaces and orthogonal row spaces, so
\begin{equation}
    \|XX^\top - ZZ^\top\|_2 \;=\; \max\bigl\{\|X^\top Z_\perp\|_2,\ \|X_\perp^\top Z\|_2\bigr\}.
\end{equation}
Each of these two operator norms equals $\sin\theta_{\max}$: indeed, the singular values of $X^\top Z$ are $\cos\theta_1,\dots,\cos\theta_r$, hence those of $X^\top Z_\perp$ are $\sin\theta_1,\dots,\sin\theta_r$ (e.g.,~\citealp[Thm.~I.5.2]{stewart1990matrix}), and likewise for $X_\perp^\top Z$. The maximum singular value is therefore $\sin\theta_{\max}$, which proves~\eqref{eq:subspace_effectivelr}.
\end{proof}

\subsubsection{Wedin's \texorpdfstring{$\sin\Theta$}{sin Theta} Theorem}
\label{app:wedin}

We use the singular-vector form of the Davis--Kahan / Wedin perturbation theorem.

\begin{theorem}[Wedin $\sin\Theta$ theorem; \citealp{wedin1972perturbation}, see also \citealp{davis1970rotation,stewart1990matrix}]\label{thm:wedin}
Let $A, \tilde A \in \mathbb{R}^{m\times n}$ with $\tilde A = A + E$, and let
\begin{equation}
    A = [U_0\ U_1]\begin{bmatrix}\Sigma_0 & \\ & \Sigma_1\end{bmatrix}\begin{bmatrix}V_0^\top \\ V_1^\top\end{bmatrix},
    \qquad
    \tilde A = [\tilde U_0\ \tilde U_1]\begin{bmatrix}\tilde\Sigma_0 & \\ & \tilde\Sigma_1\end{bmatrix}\begin{bmatrix}\tilde V_0^\top \\ \tilde V_1^\top\end{bmatrix}
\end{equation}
be partitioned SVDs in which $\Sigma_0, \tilde\Sigma_0$ collect the top $r$ singular values and $U_0,V_0,\tilde U_0,\tilde V_0$ have $r$ orthonormal columns. Define the spectral gap
\begin{equation}\label{eq:wedin_gap}
    \Delta \;:=\; \min_{i \leq r,\,j > r}\,\bigl|\sigma_i(A) - \sigma_j(\tilde A)\bigr|.
\end{equation}
If $\Delta > 0$, then
\begin{equation}\label{eq:wedin_bound}
    \max\bigl\{\operatorname{dist}(\tilde U_0, U_0),\ \operatorname{dist}(\tilde V_0, V_0)\bigr\}
    \;\leq\; \frac{\max\bigl\{\|E V_0\|_2,\ \|E^\top U_0\|_2\bigr\}}{\Delta}
    \;\leq\; \frac{\|E\|_2}{\Delta}.
\end{equation}
\end{theorem}

We will apply Theorem~\ref{thm:wedin} to $A = \bW_t$, $\tilde A = \bW_{t+1}$, and the perturbation $E = \mathbf{E}_t$ from~\eqref{eq:spec_perturbation}.

\subsubsection{Proof of the Spectral Rotational Equilibrium Bound}
\label{app:spec_rotation_proof}

We now combine Lemma~\ref{lemma:subspace_effectivelr} and Theorem~\ref{thm:wedin} to bound the principal rotation angles between consecutive iterates on the spectral sphere. Let
\begin{equation}
    \bW_t \;=\; [\bU^0_t\ \bU^1_t]\begin{bmatrix}\Sigma^0_t & \\ & \Sigma^1_t\end{bmatrix}\begin{bmatrix}{\bV^0_t}^\top \\ {\bV^1_t}^\top\end{bmatrix},
    \qquad
    \bW_{t+1} \;=\; [\bU^0_{t+1}\ \bU^1_{t+1}]\begin{bmatrix}\Sigma^0_{t+1} & \\ & \Sigma^1_{t+1}\end{bmatrix}\begin{bmatrix}{\bV^0_{t+1}}^\top \\ {\bV^1_{t+1}}^\top\end{bmatrix}
\end{equation}
be partitioned SVDs in which $\bU^0_\cdot, \bV^0_\cdot \in \mathbb{R}^{\dout\times r}, \mathbb{R}^{\din\times r}$ span the top-$r$ left and right singular subspaces, respectively. Set $\mathbf{E}_t$ as in~\eqref{eq:spec_perturbation}, so that $\|\mathbf{E}_t\|_2 = \eta_t R$.

\paragraph{General-rank bound.}
Theorem~\ref{thm:wedin} applied to $A = \bW_t$, $\tilde A = \bW_{t+1}$, $E = \mathbf{E}_t$, with the spectral gap
\begin{equation}\label{eq:gap_general_r}
    \Delta_r \;=\; \sigma_r(\bW_t) - \sigma_{r+1}(\bW_{t+1})
\end{equation}
(assumed strictly positive), yields
\begin{equation}\label{eq:wedin_applied}
    \max\bigl\{\operatorname{dist}(\bU^0_{t+1}, \bU^0_t),\ \operatorname{dist}(\bV^0_{t+1}, \bV^0_t)\bigr\}
    \;\leq\; \frac{\|\mathbf{E}_t\|_2}{\Delta_r}
    \;=\; \frac{\eta_t\,R}{\Delta_r}.
\end{equation}
By Lemma~\ref{lemma:subspace_effectivelr}, the left-hand side equals the maximum sine of the principal angles, so
\begin{equation}\label{eq:sin_max_bound}
    \max\bigl\{\sin\theta^u_{\max},\ \sin\theta^v_{\max}\bigr\} \;\leq\; \frac{\eta_t\,R}{\Delta_r},
\end{equation}
where $\theta^u_{\max}$ (resp.\ $\theta^v_{\max}$) is the largest principal angle between the top-$r$ left (resp.\ right) singular subspaces of $\bW_t$ and $\bW_{t+1}$.

\paragraph{Specialization to the leading direction.}
The constraint $\|\bW_t\|_2 = R$ is active precisely on the leading singular value, so the strongest geometric restriction is on the top singular pair $(\bu_t, \bv_t)$. Setting $r = 1$ and using $\sigma_1(\bW_t) = \sigma_1(\bW_{t+1}) = R$, the spectral gap~\eqref{eq:gap_general_r} becomes
\begin{equation}\label{eq:gap_r1}
    \Delta_1 \;=\; R - \sigma_2(\bW_{t+1}),
\end{equation}
and~\eqref{eq:sin_max_bound} reduces to
\begin{equation}\label{eq:1stprinciple_angle}
    \max\bigl\{\sin\theta_u,\ \sin\theta_v\bigr\} \;\leq\; \frac{\eta_t\,R}{R - \sigma_2(\bW_{t+1})},
\end{equation}
where $\theta_u = \arccos|\bu_{t+1}^\top \bu_t|$ and $\theta_v = \arccos|\bv_{t+1}^\top \bv_t|$ are the angles between the leading singular vectors of consecutive iterates.

\paragraph{Spectral effective learning rate.}
Defining the spectral effective learning rate along the primary direction as $\eta^{\mathrm{spec}}_{\mathrm{eff},\,t} := \max\{\theta_u, \theta_v\}$, the small-angle approximation $\sin\theta \approx \theta$ (valid in typical training where $\eta_t R \ll \Delta_1$) converts~\eqref{eq:1stprinciple_angle} into
\begin{equation}\label{eq:spec_eff_lr}
    \boxed{\;\eta^{\mathrm{spec}}_{\mathrm{eff},\,t} \;\lesssim\; \eta_t\,\biggl(\frac{R}{R - \sigma_2(\bW_{t+1})}\biggr).\;}
\end{equation}
The bound~\eqref{eq:spec_eff_lr} is derived by applying Wedin's theorem to the pair $(\bW_t, \widetilde{\bW}_{t+1})$ rather than to $(\bW_t, \bW_{t+1})$. However, the projected iterate $\bW_{t+1} = \mathcal{P}_{\mathcal{M}_{S}(R)}(\widetilde{\bW}_{t+1})$ shares the singular subspaces of $\widetilde{\bW}_{t+1}$ and preserves $\sigma_j$ for every $j \geq 2$. Both the principal angles $\theta_u, \theta_v$ and the spectral gap $R - \sigma_2(\bW_{t+1})$ are therefore identical when $\widetilde{\bW}_{t+1}$ is replaced by $\bW_{t+1}$. The bound~\eqref{eq:spec_eff_lr} consequently applies to the actual iterate produced by Algorithm~\ref{alg:macro}, with no additional assumption.

This completes the proof of the spectral rotation bound stated in the main body. The contrast with the Frobenius case $\eta^{\mathrm{frob}}_{\mathrm{eff},\,t} \approx \eta_t$ is geometric: under the Frobenius constraint the rotation angle is a function of the learning rate alone, whereas under the spectral constraint it is modulated by the training-dependent spectral gap $R - \sigma_2(\bW_{t+1})$, producing the adaptive, anisotropic dynamics empirically visualized in Figure~\ref{fig:spec_rotation}.

\subsection{Update Cancellation model in Section~\ref{sec:training_precision}}
\label{app:bf16_update_retention}

This appendix explains the Gaussian model we use to analyze the update cancellation under bf16 master weight training, propose the Proposition~\ref{prop:bf16_update_retention} and its proof, and gives the learning-rate calculations under lr cosine decay schedule.
Numerical comparisons are reported in
Appendix~\ref{app:precision_experiments}.
As in Section~\ref{sec:training_precision}, $\Delta\bW_t$ denotes
an additive optimizer increment before retraction and rounding;
for Muon, it excludes weight decay.

\subsubsection{Unit of Least Precision and increments cancellation under BF16 precision}

Let $w$ be a nonzero BF16 value, and let $w^-<w<w^+$ be its
adjacent representable values.
We consider finite normal values away from underflow and overflow.
Under rounding to nearest, an increment is canceled when
$w+\delta$ lies in the rounding interval of
$w$~\citep{zamirai2020revisiting}.
Ignoring midpoint ties, this condition is
\begin{equation}
\label{eq:bf16_directional_interval}
    Q_{\mathrm{BF16}}(w+\delta)=w
    \quad\Longleftrightarrow\quad
    -\frac{w-w^-}{2}<\delta<\frac{w^+-w}{2}.
\end{equation}
Midpoint ties are resolved by the ties-to-even rule.

BF16 has seven explicit fraction bits.
Away from exponent boundaries, write $|w|=m2^e$ with $1<m<2$.
Both neighboring spacings equal $h(w)=2^{e-7}$, so the local
relative cancellation threshold is
\begin{equation}
\label{eq:bf16_appendix_local_threshold}
    \tau_{\mathrm{loc}}(w)
    =\frac{h(w)}{2|w|}
    =\frac{2^{-8}}{m},
    \qquad
    2^{-9}<\tau_{\mathrm{loc}}(w)<2^{-8}.
\end{equation}
At $|w|=2^e$, the spacing toward zero is $2^{e-8}$, while the
spacing away from zero is $2^{e-7}$.
The corresponding relative half-spacings are $2^{-9}$ and $2^{-8}$.
These formulas exclude the smallest-normal boundary, where the
subnormal spacing must be considered separately.

For example, at $w=1.5$, the spacing is $0.0078125$, and the
absolute half-spacing is $0.00390625$.
An increment $\delta=0.002$ is canceled, whereas $\delta=0.004$
produces a realized increment of $0.0078125$.
Thus, a surviving increment need not retain its proposed magnitude.

We approximate the varying cancellation intervals by the common
threshold $|\delta|\leq\tau|w|$.
For all BF16 numerical estimates in this paper, we use
\begin{equation}
\label{eq:bf16_appendix_reference_threshold}
    \tau\approx1.96\times10^{-3}.
\end{equation}
This is a lower-end reference from the BF16 format, not a fitted
loss-onset threshold or the exact threshold of every coordinate.
The threshold model cancels small increments and leaves surviving
proposed increments unchanged; it does not reproduce their full
BF16 rounding.

\subsubsection{Formal update cancellation model and the proposition}

In this section, we formally present the model we use to explain the loss degradation during the late stage of training in Section~\ref{sec:training_precision}. Fix a training step and flatten $\bW_t$ and
$\Delta\bW_t$ into $N=pn$ coordinate pairs
$(w_i,\delta_i)$. We define
\[
    D_i:=\mathbf{1}\{|\delta_i|>\tau|w_i|\}.
\]
Then the canceled-coordinate fraction and the retained fraction of
proposed absolute-update magnitude are defined as:
\begin{equation}
\label{eq:bf16_empirical_retention}
    \widehat p_{\mathrm{drop}}
    =\frac{1}{N}\sum_{i=1}^N(1-D_i),
    \qquad
    \widehat{\mathcal{R}}
    =\frac{\sum_{i=1}^N|\delta_i|D_i}
           {\sum_{i=1}^N|\delta_i|}.
\end{equation}
The two statistics use different weights: each coordinate contributes
equally to $\widehat p_{\mathrm{drop}}$, whereas its contribution to
$\widehat{\mathcal{R}}$ is proportional to $|\delta_i|$.

Let $I$ be uniform on $\{1,\ldots,N\}$.
Then
\begin{equation}
\label{eq:bf16_coordinate_sampling}
\begin{aligned}
    \widehat p_{\mathrm{drop}}
    &=\Pr_I\!\left(|\delta_I|\leq\tau|w_I|\right),\\
    \widehat{\mathcal{R}}
    &=\frac{\mathbb{E}_I[|\delta_I|D_I]}
            {\mathbb{E}_I|\delta_I|}.
\end{aligned}
\end{equation}
These are exact identities for the fixed arrays.
The sampling only expresses proportions across coordinates
as probabilities and expectations; it does not change the
rounding procedure.

To simulate the update cancellation, we follow the observations that weight matrices can be modeled as zero-mean Gaussian random matrices with i.i.d. entries~\citep{huang2024billm, kedia2024transformers, babaoglu2026multi}. To match the random Gaussian Model with the scale of the weights and updates we set second moments of the Gaussian Model by setting
\begin{equation}
\label{eq:bf16_gaussian_moment_matching}
    \sigma_w^2=\frac{1}{N}\sum_{i=1}^N w_i^2 = \|W_t\|_{\text{RMS}}^2,
    \qquad
    \sigma_s^2=\frac{1}{N}\sum_{i=1}^N \delta_i^2 = \|\Delta W_t\|_{\text{RMS}}^2.
\end{equation}
Then we have following proposition:
\begin{proposition}[Update retention under a Gaussian threshold model]
\label{prop:bf16_update_retention}
Let $w\sim\mathcal N(0,\|\bW_t\|_{\text{RMS}})$ and
$\delta\sim\mathcal N(0,\|\Delta \bW_t\|_{\text{RMS}}^2)$ be independent.
For a cancellation threshold $\tau>0$, set $\rho=\sigma_s/\sigma_w$ and
$\delta_\tau=\delta\,\mathbf 1\{|\delta|>\tau|w|\}$.
The canceled-coordinate probability and the retained fraction
of absolute-update magnitude are
\vspace{-0.1cm}
\begin{equation}
\label{eq:bf16_retained_prob_and_magnitude}
    p_{\mathrm{drop}}(\rho):=\Pr(|\delta|\leq\tau|w|)
      =\frac{2}{\pi}\arctan\frac{\tau}{\rho}, \quad \mathcal R(\rho)
    :=\frac{\mathbb E|\delta_\tau|}{\mathbb E|\delta|}
      =\frac{1}{\sqrt{1+(\tau/\rho)^2}}.
\end{equation}
\end{proposition}
\begin{proof}[Proof of Proposition~\ref{prop:bf16_update_retention}]
\textbf{Discarded-coordinate probability.}
Write $w=\sigma_w Z_w$ and $\delta=\sigma_s Z_s$, where
$Z_w$ and $Z_s$ are independent standard Gaussian variables.
The cancellation event becomes
\[
    |\delta|\leq\tau|w|
    \quad\Longleftrightarrow\quad
    \left|\frac{Z_s}{Z_w}\right|\leq\frac{\tau}{\rho}.
\]
The ratio $Z_s/Z_w$ has density
\[
    f(z)
    =\int_{-\infty}^{\infty}
      \frac{|y|}{2\pi}
      \exp\!\left(-\frac{(1+z^2)y^2}{2}\right)\,dy
    =\frac{1}{\pi(1+z^2)}.
\]
Integrating this density gives
\[
    p_{\mathrm{drop}}(\rho)
    =\int_{-\tau/\rho}^{\tau/\rho}\frac{dz}{\pi(1+z^2)}
    =\frac{2}{\pi}\arctan\frac{\tau}{\rho}.
\]

\textbf{Retained absolute-update magnitude.}
Condition on $w=y$ and put $c=\tau|y|$.
Independence implies
\begin{align*}
    \mathbb{E}\!\left[|\delta_\tau|\mid w=y\right]
    &=2\int_c^\infty
       x\,\frac{\exp(-x^2/(2\sigma_s^2))}
                  {\sigma_s\sqrt{2\pi}}\,dx\\
    &=\sigma_s\sqrt{\frac{2}{\pi}}\,
      \exp\!\left(-\frac{\tau^2y^2}{2\sigma_s^2}\right).
\end{align*}
Since $\mathbb{E}|\delta|=\sigma_s\sqrt{2/\pi}$,
averaging over $w$ yields
\begin{align*}
    \mathcal{R}(\rho)
    &=\mathbb{E}_w
      \exp\!\left(-\frac{\tau^2w^2}{2\sigma_s^2}\right)\\
    &=\frac{1}{\sigma_w\sqrt{2\pi}}
      \int_{-\infty}^{\infty}
      \exp\!\left[
      -\frac12\left(\frac{1}{\sigma_w^2}
                   +\frac{\tau^2}{\sigma_s^2}\right)y^2
      \right]\,dy\\
    &=\left(1+\frac{\tau^2\sigma_w^2}{\sigma_s^2}\right)^{-1/2}
      =\frac{1}{\sqrt{1+(\tau/\rho)^2}}.
\end{align*}
\end{proof}

\paragraph{Retained relative scale.}
The Gaussian means of the absolute values satisfy
\[
    \frac{\mathbb{E}|\delta|}{\mathbb{E}|w|}
    =\frac{\sigma_s}{\sigma_w}=\rho.
\]
It follows that
\[
    \frac{\mathbb{E}|\delta_\tau|}{\mathbb{E}|w|}
    =\rho\,\mathcal{R}(\rho)
    =\frac{\rho^2}{\sqrt{\rho^2+\tau^2}}.
\]
For the idealized Frobenius-normalized step, $\rho=\eta_t$, giving
\begin{equation}
\label{eq:bf16_retained_relative_step}
    \eta_{\mathrm{ret},t}
    :=\frac{\mathbb{E}|\delta_\tau|}{\mathbb{E}|w|}
    =\eta_t\mathcal{R}(\eta_t)
    =\frac{\eta_t^2}{\sqrt{\eta_t^2+\tau^2}}.
\end{equation}
This diagnostic uses mean absolute magnitudes, rather than RMS
magnitudes of the thresholded update, and is not an effective
learning rate for the realized BF16 trajectory.

For $\tau/\rho\to0$, the binomial expansion gives
\[
    1-\mathcal{R}(\rho)
    =\frac12\left(\frac{\tau}{\rho}\right)^2
     +O\!\left(\left(\frac{\tau}{\rho}\right)^4\right).
\]
For $\rho/\tau\to0$,
\[
    \mathcal{R}(\rho)\sim\frac{\rho}{\tau},
    \qquad
    \rho\,\mathcal{R}(\rho)\sim\frac{\rho^2}{\tau}.
\]
These limits describe gradual changes in cancellation, not a
binary transition of the entire matrix.

\paragraph{Relative-update scales for the three optimizers.}
For a nonzero direction $O_t$, omitting $\epsilon$, the proposed
increments give
\begin{equation}
\label{eq:bf16_general_relative_scales}
    \rho_t=
    \begin{cases}
        \widetilde\eta_t\|O_t\|_F/\|\bW_t\|_F,
            &\text{Muon},\\[3pt]
        \eta_t R/\|\bW_t\|_F,
            &\text{\ourname{}-fro},\\[3pt]
        \eta_t R\|O_t\|_F/(\|O_t\|_2\|\bW_t\|_F),
            &\text{\ourname{}-spec}.
    \end{cases}
\end{equation}
Each row uses its own optimizer direction and weight matrix.
For an idealized full-rank matrix-sign direction,
$\|O_t\|_2=1$ and $\|O_t\|_F=\sqrt d$, where $d=\min\{p,n\}$.
Together with exact norm constraints, this gives the main-text
comparison in Eq.~\eqref{eq:bf16_update_scale_comparison}.
For an arbitrary nonzero spectral update direction, the exact
norm-ratio identity is
\begin{equation}
\label{eq:bf16_spectral_shape_ratio}
    \rho_t
    =\eta_t\frac{R}{\|\bW_t\|_2}
      \sqrt{\frac{\kappa(\bW_t)}{\kappa(O_t)}},
    \qquad
    \kappa(A)=d\frac{\|A\|_2^2}{\|A\|_F^2}.
\end{equation}
The full-rank matrix-sign idealization sets $\kappa(O_t)=1$.
Substituting the corresponding $\rho_t$ into $\mathcal R(\rho_t)$
gives the method-specific retention expressions used in the main text.

\section{Numerical Experiments Supplementary Materials}
\label{app:experimental_details}
This appendix provides supplementary experimental details for the RMSNorm study in Section~\ref{sec:rmsnorm_interplay}, the weight-decay and dynamics study in Section~\ref{sec:effects_on_wd}, the precision study in Section~\ref{sec:training_precision}, and the standard pre-training experiments in Section~\ref{sec:justification}. The configuration for each experiment group is reported in its corresponding subsection. Computational-efficiency results are reported in Table~\ref{tab:optimizer_efficiency}.

\subsection{Numerical Experiments Details in Section \ref{sec:rmsnorm_interplay}}
\label{app:rmsnorm_experiments}

\paragraph{Experimental Setting.} We test whether \ourname{} can train a 330M-parameter Transformer to convergence with all learnable RMSNorm layers removed. The model is a QWEN3-like architecture with RoPE, GQA, and Norm-Gated SwiGLU, augmented by the parameter-free normalizations introduced in Section~\ref{sec:rmsnorm_interplay}: a parameter-free RMSNorm after the attention block, QK-norm, and the Norm-Gated SwiGLU activation. Table~\ref{tab:norm_free_setup} reports the full model and training configuration.

\begin{table}[htbp]
  \centering
  \renewcommand{\arraystretch}{1.2}
  \setlength{\tabcolsep}{9pt}
  \begin{tabular}{llc}
    \toprule
    \textbf{Category} & \textbf{Parameter} & \textbf{330M} \\
    \midrule

    \multirow{7}{*}{\textbf{Model Configuration}}
    & dimension & 1024 \\
    & number of layers & 24 \\
    & heads & 16 \\
    & kv-heads & 8 \\
    & sequence length & 1024 \\
    & vocabulary size & 50304 \\
    & Embedding Tie & True \\

    \midrule

    \multirow{8}{*}{\textbf{Train Configuration}}
    & batch size & 64 \\
    & gradient accumulation step & 2 \\
    & lr schedule & linear warmup + cosine decay \\
    & warmup steps & 850 \\
    & cosine decay lr & $10^{-3} \times lr$ \\
    & Train Steps & 8500 \\
    & \# GPUs & 8 \\
    & \# Tokens & 8.9B \\
    \bottomrule
  \end{tabular}
  \vspace{0.2cm}
  \caption{Model architecture and training configuration for the normalization-free 330M QWEN3-like experiment in Section~\ref{sec:rmsnorm_interplay}.}
  \label{tab:norm_free_setup}
\end{table}

\paragraph{Hyperparameter Sweep and Selection.} For each of Muon, \ourname{}-spec, and \ourname{}-fro, we sweep the learning rate $\eta_t$ over $\{3\times 10^{-3},\, 5\times 10^{-3},\, 7\times 10^{-3},\, 1\times 10^{-2},\, 3\times 10^{-2}\}$. For the two constrained variants, we additionally sweep the radius hyperparameter $r$ over $\{0.7, 0.8, 0.9, 1.0, 1.1, 1.2, 1.3, 1.4, 1.5\}$. Additionally, we also search qkv-split for \ourname{}-spec following~\cite{xie2026controlledllmtrainingspectral}. Table~\ref{tab:norm_free_sweep} summarizes the search ranges, and Table~\ref{tab:norm_free_chosen_hp} reports the per-optimizer hyperparameters that produce the bolded entries of Table~\ref{tab:ablation_qwen} in the main text.

\begin{table}[htbp]
  \centering
  \renewcommand{\arraystretch}{1.2}
  \setlength{\tabcolsep}{9pt}
  \begin{tabular}{lccc}
    \toprule
    \textbf{Optimizer} & \textbf{learning rate $\eta_t$} & \textbf{radius $r$} & \textbf{qkv-split} \\
    \midrule

    Muon & $[3 \times 10^{-3},\, 3 \times 10^{-2}]$ & - & - \\
    \ourname{}-spec & $[3 \times 10^{-3},\, 3 \times 10^{-2}]$ & $[0.7,\, 1.5]$ & [True, False] \\
    \ourname{}-fro & $[3 \times 10^{-3},\, 3 \times 10^{-2}]$ & $[0.7,\, 1.5]$ & - \\
    \bottomrule
  \end{tabular}
  \vspace{0.2cm}
  \caption{Hyperparameter search ranges for each optimizer in the normalization-free setting.}
  \label{tab:norm_free_sweep}
\end{table}

\begin{table}[htbp]
  \centering
  \renewcommand{\arraystretch}{1.2}
  \setlength{\tabcolsep}{9pt}
  \begin{tabular}{lccc}
    \toprule
    \textbf{Parameter} & \textbf{Muon} & \textbf{\ourname{}-spec} & \textbf{\ourname{}-fro} \\
    \midrule

    learning rate $\eta_t$ & 0.01 & 0.01 & 0.01 \\
    weight decay & 0.0 & 0.0 & 0.0\\
    radius $r$ & - & 1.2 & 0.8 \\
    alignment $c$ & - & 1.0 & 1.0  \\
    qkv-split & - & False & - \\
    \bottomrule
  \end{tabular}
  \vspace{0.2cm}
  \caption{Selected hyperparameters per optimizer used to produce the bolded entries of Table~\ref{tab:ablation_qwen}.}
  \label{tab:norm_free_chosen_hp}
\end{table}
\paragraph{Numerical Result Plots.} Figure~\ref{fig:norm_free_loss_curves} plots the training and validation loss across the 8500 training steps for all three optimizers at the chosen learning rate $\eta_t = 10^{-2}$, using the radii reported in Table~\ref{tab:norm_free_chosen_hp}. These curves give the per-step picture behind the best result in Table~\ref{tab:ablation_qwen}.

\begin{figure}[htbp]
    \centering
    \begin{minipage}{0.45\textwidth}
        \centering
        \includegraphics[width=\linewidth]{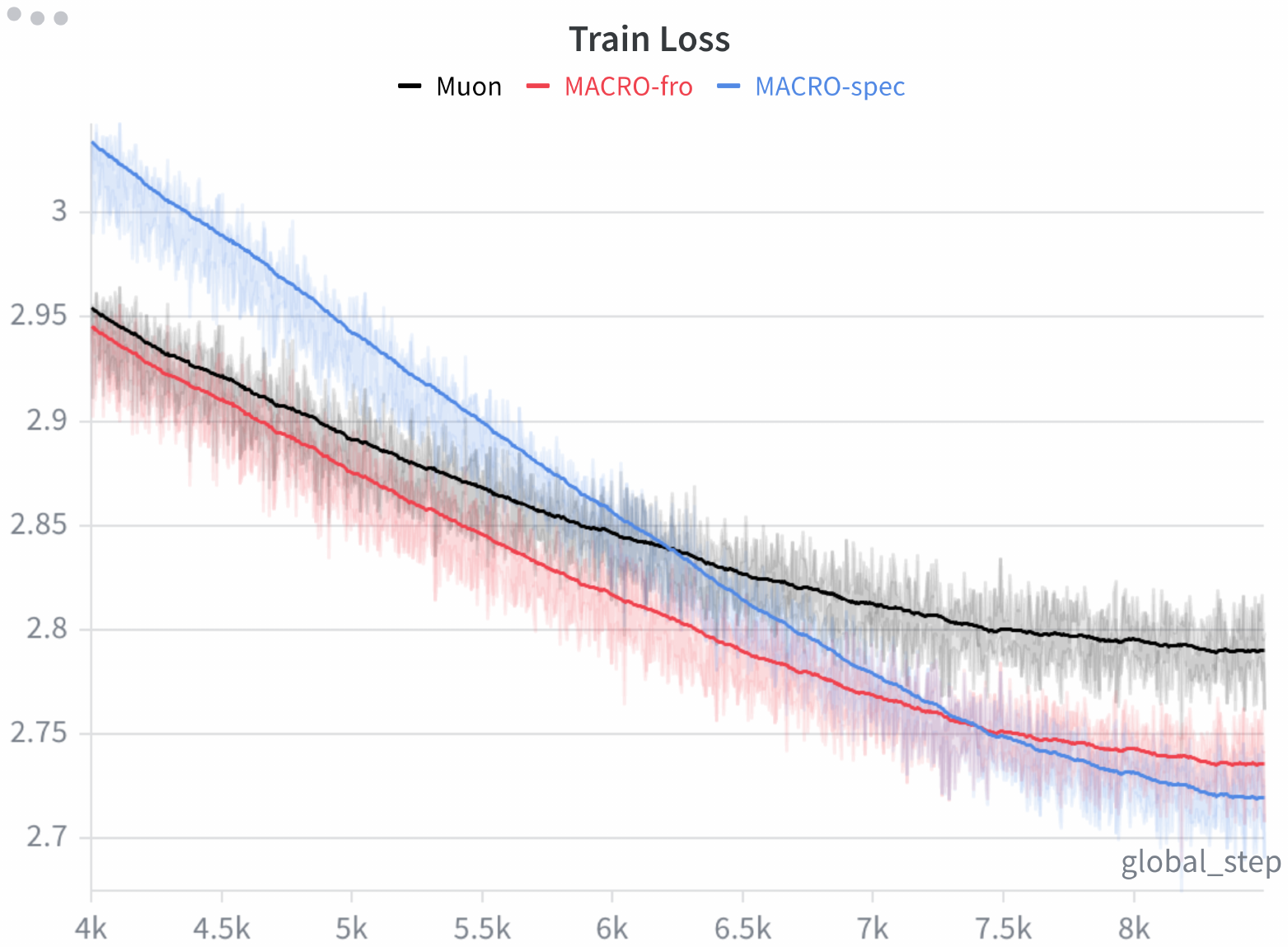}
        \vspace{0.1cm}
        {\small (a) Training Loss.}
    \end{minipage}\hfill
    \begin{minipage}{0.45\textwidth}
        \centering
        \includegraphics[width=\linewidth]{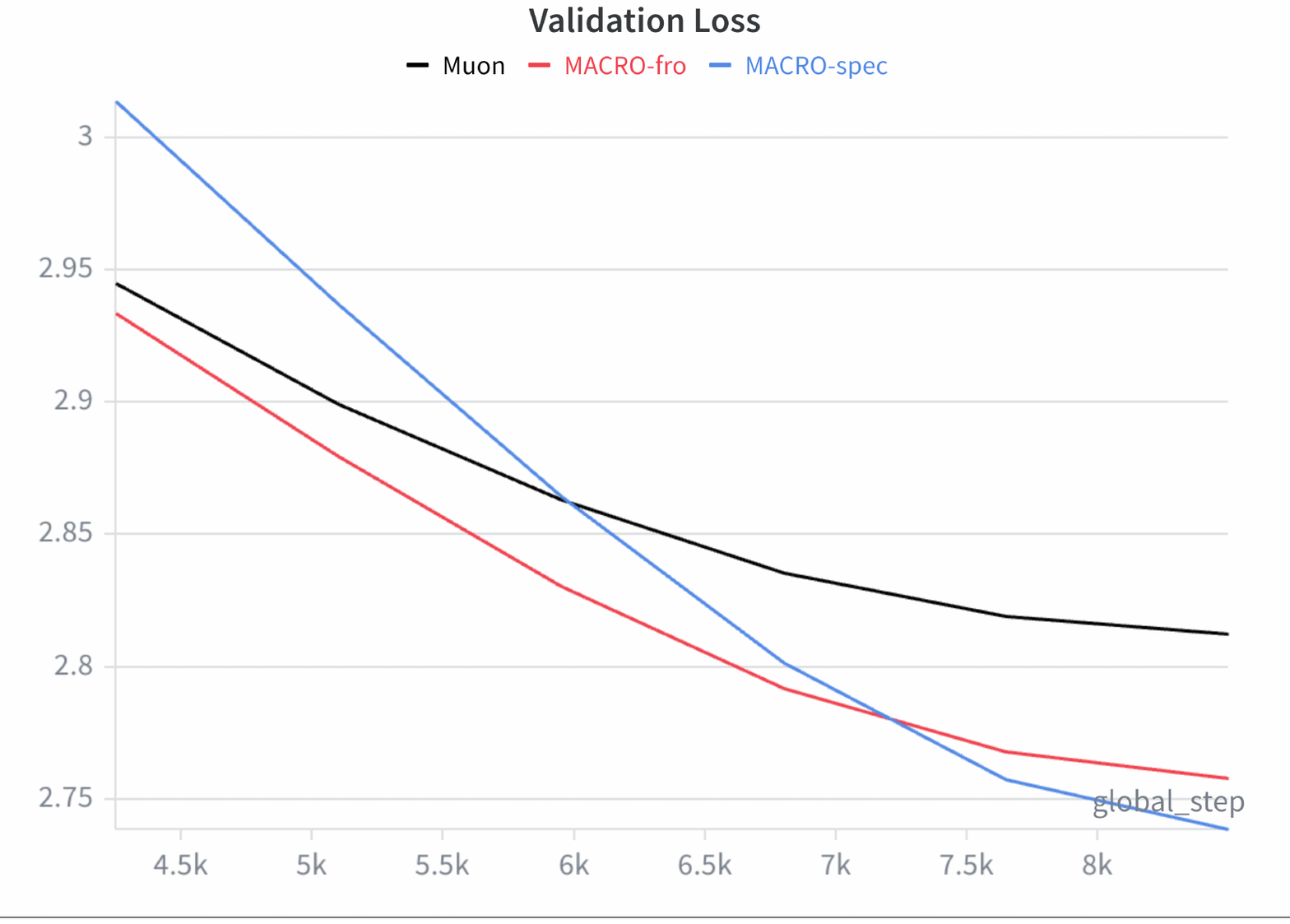}
        \vspace{0.1cm}
        {\small (b) Validation Loss.}
    \end{minipage}
    \vspace{0.2cm} % Give some breathing room before the main caption
    \caption{Training and validation loss for Muon, \ourname{}-fro, and \ourname{}-spec on the 330M QWEN3-like model with all learnable RMSNorm layers removed, at the chosen learning rate $\eta_t = 10^{-2}$ (Table~\ref{tab:norm_free_chosen_hp}).}
    \label{fig:norm_free_loss_curves}
\end{figure}

\paragraph{Ablation Study.} To validate the architecture design in Section~\ref{sec:rmsnorm_interplay}, we run \ourname{}-spec while toggling three architectural choices: the activation (standard SwiGLU, Norm-Gated SwiGLU, or ReLU), the parameter-free RMSNorm after the attention block (Att-Norm), and the QK-norm. Each configuration is trained at three constraint radii $r \in \{1.1, 1.2, 1.3\}$. Table~\ref{tab:norm_free_ablation} reports the resulting validation loss. Standard SwiGLU diverges to \texttt{NaN} at every tested radius, even with both Att-Norm and QK-norm enabled. The Norm-Gated SwiGLU activation is the component that recovers stability, and adding Att-Norm and QK-norm on top yields the lowest validation loss across all three radii. ReLU is stable but underperforms Norm-Gated SwiGLU with both Att-Norm and QK-norm enabled.

\begin{table}[htbp]
  \centering
  \renewcommand{\arraystretch}{1.2}
  \setlength{\tabcolsep}{10pt}
  \begin{tabular}{cccccc}
    \toprule
    \textbf{Activation} & \textbf{Att-Norm} & \textbf{QK-norm} & \textbf{$r=1.1$} & \textbf{$r=1.2$} & \textbf{$r=1.3$} \\
    \midrule

    SwiGLU & True & True & \texttt{NaN} & \texttt{NaN} & \texttt{NaN} \\
    Norm-Gated SwiGLU & False & False & 2.76 & 2.762 & 2.77  \\
    Norm-Gated SwiGLU & True & False & 2.747 & 2.746 & 2.749 \\
    Norm-Gated SwiGLU & True & True & 2.744 & 2.742 & 2.743  \\
    ReLU & True & True & 2.762 & 2.755 & 2.757 \\
    \bottomrule
  \end{tabular}
  \vspace{0.2cm}
  \caption{Ablation Study on RMSNorm-free training}
  \label{tab:norm_free_ablation}
\end{table}

\subsection{Numerical Experiments Details in Section \ref{sec:effects_on_wd}}
\label{app:weight_decay_experiments}
% TODO[APP-DYNAMICS]: Update the introductory claims, gradient-norm paragraph/caption, and muP paragraph/caption to match the proposed-step definition in eq:relative_lr. Retain gradient-norm decay and muP transfer as empirical observations rather than strict consequences of the ratio identity; retain all existing data and true model scales.
% TODO[APP-WARMUP-LR-WD]: The 240M/600M warmup and separate 240M joint LR-WD results have been moved into the paragraphs below. Confirm their actual peak LR, radius, initialization projection, schedule changes, parameter groups, budgets, and statistical definitions; do not inherit the older configurations in tab:norm_free_setup or tab:section5_setup. Keep tab:wd_ablation as a distinct 330M experiment.

This subsection supplies the empirical evidence for the three claims of Section~\ref{sec:effects_on_wd}. 
% First, the locked relative learning rate $\eta_{\text{rel},t} = c\,\eta_t$ produces a monotonically decaying global gradient norm under the cosine learning-rate schedule. Second, 
Manifold constraints subsume the regularization role of decoupled weight decay. Moreover, the locked update-weight-ratio enables zero-shot Maximal Update Parametrization ($\mu$P) transfer across model widths. The gradient-norm, 330M weight-decay ablation, and $\mu$P experiments reuse the QWEN3-like configurations summarized in Table~\ref{tab:legacy_dynamics_setup}.

% \begin{table}[htbp]
%   \centering
%   \renewcommand{\arraystretch}{1.2}
%   \setlength{\tabcolsep}{8pt}
%   \begin{tabular}{llccc}
%     \toprule
%     \textbf{Category} & \textbf{Parameter} & 120M & 330M & 1B  \\
%     \midrule
%     \multirow{7}{*}{\textbf{Model Config}}
%     & dimension & 768 & 1024 & 2048  \\
%     & number of layers & 12 & 24 & 20 \\
%     & heads & 6 & 16 & 16 \\
%     & kv\_heads & 3 & 8 & 8 \\
%     & sequence length & 1024 & 1024 & 1024 \\
%     & vocabulary size & 50304 & 50304 & 50304 \\
%     & Embedding Tie & True & True & True \\
%     \midrule
%     \multirow{7}{*}{\textbf{Train Config}}
%     & batch size & 128 & 64 & 16 \\
%     & gradient accumulation step & 1 & 2 & 16 \\
%     & warmup steps & 360 & 850 & 2400 \\
%     & cosine decay lr & $10^{-3} \times lr$ & $10^{-3} \times lr$ & $10^{-3} \times lr$ \\
%     & Train Steps & 3600 & 8500 & 24000 \\
%     & \# GPUs & 8 & 8 & 8  \\
%     & \# Tokens & 3.7B & 8.9B & 50B \\
%     \bottomrule
%   \end{tabular}
%   \vspace{0.2cm}
%   \caption{Legacy model architecture and training configuration referenced by the 120M and 330M dynamics experiments in this subsection.}
%   \label{tab:legacy_dynamics_setup}
% \end{table}

\begin{table}[htbp]
  \centering
  \renewcommand{\arraystretch}{1.2}
  \setlength{\tabcolsep}{8pt}
  \begin{tabular}{llcc}
    \toprule
    \textbf{Category} & \textbf{Parameter} & 120M & 330M  \\
    \midrule
    \multirow{7}{*}{\textbf{Model Config}}
    & dimension & 768 & 1024  \\
    & number of layers & 12 & 24 \\
    & heads & 6 & 16 \\
    & kv\_heads & 3 & 8 \\
    & sequence length & 1024 & 1024 \\
    & vocabulary size & 50304 & 50304 \\
    & Embedding Tie & True & True \\
    \midrule
    \multirow{7}{*}{\textbf{Train Config}}
    & batch size & 128 & 64 \\
    & gradient accumulation step & 1 & 2 \\
    & warmup steps & 360 & 850 \\
    & cosine decay lr & $10^{-3} \times lr$ & $10^{-3} \times lr$ \\
    & Train Steps & 3600 & 8500 \\
    & \# GPUs & 8 & 8  \\
    & \# Tokens & 3.7B & 8.9B \\
    \bottomrule
  \end{tabular}
  \vspace{0.2cm}
  \caption{Training configuration referenced by the 120M and 330M dynamics experiments in Section~\ref{sec:rmsnorm_interplay} and Section~\ref{sec:effects_on_wd}.}
  \label{tab:legacy_dynamics_setup}
\end{table}

\paragraph{Weight Decay Ablation.} Section~\ref{sec:effects_on_wd} argues that manifold constraints subsume the regularization effect of decoupled weight decay. To verify this empirically, we train the 330M QWEN3-like model with each of MuonH-spec, MuonH-fro, \ourname{}-spec, and \ourname{}-fro under two settings: $\lambda = 0.0$ (no weight decay) and $\lambda = 0.1$. We sweep the learning rate over $\eta_t \in \{0.007,\, 0.01,\, 0.03\}$ and hold every other hyperparameter fixed at the values in Table~\ref{tab:section5_chosen_hp}. Table~\ref{tab:wd_ablation} reports the resulting validation loss. For every optimizer and every learning rate, the two $\lambda$ settings agree within $0.001$. Adding weight decay therefore has no measurable effect once a manifold constraint is active, which confirms the substitution claim.

\begin{table}[htbp]
  \centering
  \renewcommand{\arraystretch}{1.2}
  \setlength{\tabcolsep}{10pt}
  \begin{tabular}{lccc}
    \toprule
    \textbf{Optimizer} & $\eta_t = 0.007$ & $\eta_t = 0.01$ & $\eta_t = 0.03$ \\
    \midrule
    MuonH-spec ($\lambda = 0.0$) & 2.728 & 2.715 & 2.766 \\
    MuonH-spec ($\lambda = 0.1$) & 2.729 & 2.715 & 2.766 \\
    \midrule
    MuonH-fro  ($\lambda = 0.0$) & 2.740 & 2.722 & 2.717 \\
    MuonH-fro  ($\lambda = 0.1$) & 2.740 & 2.722 & 2.717 \\
    \midrule
    \ourname{}-spec ($\lambda = 0.0$) & 2.728 & 2.715 & 2.783 \\
    \ourname{}-spec ($\lambda = 0.1$) & 2.727 & 2.715 & 2.783 \\
    \midrule
    \ourname{}-fro  ($\lambda = 0.0$) & 2.742 & 2.730 & 2.718 \\
    \ourname{}-fro  ($\lambda = 0.1$) & 2.742 & 2.729 & 2.718 \\
    \bottomrule
  \end{tabular}
  \vspace{0.2cm}
  \caption{Weight-decay ablation on the 330M QWEN3-like model. Each manifold-constrained optimizer is trained with weight-decay coefficient $\lambda \in \{0.0,\, 0.1\}$ at three learning rates $\eta_t$. Across every cell, the two $\lambda$ settings differ by at most $0.001$, which confirms that manifold constraints subsume the regularization role of decoupled weight decay. This phenomenon contradicts the Muon which benefits from weight decay widely reported in the literature. }
  \label{tab:wd_ablation}
\end{table}

\paragraph{Joint Learning-Rate and Weight-Decay Sweep for Muon.} We next ask whether jointly tuning Muon's learning rate and weight decay closes the performance gap with \ourname{} without weight decay on the constrained matrices. Table~\ref{tab:muon_lr_wd_coupling} reports a separate sweep on the 240M QWEN3-like model with the model config in Table~\ref{tab:section5_setup}, using learning rates $\{0.008,\,0.016,\,0.032\}$ and weight-decay coefficients $\{0.01,\,0.05,\,0.10,\,0.50\}$. Within the tested grid, Muon achieves a best validation PPL of 20.583, essentially matching 20.579 for \ourname{}-fro, while \ourname{}-spec achieves a lower PPL of 20.502.
% TODO[LR-WD-SETTINGS]: Confirm the dataset, token budget, seed/statistical definition, warmup schedule, radius selection, and parameter-group treatment for tab:muon_lr_wd_coupling. Its Muon grid has no WD=0 column; do not infer one or an equal-total-budget tuning protocol. Matrix WD=0 does not specify the AdamW embedding/1D groups' WD.

\begin{table}[!htbp]
    \renewcommand{\arraystretch}{1.20}
    \setlength{\tabcolsep}{9pt}

    \centering
    \small

    \caption{
      Joint learning-rate and weight-decay sweep for Muon on the
      240M QWEN3-like model, with \ourname{} variants as references.
    }
    \label{tab:muon_lr_wd_coupling}

    \begin{threeparttable}
      \begin{tabular}{lllcccc}
        \toprule
        \multirow{2}{*}{\textbf{Optimizer}}
        & \multirow{2}{*}{\textbf{Learning Rate}}
        & \multirow{2}{*}{\textbf{Metric}}
        & \multicolumn{4}{c}{\textbf{Weight Decay}} \\
        \cmidrule(lr){4-7}
        & & & 0.01 & 0.05 & 0.10 & 0.50 \\
        \midrule

        \multirow{6}{*}{Muon}
        & \multirow{2}{*}{0.008}
        & Train
        & 3.0731 & 3.0579 & \textbf{3.0499} & 3.0679 \\
        & & Val PPL
        & 21.348 & 21.020 & \textbf{20.845} & 21.168 \\
        \cmidrule(lr){2-7}

        & \multirow{2}{*}{0.016}
        & Train
        & 3.0553 & \textbf{3.0420} & 3.0427 & 3.1022 \\
        & & Val PPL
        & 20.972 & 20.680 & \textbf{20.676} & 21.874 \\
        \cmidrule(lr){2-7}

        & \multirow{2}{*}{0.032}
        & Train
        & 3.0420 & \textbf{3.0382} & 3.0524 & 3.1686 \\
        & & Val PPL
        & 20.691 & \textbf{20.583} & 20.848 & 23.345 \\
        \midrule

        \multirow{2}{*}{\ourname{}-fro}
        & \multirow{2}{*}{0.016}
        & Train
        & \multicolumn{4}{c}{3.0347 (weight decay = 0.0)} \\
        & & Val PPL
        & \multicolumn{4}{c}{20.579 (weight decay = 0.0)} \\
        \midrule

        \multirow{2}{*}{\ourname{}-spec}
        & \multirow{2}{*}{0.008}
        & Train
        & \multicolumn{4}{c}{\textbf{3.0308} (weight decay = 0.0)} \\
        & & Val PPL
        & \multicolumn{4}{c}{\textbf{20.502} (weight decay = 0.0)} \\
        \bottomrule
      \end{tabular}

      \begin{tablenotes}[flushleft]
        \scriptsize
        \item[]
        The \ourname{}-fro and \ourname{}-spec reference results use
        weight decay $0.0$. Their result cells are merged because weight
        decay is not swept for these configurations.
      \end{tablenotes}
    \end{threeparttable}
  \end{table}

\paragraph{$\mu$P Transfer.} The locked relative learning rate also enables zero-shot $\mu$P transfer across model widths. To make the experiment compatible with $\mu$P~\citep{yang2023spectral}, we initialize the linear-layer weights from $\mathcal{N}(0, \sigma^2)$ with $\sigma = \min\!\left\{\sqrt{\dout/\din},\, 1\right\}/\sqrt{\din}$ and set the Frobenius radius to $R = r\,\min\!\left\{\sqrt{\dout/\din},\, 1\right\}\sqrt{\dout}$, which absorbs the architecture-dependent $\Theta(1)$ constant into the tunable hyperparameter $r$ (see Section~\ref{sec:activation_control}). For Spectral Sphere $\mathcal{M}_R$ the radius is same with the radii in Section~\ref{sec:activation_control}. Holding the rest of the 330M training configuration of Table~\ref{tab:legacy_dynamics_setup} fixed, we vary the model width across $\{256,\, 512,\, 1024,\, 2048\}$ and the learning rate across $\{0.005,\, 0.007,\, 0.009,\, 0.01,\, 0.02,\, 0.03\}$. Figure~\ref{fig:muP} reports the resulting validation loss for both \ourname{}-spec and \ourname{}-fro. Within each width, the validation-loss curve is nearly flat across the learning-rate sweep, and the minimum sits at the same learning rate ($\eta_t \approx 0.01$ for \ourname{}-spec and $\eta_t \approx 0.02$ for \ourname{}-fro) across all four widths. The locked $\eta_{\text{rel},t}$ therefore delivers zero-shot $\mu$P transfer in practice.

\begin{figure}[htbp]
    \centering
    \begin{subfigure}{0.48\textwidth}
        \centering
        \includegraphics[width=\linewidth]{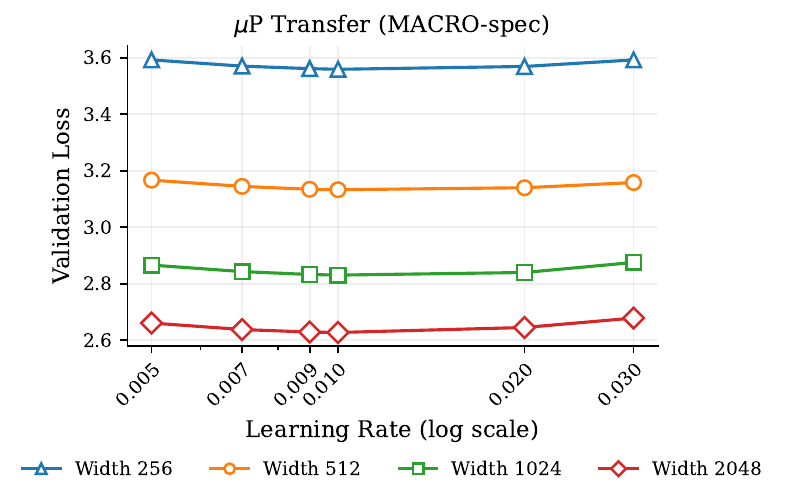}
        \caption{\ourname{}-spec.}
        \label{fig:muP_spec}
    \end{subfigure}\hfill
    \begin{subfigure}{0.48\textwidth}
        \centering
        \includegraphics[width=\linewidth]{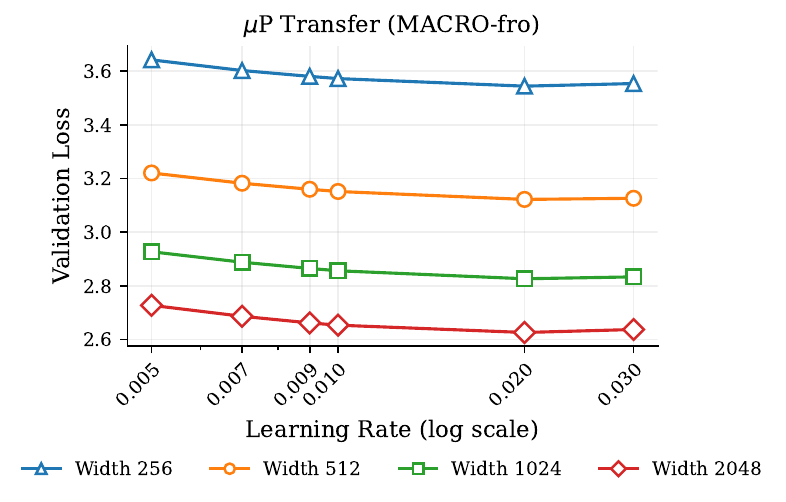}
        \caption{\ourname{}-fro.}
        \label{fig:muP_fro}
    \end{subfigure}
    \caption{Validation loss versus learning rate at four model widths $\{256,\, 512,\, 1024,\, 2048\}$ for the 330M QWEN3-like configuration of Table~\ref{tab:legacy_dynamics_setup}. Each width is trained with the same per-width radius $r$ and the same alignment constant $c$. The optimum learning rate is consistent across widths, which empirically verifies zero-shot $\mu$P transfer for both \ourname{}-spec and \ourname{}-fro.}
    \label{fig:muP}
\end{figure}

\paragraph{Learning-Rate Warmup.} Table~\ref{tab:warmup_constraint_comparison} compares Muon and the two \ourname{} variants using warmup ratios $\{0\%,\,5\%,\,10\%,\,15\%\}$ at 240M and $\{0\%,\,5\%,\,15\%\}$ at 600M, reporting training loss averaged over the last 5\% of training and final validation perplexity. Both 240M model and 600M model follow the experimental setting in Table~\ref{tab:section5_setup}. Across the tested 240M and 600M configurations, both \ourname{} variants achieve their lowest validation perplexity without warmup. Muon benefits from 5\% warmup at 240M; at 600M, its 0\% and 5\% results are close (validation PPL 15.312 and 15.309).
% TODO[WARMUP-SETTINGS]: Confirm the peak LR and radius for each optimizer/model, the cosine schedule and how it changes with the warmup ratio, total steps/token budgets, seeds, and affected parameter groups (constrained matrices versus AdamW embedding/1D parameters). Keep the unreported 600M 10-percent entries missing.

\begin{table}[H]
    \renewcommand{\arraystretch}{1.20}
    \setlength{\tabcolsep}{7pt}

    \centering
    \small

    \caption{
      Effect of warmup ratio on the last-5\% training loss and final
      validation perplexity for unconstrained Muon and constrained
      \ourname{} variants.
    }
    \label{tab:warmup_constraint_comparison}

    \begin{threeparttable}
      \begin{tabular}{llcccc}
        \toprule
        \multirow{2}{*}{\textbf{Optimizer}}
        & \multirow{2}{*}{\textbf{Warmup}}
        & \multicolumn{2}{c}{\textbf{240M}}
        & \multicolumn{2}{c}{\textbf{600M}} \\
        \cmidrule(lr){3-4}
        \cmidrule(lr){5-6}
        & & Train & Val PPL & Train & Val PPL \\
        \midrule

        \multirow{4}{*}{Muon}
        & 0\%  & 3.0484 & 20.868 & 2.7418 & 15.312 \\
        & \textbf{5\%}
        & \textbf{3.0379}
        & \textbf{20.650}
        & \textbf{2.7416}
        & \textbf{15.309} \\
        & 10\% & 3.0396 & 20.679 & -- & -- \\
        & 15\% & 3.0413 & 20.712 & 2.7479 & 15.402 \\
        \midrule

        \rowcolor{black!5}
        & \textbf{0\%}
        & \textbf{3.0328}
        & \textbf{20.543}
        & \textbf{2.7347}
        & \textbf{15.227} \\
        \rowcolor{black!5}
        & 5\%  & 3.0336 & 20.560 & 2.7363 & 15.252 \\
        \rowcolor{black!5}
        & 10\% & 3.0357 & 20.604 & -- & -- \\
        \rowcolor{black!5}
        \multirow{-4}{*}{\ourname{}-fro (Ours)}
        & 15\% & 3.0393 & 20.674 & 2.7421 & 15.340 \\
        \midrule

        & \textbf{0\%}
        & \textbf{3.0334}
        & \textbf{20.559}
        & \textbf{2.7343}
        & \textbf{15.216} \\
        & 5\%  & 3.0370 & 20.632 & 2.7366 & 15.253 \\
        & 10\% & 3.0404 & 20.700 & -- & -- \\
        \multirow{-4}{*}{\ourname{}-spec (Ours)}
        & 15\% & 3.0430 & 20.752 & 2.7435 & 15.354 \\
        \bottomrule
      \end{tabular}

    \end{threeparttable}
  \end{table}

\subsection{Numerical Experiments for Section~\ref{sec:training_precision}}
\label{app:precision_experiments}

\paragraph{Experimental setup.}
We compare Muon, \ourname{}-fro, and \ourname{}-spec on the
240M and 600M QWEN3-like models (model setting is in \ref{tab:section5_setup}) using
TorchTitan~\citep{liang2024torchtitan}.
The precision comparison changes \texttt{training.dtype}
from \texttt{float32} to \texttt{bfloat16}.
The BF16 setting reduces the storage precision of parameters,
gradients, and optimizer states without an additional FP32
master-weight copy; it does not require every intermediate
operation to use BF16.
We retain the implementation's existing higher-precision operations
and add no stochastic rounding, compensated accumulation, or
error-feedback buffers.
All runs use seed 1. The comparisons use the same learning-rate schedule in both
precisions for each optimizer, with minimum lr decay factor $f=0.01$.

\paragraph{Performance with BF16 master weight training.}
Table~\ref{tab:low_precision_performance} preserves the full
training-loss and validation-perplexity comparison. Training loss is averaged over the last $5\%$ of each run;
validation perplexity is computed from the final validation loss.
The training-gap column retains the convention
$\Delta_{\mathrm{Train}}=L_{\mathrm{train}}^{\mathrm{FP32}}
-L_{\mathrm{train}}^{\mathrm{BF16}}$.

\begin{table*}[htbp]
    \renewcommand{\arraystretch}{1.20}
    \setlength{\tabcolsep}{8pt}
    \centering
    \small
    \begin{threeparttable}
      \caption{Performance under the FP32 and BF16 training
      configurations, with the BF16 learning-rate floors specified
      below. Training loss is averaged over the last $5\%$ of training.}
      \label{tab:low_precision_performance}
      \begin{tabular}{llccccc}
        \toprule
        \textbf{Model}
        & \textbf{Optimizer}
        & \textbf{FP32 Train}
        & \textbf{BF16 Train}
        & $\boldsymbol{\Delta_{\text{Train}}}$
        & \textbf{FP32 PPL}
        & \textbf{BF16 PPL} \\
        \midrule
        \multirow{3}{*}{240M}
        & Muon
        & 3.0373 & 4.1813 & $-1.1440$ & 20.593 & 64.845 \\
        & \ourname{}-fro
        & 3.0311 & 3.0435 & $-0.0124$ & 20.494 & 20.748 \\
        & \ourname{}-spec
        & 3.0432 & 3.0706 & $-0.0274$ & 20.732 & 21.324 \\
        \midrule
        \multirow{3}{*}{600M}
        & Muon
        & 2.7446 & 3.6801 & $-0.9356$ & 15.349 & 39.269 \\
        & \ourname{}-fro
        & 2.7392 & 2.7611 & $-0.0219$ & 15.296 & 15.665 \\
        & \ourname{}-spec
        & 2.7405 & 2.7859 & $-0.0453$ & 15.310 & 16.056 \\
        \bottomrule
      \end{tabular}
      % \begin{tablenotes}[flushleft]
      %   \scriptsize
      %   \linespread{0.9}\selectfont
      %   \item[]
      %   $\Delta_{\mathrm{Train}}=L_{\mathrm{train}}^{\mathrm{FP32}}
      %   -L_{\mathrm{train}}^{\mathrm{BF16}}$.
      %   A more negative value indicates a larger BF16 degradation.
      %   Differences are computed before display rounding.
      %   \item[]
      %   All experiments use seed 1. FP32 uses
      %   \texttt{min\_lr\_factor}$=0.01$.
      %   For the 240M model, the reported BF16 results for
      %   \ourname{}-fro and \ourname{}-spec use their best observed
      %   \texttt{min\_lr\_factor}$=0.05$.
      %   All remaining BF16 results use
      %   \texttt{min\_lr\_factor}$=0.01$.
      % \end{tablenotes}
    \end{threeparttable}
\end{table*}

\subsubsection{Update Cancellation Estimates under cosine decay LR schedule}
\label{app:bf16_schedule_comparison}
To analyze the performance gap between precisions shown in the Fig~\ref{fig:master_weight_precision_comparison}, we substitute the $\rho_t$ from
Eq.~\eqref{eq:bf16_update_scale_comparison} into the Proposition~\ref{prop:bf16_update_retention}:
\vspace{-0.03cm}
\begin{equation}
    \mathcal R_t^{\mathrm{Muon}}
    =
    \left[
    1+\left(
    \frac{\tau_0\|\bW_t\|_F}
         {\widetilde\eta_t\sqrt d}
    \right)^2
    \right]^{-1/2},
    \mathcal R_t^{\mathrm{fro}}
    =
    \left[1+\left(\frac{\tau_0}{\eta_t}\right)^2\right]^{-1/2}.
\end{equation}
\vspace{-0.03cm}
Figure~\ref{fig:bf16_weight_norm} illustrates that the parameter norm of Muon under BF16 grows to approximately $129$, far above the norm under FP32 precision.
Such growth drives $\mathcal R_t^{\mathrm{Muon}}$ up to $0.7$ when \(\widetilde\eta_t = 0.008\).
In contrast, \ourname{}-fro's retention depends directly on $\eta_t$,
without the additional weight-norm factor.
% When $\eta_t\gg\tau_0$, the canceled magnitude fraction is approximately
% $\tfrac12(\tau_0/\eta_t)^2$; when $\eta_t\ll\tau_0$, only a fraction
% approximately $\eta_t/\tau_0$ remains. For instance, we use $\mathcal R_t^{\mathrm{fro}}=1/2$ as a reference suppression
% level, reached when $\eta_t=\tau_0/\sqrt3$.
% For the 240M Frobenius cosine schedule with
% $\eta_{\max}=0.016$, $T_w=920$, $T=9200$, and lr minimum factor $f=0.01$,
% this occurs at approximately $7884$ steps, or $85.7\%$ of training.
% The matched-schedule validation-loss gap,
% $L_{\mathrm{BF16}}-L_{\mathrm{FP32}}$, grows from $0.0007$ at step
% $6800$ to $0.0073$ at step $8200$ and $0.0172$ at the end.
% Raising the min LR decay factor from $0.01$ to $0.05$ delays the model's
% half-retention point to approximately $8423$ steps ($91.6\%$),
% while the final observed gap decreases to $0.0123$.
% The schedule inversion and detailed comparisons are reported in
% Appendix~\ref{app:bf16_update_retention} and
% Appendix~\ref{app:precision_experiments}.
% These trends are consistent with the increasing cancellation during late decay, as predicted by Prop~\ref{prop:bf16_update_retention}. 

% \subsubsection{Cosine decay Learning-Rate Schedule and Suppression Markers}

Under the cosine decay LR scheduler, consider the continuous cosine schedule
\begin{equation}
\label{eq:bf16_cosine_schedule}
    \eta_t=\eta_{\max}\left[
    f+\frac{1-f}{2}\left(1+
    \cos\left(\pi\frac{t-T_w}{T-T_w}\right)\right)\right],
    \qquad T_w\leq t\leq T,
\end{equation}
where $0\leq f<1$ and $0\leq T_w<T$.
The schedule decreases from $\eta_{\max}$ to $f\eta_{\max}$.
For $\eta_\star\in(f\eta_{\max},\eta_{\max})$, solving
$\eta_{t_\star}=\eta_\star$ gives
\begin{equation}
\label{eq:bf16_scale_crossover}
    t_\star=T_w+\frac{T-T_w}{\pi}
    \arccos\left(
    \frac{2(\eta_\star/\eta_{\max}-f)}{1-f}-1\right).
\end{equation}
The numerical step estimates use this continuous interpolation
and are approximate.

For the idealized Frobenius-normalized step, $\rho_t=\eta_t$.
Two markers follow from Proposition~\ref{prop:bf16_update_retention}:
\begin{align*}
    \eta_t=\tau
    &\quad\Longrightarrow\quad
    p_{\mathrm{drop}}=\frac12,\qquad
    \mathcal{R}=\frac{1}{\sqrt2},\\
    \eta_t=\frac{\tau}{\sqrt3}
    &\quad\Longrightarrow\quad
    \mathcal{R}=\frac12.
\end{align*}
We denote the second marker by $t_{1/2}$.
More generally, a retained fraction $h\in(0,1)$ corresponds to
\[
    \mathcal{R}(\eta_t)=h
    \quad\Longleftrightarrow\quad
    \eta_t=\frac{\tau h}{\sqrt{1-h^2}}.
\]
A marker is reached during decay only if its corresponding
learning rate lies within the schedule range.
These markers specify cancellation levels in the model, not
the first step with a nonzero validation-loss gap.
The 240M numerical estimates and observed gaps are reported in
Appendix~\ref{app:bf16_schedule_comparison}.

For the 240M Frobenius comparison, set
\[
    \eta_{\max}=0.016,\qquad
    T_w=920,\qquad T=9200,\qquad
    \tau=\tau_0=2^{-9}.
\]
The threshold is the same format-based reference used in the main text,
not a parameter fitted to the loss curves.
With $f=0.01$, Eq.~\eqref{eq:bf16_scale_crossover} gives
$t_{1/2}\approx7884$, or $85.7\%$ of training.
Table~\ref{tab:bf16_retention_diagnostics} separates computed
model quantities from measured validation gaps.
The predicted retention decreases during late decay, while the
observed gap grows after being small in the middle of training.
An early precision gap also exists; the comparison concerns its
renewed late-stage growth, not its first occurrence.

\begin{table}[htbp]
    \centering
    \small
    \begin{tabular}{rccc}
        \toprule
        Step & $\eta_t$ & Model $\mathcal{R}(\eta_t)$
             & Measured $G_{\mathrm{val}}$ \\
        \midrule
        6800 & $3.223\times10^{-3}$ & $0.8552$ & $0.000710$ \\
        8200 & $7.233\times10^{-4}$ & $0.3473$ & $0.007316$ \\
        9200 & $1.600\times10^{-4}$ & $0.0816$ & $0.017157$ \\
        \bottomrule
    \end{tabular}
    \caption{Model retention estimates and observed validation-loss
    gaps for the 240M \ourname{}-fro comparison.
    Both precisions use \texttt{min\_lr\_factor}$=0.01$.
    Retention estimates use $\tau_0=2^{-9}$ and are not
    measurements of realized BF16 updates.}
    \label{tab:bf16_retention_diagnostics}
\end{table}

With the BF16 floor increased to $f=0.05$, the same model gives
$t_{1/2}\approx8423$, or $91.6\%$ of training.
Table~\ref{tab:bf16_floor_retention} compares the two BF16
floors against the same FP32 reference at $f=0.01$.
The final model retention increases from $0.0816$ to $0.3790$,
while the measured final gap decreases from $0.017157$ to $0.012312$.
The model marker is a reference cancellation level, not a predicted
validation-loss changepoint.

\begin{table}[htbp]
    \centering
    \small
    \begin{tabular}{cccc}
        \toprule
        BF16 floor $f$ & Model $t_{1/2}$
        & Model $\mathcal{R}(\eta_T)$
        & Measured final $G_{\mathrm{val}}$ \\
        \midrule
        $0.01$ & $7884$ & $0.0816$ & $0.017157$ \\
        $0.05$ & $8423$ & $0.3790$ & $0.012312$ \\
        \bottomrule
    \end{tabular}
    \caption{Learning-rate-floor comparison for 240M \ourname{}-fro.
    The FP32 reference uses $f=0.01$ in both rows.
    Model quantities use $\tau_0=2^{-9}$ and the continuous
    cosine-schedule interpolation.
    The validation gap is $G_{\mathrm{val}}
    =L_{\mathrm{val}}^{\mathrm{BF16}}
    -L_{\mathrm{val}}^{\mathrm{FP32}}$.}
    \label{tab:bf16_floor_retention}
\end{table}

\paragraph{Supplementary training-loss trajectories.}
Figure~\ref{fig:bf16_train_loss_difference} contains the training-loss
comparisons moved from the main-text performance figure.
The plotted gap uses the original sign convention
$L_{\mathrm{train}}^{\mathrm{FP32}}
-L_{\mathrm{train}}^{\mathrm{BF16}}$.
% TODO[FIGURE-CONFIG]: Record the actual min_lr_factor for every plotted
% run from the plotting inputs. Do not infer that these curves use the
% same matched or selected-floor configuration as either table above.
% This reuses the existing asset from the former lower row of Figure 6.
\begin{figure*}[htbp]
    \centering
    \includegraphics[width=\textwidth]{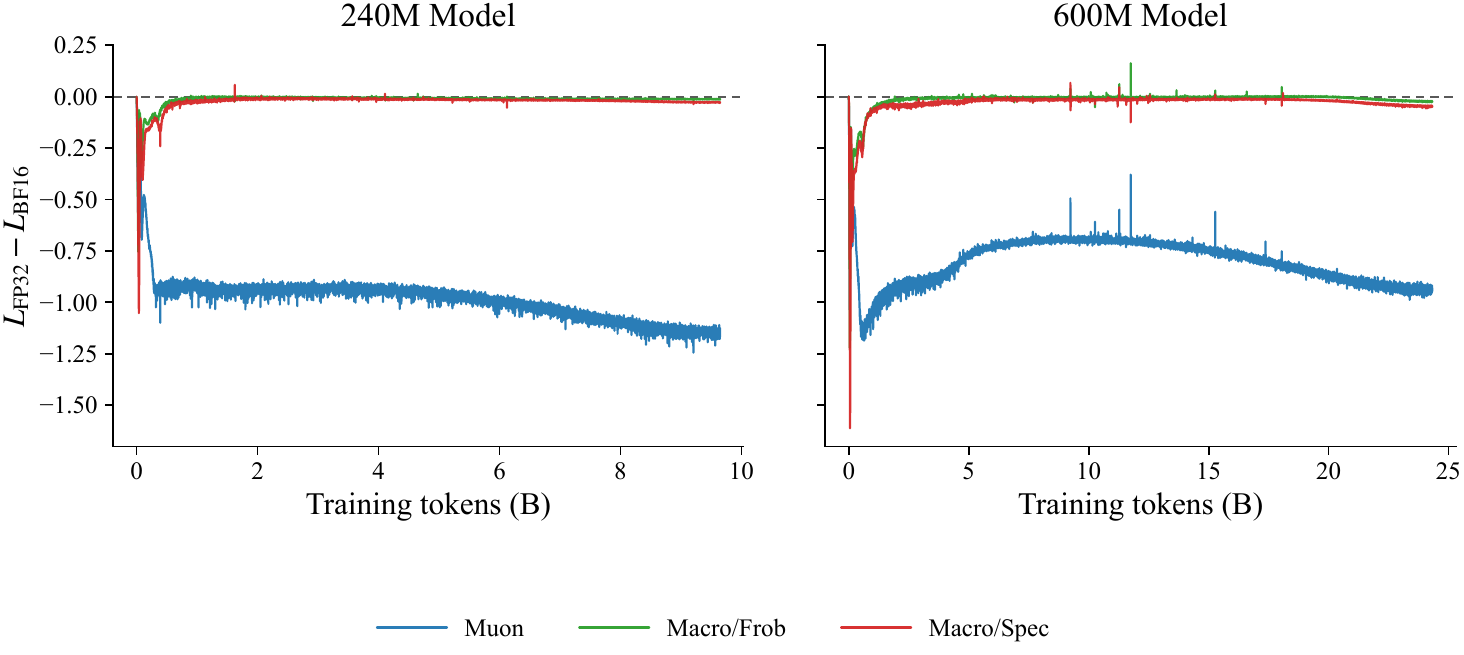}
    \caption{Training-loss differences between the FP32 and BF16
    configurations over cumulative training tokens for the 240M and
    600M models. Negative values of
    $L_{\mathrm{train}}^{\mathrm{FP32}}
    -L_{\mathrm{train}}^{\mathrm{BF16}}$
    indicate higher training loss in BF16.}
    \label{fig:bf16_train_loss_difference}
\end{figure*}

\begin{figure*}[htbp]
    \centering
    \includegraphics[width=\textwidth]{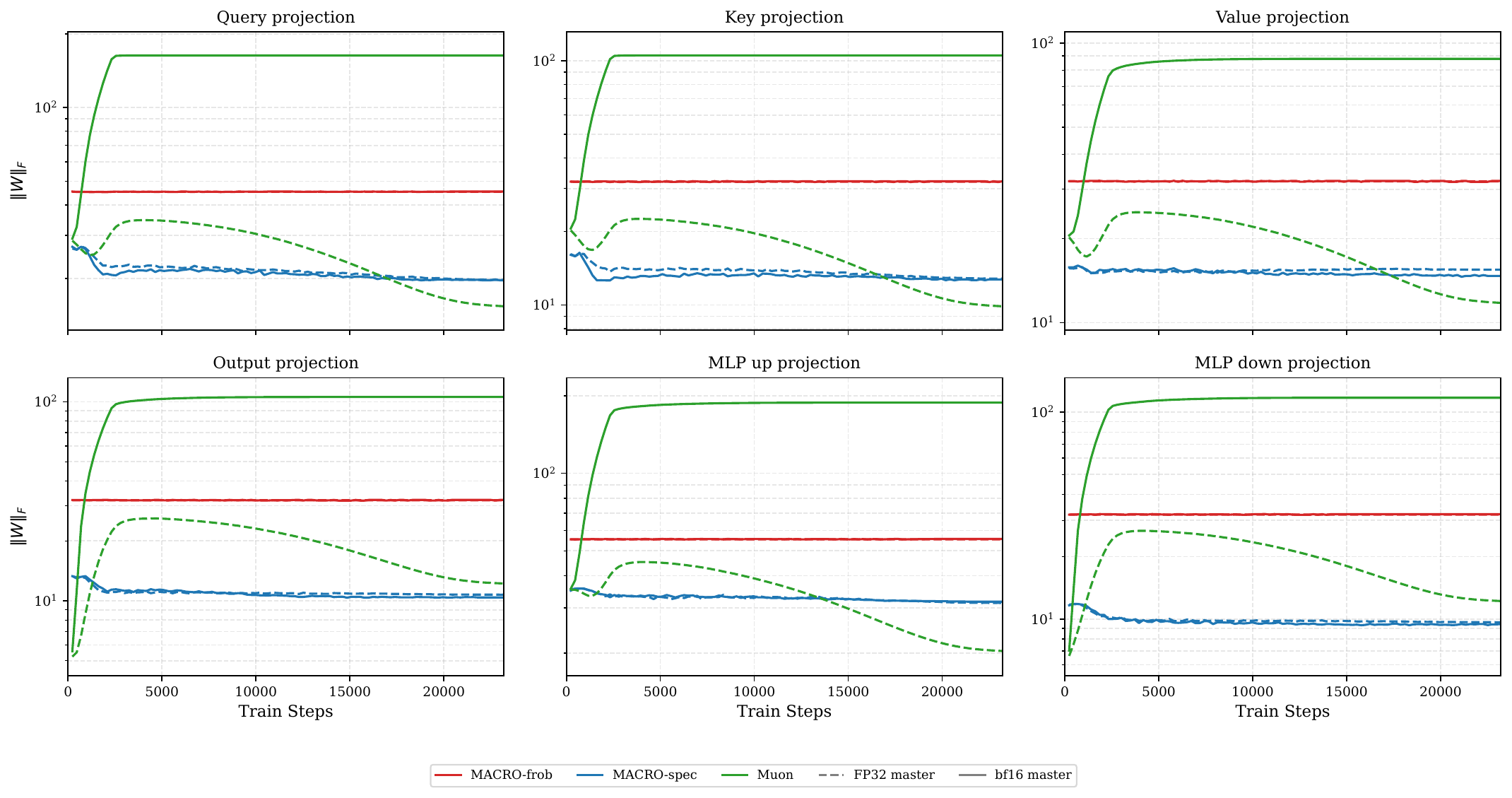}
    \caption{Parameter norm differences between the FP32 and BF16
    configurations over cumulative training steps. The norm dynamics under Muon will below up at the early stage of Muon.}
    \label{fig:bf16_weight_norm}
\end{figure*}

\subsection{Numerical Experiments Details in Section \ref{sec:justification}}
\label{app:pretraining_experiments}

\paragraph{Experimental Setup.}
Section~\ref{sec:justification} evaluates \ourname{} on QWEN3-like architectures at the 240M, 600M, and 1.2B scales using the TorchTitan~\citep{liang2024torchtitan} training pipeline. Table~\ref{tab:section5_setup} reports the resolved configurations of the three representative \ourname{}-fro runs. All three runs use C4 sequences of length 2048 and a global batch size of 512. 
% TODO[TABLE19-MODEL-DEFINITION]: Archive the exact model-definition code or resolved model configs for the three flavors below, then fill dimension, heads, kv_heads, and embedding tying. Do not infer these fields from flavor names or TorchTitan defaults.
% Exact runs: ParameterEfficient_Landing/Qwen3_240m_9.6B_512bs_benchmark/macro_frob_1-8-kangan-sps19br9; ParameterEfficient_Landing/Qwen3_0.6B_19_3B_512bs_benchmark/macro_frob_2-16-kangan-ndxngwnj; ParameterEfficient_Landing/Qwen3_1.2B_48B_512bs_lr_sweep/macro_frob_1-16-kangan-vr6czhqm.
\begin{table}[htbp]
  \centering
  \small
  \renewcommand{\arraystretch}{1.1}
  \setlength{\tabcolsep}{5pt}
  \begin{tabular}{llccc}
    \toprule
    \textbf{Category} & \textbf{Parameter} & 240M & 600M & 1.2B \\
    \midrule
    \multirow{8}{*}{\textbf{Model Config}}
    & model flavor & \texttt{\scriptsize qwen3\_120m\_split} & \texttt{\scriptsize qwen3\_0\_6b\_split} & \texttt{\scriptsize qwen3\_1\_2b\_split} \\
    & dimension & 768 & 1024 & 1536 \\
    & number of layers & 12 & 28 & 28 \\
    & heads & 12 & 16 & 12 \\
    & kv\_heads & 4 & 8 & 4 \\
    & sequence length & 2048 & 2048 & 2048 \\
    & vocabulary size & 151936 & 151936 & 151936 \\
    & Embedding Tie & True & True & True \\
    \midrule
    \multirow{11}{*}{\textbf{Train Config}}
    & training framework & TorchTitan & TorchTitan & TorchTitan \\
    & training dataset & C4 & C4 & C4 \\
    & global batch size & 512 & 512 & 512 \\
    & local batch size & 64 & 32 & 32 \\
    & warmup steps (ratio) & 920 (10\%) & 2320 (10\%) & 4600 (10\%) \\
    & minimum LR factor & 0.01 & 0.01 & 0.01 \\
    & completed steps & 9200 & 23200 & 46000 \\
    & total accelerators & 8 $\times$ H100 & 16 $\times$ H100 & 16 $\times$ H100 \\
    & tokens seen & 9.647B & 24.327B & 48.234B \\
    \bottomrule
  \end{tabular}
  \vspace{0.2cm}
  \caption{model and training configuration for the representative 240M, 600M, and 1.2B.}
  \label{tab:section5_setup}
\end{table}
\paragraph{Algorithm Efficiency Comparison Table}
To compare training efficiency across optimization algorithms and manifold constraints, we report per-GPU token throughput (TPS), per-GPU compute throughput (TFLOPS), model FLOPs utilization (MFU), and end-to-end (E2E) step time. All metrics are computed over the training steps and averaged across five random seeds. The results are summarized in the following table. Higher TPS, TFLOPS, and MFU indicate better efficiency, whereas lower E2E step time is better.  
\begin{itemize}
    \item MACRO vs. Muon: MACRO-fro shows only a small efficiency gap compared with Muon. Its TPS is 204.4K, compared with 207.7K for Muon (-1.6\%), while its MFU decreases by only 0.62 percentage points. The E2E step time increases from 0.631 s to 0.641 s (+1.6\%, or about 10 ms).
    \item Constraint Comparison: The Frobenius constraint is consistently more efficient than the spectral constraint, which is consistent with the algorithm design. The spectral sphere requires power iteration for tangent-space computation and retraction, making it more expensive than Frobenius norm computation.
    \item Algorithm Comparison: Under the Frobenius constraint, the efficiency ranking is MuonH-Frob $\approx$ MACRO-fro > FSO. Under the spectral constraint, the ranking is MuonH-Spec $\approx$ MACRO-Spec > SSO. This can be explained by the algorithm design. SSO and FSO use an inner loop to solve the tangent-space projection. With a maximum of 20 inner-loop iterations, they evaluate the Msign function up to 20 times, whereas Muon, MuonH, and MACRO evaluate it only once. Therefore, SSO and FSO are less efficient than the other manifold-constrained Muon variants under our setting.
\end{itemize}
\begin{table*}[htbp]
    \renewcommand{\arraystretch}{1.25}
    \setlength{\tabcolsep}{12pt}

    \centering
    \small

    \caption{
      Efficiency comparison across different optimizers.
      E2E denotes the end-to-end time per training step.
    }
    \label{tab:optimizer_efficiency}

    \begin{tabular}{lcccc}
      \toprule
      \textbf{Optimizer}
      & \textbf{E2E (s)} $\downarrow$
      & \textbf{TPS/GPU} $\uparrow$
      & \textbf{TFLOPS/GPU} $\uparrow$
      & \textbf{Avg.\ MFU} $\uparrow$ \\
      \midrule

      Muon
      & 0.6311
      & 207,693
      & 392.40
      & 39.68\% \\

      \rowcolor{black!5}
      \ourname{}-spec
      & 0.7124
      & 183,982
      & 347.60
      & 35.15\% \\

      \rowcolor{black!5}
      \ourname{}-fro
      & 0.6411
      & 204,440
      & 386.26
      & 39.06\% \\

      MuonH-spec
      & 0.6864
      & 190,944
      & 360.76
      & 36.48\% \\

      MuonH-fro
      & \textbf{0.6156}
      & \textbf{212,910}
      & \textbf{402.26}
      & \textbf{40.67\%} \\

      SSO
      & 1.7452
      & 75,104
      & 141.90
      & 14.35\% \\

      FSO
      & 1.2086
      & 108,454
      & 204.91
      & 20.72\% \\

      \bottomrule
    \end{tabular}
  \end{table*}
We clarify that this efficiency comparison is based on plain implementations without specialized kernel optimizations when training a 240M-parameter model. The infrastructure design using Megatron in the SSO paper could narrow the efficiency gap.

\paragraph{Hyperparameter Sweep and Selection.} We fix $\beta_1 = 0.9$ for every optimizer. For the SSO and FSO double-loop solvers, we set the subproblem tolerance to $10^{-3}$ and cap the inner-loop iterations at $8$. For Muon, \ourname{}-spec, \ourname{}-fro, MuonH-spec, and MuonH-fro, we sweep the learning rate $\eta_t$ over $\{2.0\times 10^{-3},\, 4.0\times 10^{-3},\, 8.0\times 10^{-3},\, 1.6\times 10^{-2},\, 3.2\times 10^{-2}\}$. 
% For the four manifold-constrained variants, we additionally sweep the constraint radius $r$ over $\{0.5,\, 1.0,\, 2.0,\, 5.0,\, 10.0\}$. For \ourname{}-spec and MuonH-spec, we also sweep the qkv-split flag following~\cite{xie2026controlledllmtrainingspectral}. The double-loop algorithms SSO and FSO are sensitive to compute cost, so we sweep only the learning rate and inherit the radius and qkv-split from the corresponding MuonH variant. Table~\ref{tab:section5_sweep} summarizes the search ranges, and 
Table~\ref{tab:section5_chosen_hp} reports the chosen hyperparameters that produce the bold entries of Table~\ref{tab:optimizer_comparison} in the main text.

\begin{table}[htbp]
  \centering
  \renewcommand{\arraystretch}{1.2}
  \setlength{\tabcolsep}{10pt}
  \begin{tabular}{llccc}
    \toprule
    \textbf{Model} &\textbf{Optimizer} & \textbf{learning rate} & \textbf{weight decay} & \textbf{r} \\
    \midrule
    \multirow{7}{*}{\textbf{240M}} 
    & Muon & 0.016 & 0.1 & - \\
    & MuonH-spec & 0.008 & 0.0 & 2.0 \\
    & MuonH-fro & 0.016 & 0.0 & 1.0 \\
    & SSO & 0.008 & 0.0 & 2.0 \\
    & FSO & 0.016 & 0.0 & 1.0 \\
    & \ourname{}-spec & 0.008 & 0.0 & 2.0 \\
    & \ourname{}-fro & 0.016 & 0.0 & 1.0 \\
    
    \midrule
    \multirow{5}{*}{\textbf{600M}} 
    & Muon & 0.008 & 0.1 & - \\
    & MuonH-spec & 0.008 & 0.0 & 2.0 \\
    & MuonH-fro & 0.008 & 0.0 & 1.0  \\
    & SSO & 0.008 & 0.0 & 2.0  \\
    & FSO & 0.008 & 0.0 & 2.0 \\
    & \ourname{}-spec & 0.008 & 0.0 & 2.0 \\
    & \ourname{}-fro & 0.008 & 0.0 & 1.0 \\
    \midrule
    \multirow{5}{*}{\textbf{1.2B}} 
    & Muon & 0.008 & 0.1 & - \\
    & MuonH-spec & 0.004 & 0.0 & 2.0 \\
    & MuonH-fro & 0.008 & 0.0 & 1.0  \\
    & \ourname{}-spec & 0.004 & 0.0 & 2.0  \\
    & \ourname{}-fro & 0.008 & 0.0 & 1.0 \\

    \bottomrule
  \end{tabular}
  \vspace{0.2cm}
  \caption{Selected hyperparameters per optimizer per model scale, used to produce the bold entries of Table~\ref{tab:optimizer_comparison} in the main text.}
  \label{tab:section5_chosen_hp}
\end{table}

\paragraph{Per-Step Training and Validation Loss Curves.} Figure~\ref{fig:section5_loss_curves} plots the training and validation loss across the full training schedule for each model scale. The curves give the per-step trajectory behind the final numbers reported in Table~\ref{tab:optimizer_comparison}. All manifold-constrained optimizers (\ourname{}-spec, \ourname{}-fro, MuonH-spec, MuonH-fro, SSO, FSO) track the unconstrained Muon baseline closely throughout training, and they are superior than Muon in the late stage of training. Across the three scales, \ourname{}-spec and \ourname{}-fro reach validation losses comparable to the heuristic MuonH counterparts and to the exact double-loop SSO and FSO baselines.

\begin{figure}[htbp]
  \centering

  \begin{subfigure}{1.0\textwidth}
    \centering
    \includegraphics[width=\linewidth]{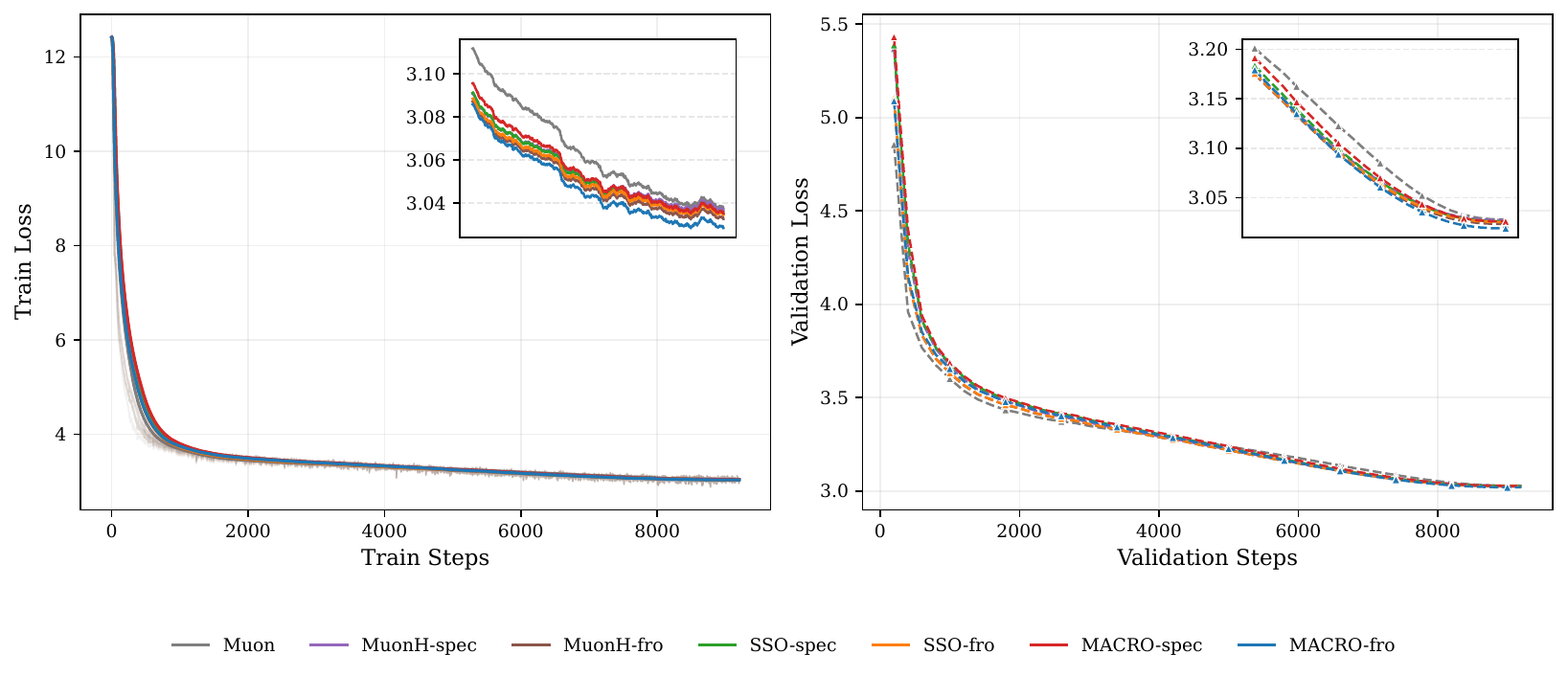}
    \caption{Train and Validation Loss for 240M training.}
    \label{fig:loss_120M}
  \end{subfigure}
  \\[2ex]

  \begin{subfigure}{1.0\textwidth}
    \centering
    \includegraphics[width=\linewidth]{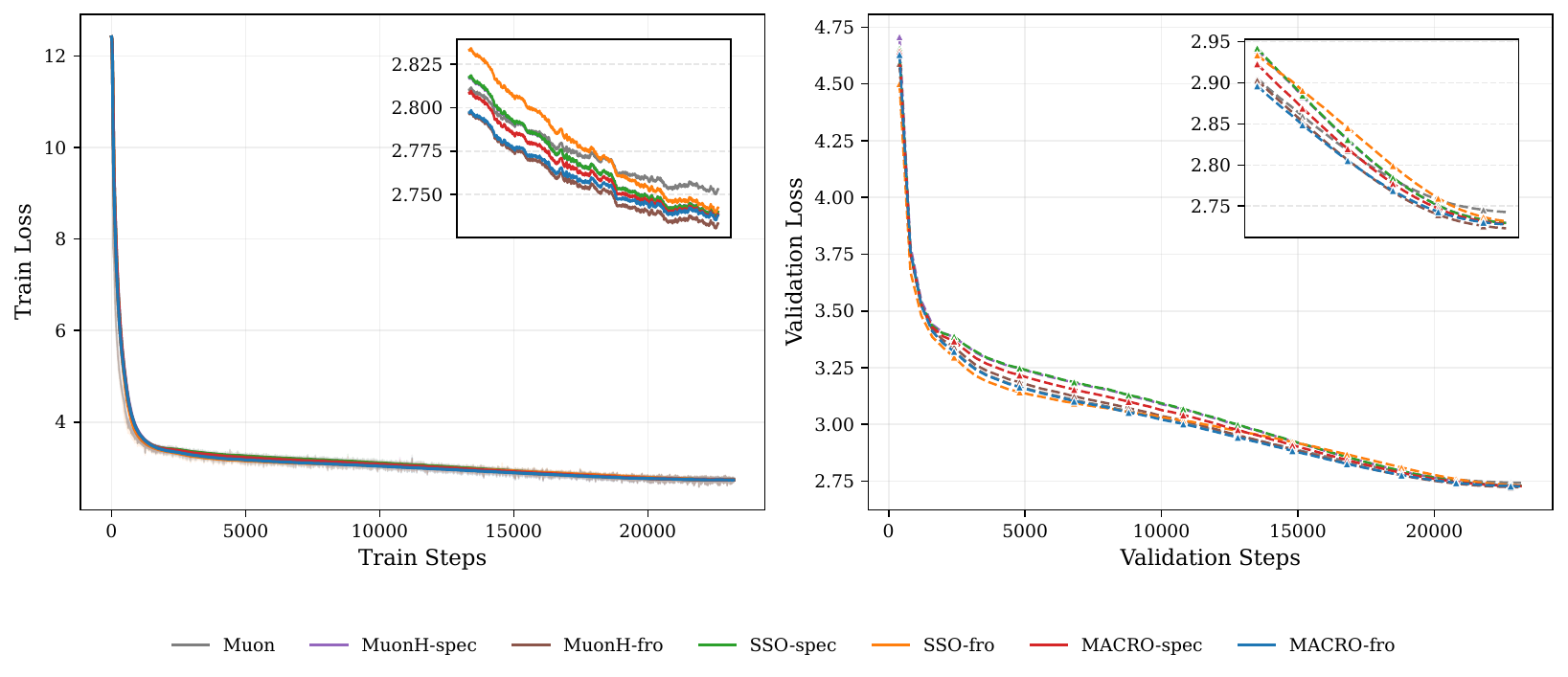}
    \caption{Train and Validation Loss for 600M training.}
    \label{fig:loss_330M}
  \end{subfigure}
  \\[2ex]

  \begin{subfigure}{1.0\textwidth}
    \centering
    \includegraphics[width=\linewidth]{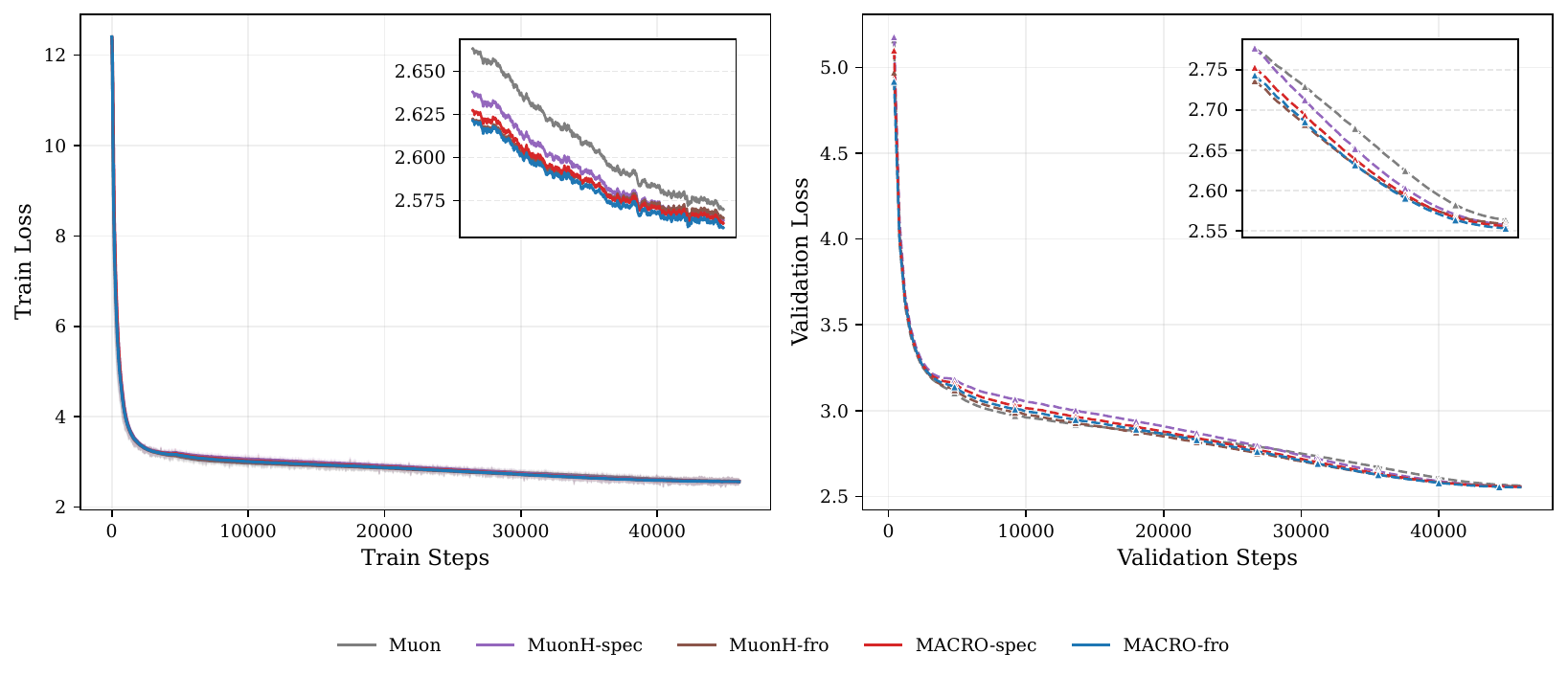}
    \caption{Train and Validation Loss for 1.2B training. SSO and FSO are omitted at this scale because of their double-loop computational cost (see Table~\ref{tab:optimizer_comparison}).}
    \label{fig:loss_1B}
  \end{subfigure}

  \caption{Training (left) and validation (right) loss across the full training schedule for the 120M, 330M, and 1B QWEN3-like models. Each subfigure compares Muon, MuonH-spec, MuonH-fro, SSO, FSO, \ourname{}-spec, and \ourname{}-fro under the chosen hyperparameters from Table~\ref{tab:section5_chosen_hp}. These curves are the per-step trajectories underlying Table~\ref{tab:optimizer_comparison}.}
  \label{fig:section5_loss_curves}
\end{figure}

\paragraph{Tangent Space Violation Ablation.} \ourname{} replaces the inner bisection of SSO and FSO with a single-step tangent space projection (Line~\ref{alg:macro_tangent_projection} of Algorithm~\ref{alg:macro}). To quantify the resulting trade-off, we measure the per-step tangent space violation, defined as the residual norm $\mathrm{vio}(\bW_t) := |\langle \Gamma_t, \widetilde{\nabla}_t\rangle|$ between the proposed update direction $\widetilde{\nabla}_t$ and the normal vector to the constraint surface $\Gamma_t$ defined in Appendix \ref{app:projection_retraction}. Figure~\ref{fig:section5_violation} reports the average of $\mathrm{vio}(\bW_t)$ across the full 8500-step training schedule of the 330M model, separated by linear-layer type (FFN-Up, Attn-Out, Attn-Value, Attn-Query, Attn-Key) and by manifold ($\cM_S$ on the left, $\cM_F$ on the right). Solid curves report \ourname{}; dashed curves report the corresponding exact double-loop solver (SSO for $\cM_S$, FSO for $\cM_F$). Two conclusions follow.

First, \ourname{} keeps the absolute tangent space violation small. The solid curves stabilize at roughly $10^{-3}$ to $10^{-2}$ across all five layer types and both manifolds. The single-step projection therefore lands near, but not exactly on, the tangent space. The double-loop dashed curves stabilize at roughly $10^{-4}$, one to two orders of magnitude lower. Second, this larger residual does not translate into worse validation loss. Table~\ref{tab:optimizer_comparison} shows that \ourname{}-spec and \ourname{}-fro reach validation losses comparable to SSO and FSO on 120M model and 330M model. 
% Moreover, \ourname{} reach better validation losses on 330M and 1B models. 
Removing the inner bisection therefore eliminates a large per-step computation overhead at no measurable cost in validation loss, and \ourname{} attains a favorable balance between tangent space residual and per-step computation.

\begin{figure}[htbp]
  \centering
  \includegraphics[width=\linewidth]{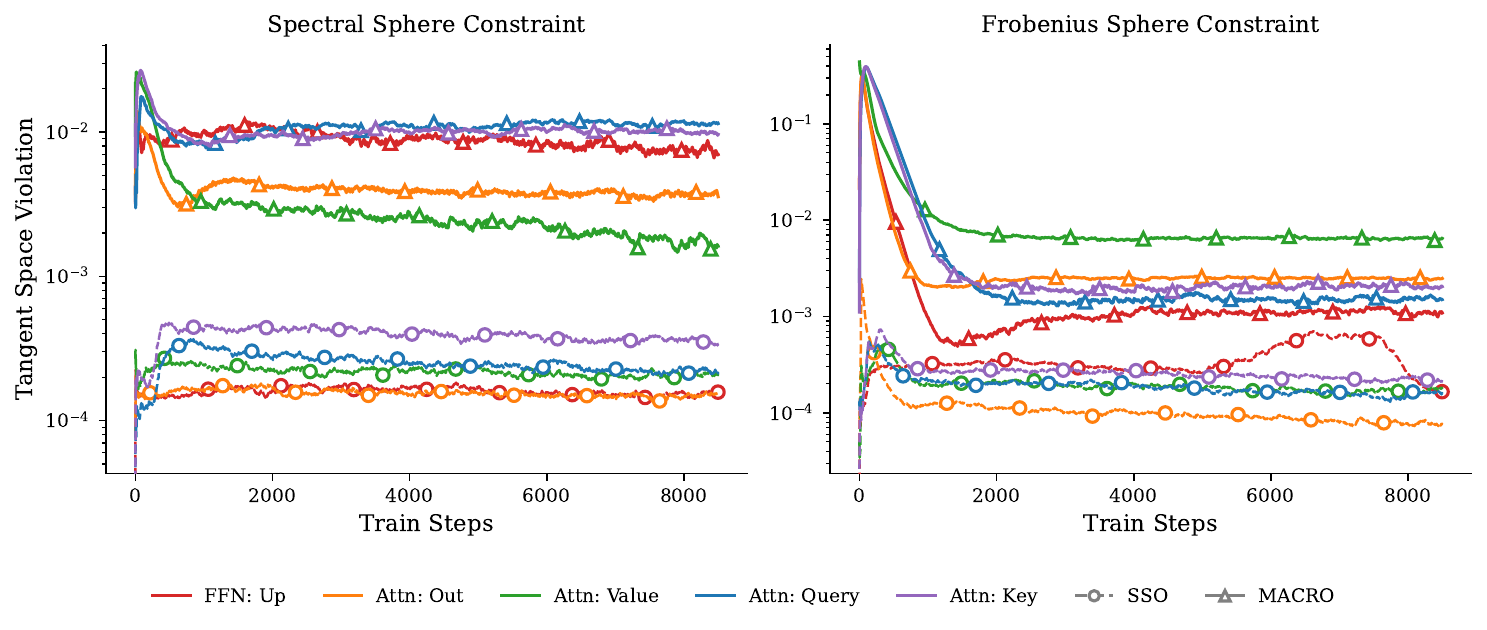}
  \caption{The average of Per-step tangent space violation $\mathrm{vio}(\bW_t) := |\langle \Gamma_t, \widetilde{\nabla}_t\rangle|$, where $\widetilde{\nabla}_t$ is the update direction and $\Gamma_t$ is the normal vector to the constraint surface, across the 8500-step training schedule of the 330M QWEN3-like model. The left panel reports the spectral sphere ($\cM = \cM_S$); the right panel reports the Frobenius sphere ($\cM = \cM_F$). Each colored curve corresponds to a linear-layer type (FFN-Up, Attn-Out, Attn-Value, Attn-Query, Attn-Key); solid curves report \ourname{} (single-step projection) and dashed curves report the corresponding exact double-loop solver (SSO for $\cM_S$, FSO for $\cM_F$). Across all five layer types and both manifolds, \ourname{} stabilizes at residuals on the order of $10^{-3}$ to $10^{-2}$, one to two orders of magnitude above the double-loop baseline. Table~\ref{tab:optimizer_comparison} shows that this larger residual does not degrade validation loss.}
  \label{fig:section5_violation}
\end{figure}

\subsection{Empirical Equivalence of Oblique and Frobenius Constraints in Section~\ref{sec:background}}
\label{app:oblique_redundancy}

In Section~\ref{sec:background}, we introduced Oblique manifolds as candidates for row-wise or column-wise constraints. In this appendix, we provide the detailed empirical justification for omitting these Oblique constraints in favor of the Frobenius Sphere.
\paragraph{Generalization Performance Advantage.}We first compare the generalization performance of the global Frobenius constraint ($\mathcal{M}_F$) against the Input and Output Oblique constraints ($\mathcal{M}_{O_{\textrm{in}}}$ and $\mathcal{M}_{O_{\textrm{out}}}$). Figure~\ref{fig:frobvsbblique_ablation} (a) plots the validation loss across various learning rates. The results show that the Frobenius constraint consistently achieves lower validation loss than both Oblique constraints. This clear performance gap provides the primary motivation for adopting the Frobenius sphere.
\paragraph{Equivalence of Internal Training Dynamics.}To understand why Oblique constraints do not improve performance, we analyze the internal optimization dynamics. We specifically compare the global Frobenius constraint against the Input Oblique constraint. Figure~\ref{fig:frobvsbblique_ablation} (b) and Figure~\ref{fig:frobvsbblique_ablation} (c) track the spectral norm and the maximum row norm ($2 \to \infty$ norm), respectively. For most learning rates, the trajectories under the Frobenius constraint closely match the trajectories under the explicitly normalized Input Oblique constraint. This matching behavior indicates that a Frobenius naturally regulates individual feature dimensions. The neural network implicitly maintains a uniform norm distribution without requiring strict row-wise or column-wise limitations.

\paragraph{Adaptive Capacity at Extreme Learning Rates.} While the dynamics are generally similar, Figure Figure~\ref{fig:frobvsbblique_ablation} (d) illustrates the maximum column norm ($1 \to \infty$ norm) dynamics. At lower and standard learning rates, both constraints produce similar flat trajectories. However, at the highest learning rate ($\eta = 3 \times 10^{-2}$), the trajectories separate. The Input Oblique constraint strictly limits the column norm, keeping the trajectory flat. In contrast, the Frobenius constraint allows this specific norm to adaptively increase. This flexibility suggests that the Frobenius sphere provides a better adaptive balance. It regularizes the overall weight matrix but allows specific feature dimensions to adapt when necessary. Because Oblique constraints offer no performance benefits and restrict this adaptive capacity, we omit them from our core theoretical and empirical analyses.

\begin{figure}[htbp]
    \centering
    \begin{minipage}{0.48\textwidth}
        \centering
        \includegraphics[width=\linewidth]{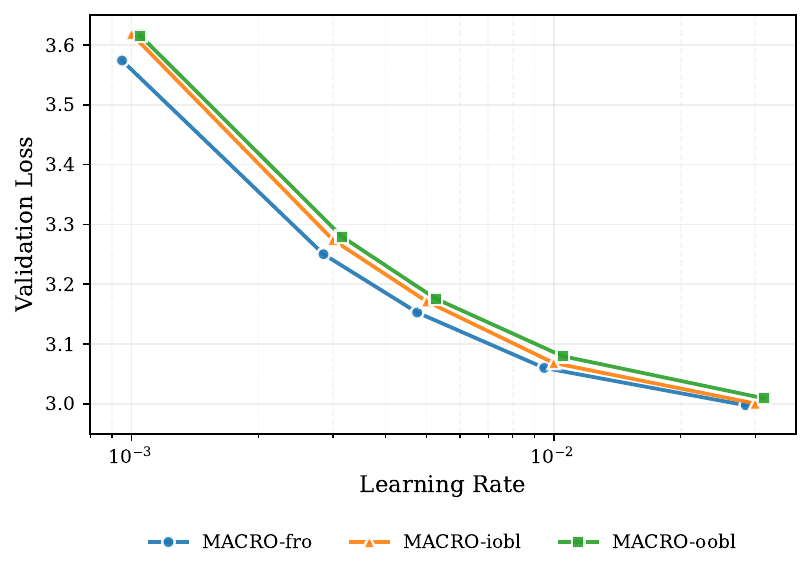}
        {\small(a) \textbf{Validation loss across different learning rates}. We compare the performance of MACRO-fro, MACRO-iobl, and MACRO-oobl. The x-axis represents the learning rate on a logarithmic scale. The global Frobenius constraint (MACRO-fro) consistently achieves the lowest validation loss across all evaluated learning rates compared to the Oblique constraints.}
    \end{minipage}\hfill
    \begin{minipage}{0.48\textwidth}
        \centering
        \includegraphics[width=\linewidth]{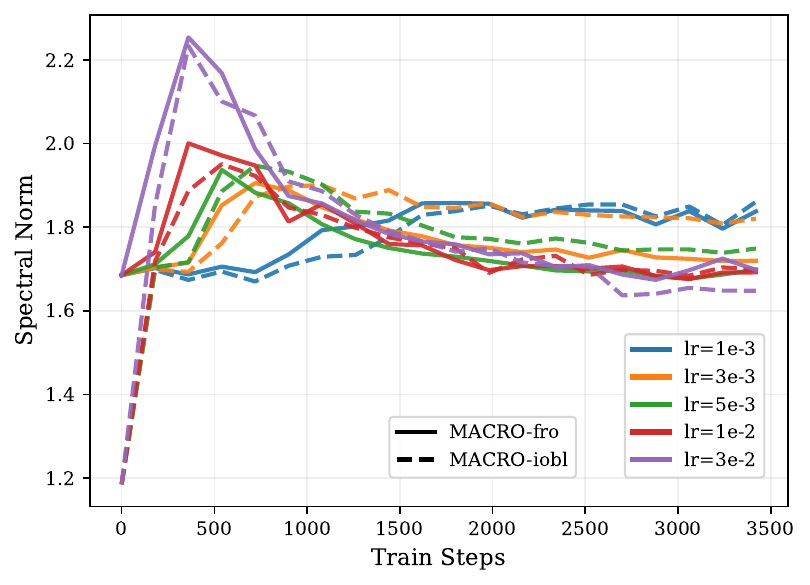}
        {\small (b) \textbf{Spectral norm dynamics during training.} We track the spectral norm for MACRO-fro (solid lines) and MACRO-iobl (dashed lines) across five different learning rates. Despite the architectural differences between the constraints, the spectral norm trajectories remain highly similar throughout the training process.}
    \end{minipage}
    
    \vspace{0.5cm}
    
    \begin{minipage}{0.48\textwidth}
        \centering
        \includegraphics[width=\linewidth]{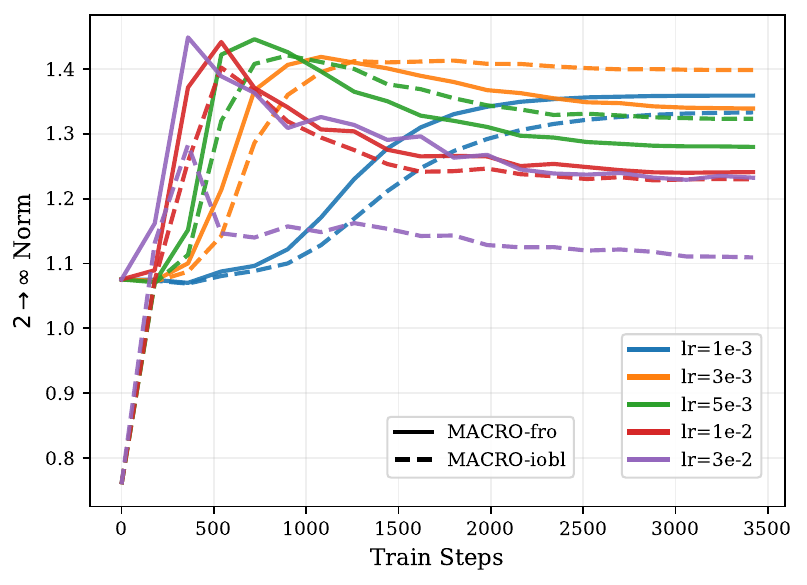}
        {\small (c) \textbf{Maximum $2 \to \infty$ Norm dynamics.} We compare the $2 \to \infty$ norm under the MACRO-fro and MACRO-iobl constraints. Similar to the spectral norm behavior, both constraints produce closely matching trajectories across the entire training period.}
    \end{minipage}\hfill
    \begin{minipage}{0.48\textwidth}
        \centering
        \includegraphics[width=\linewidth]{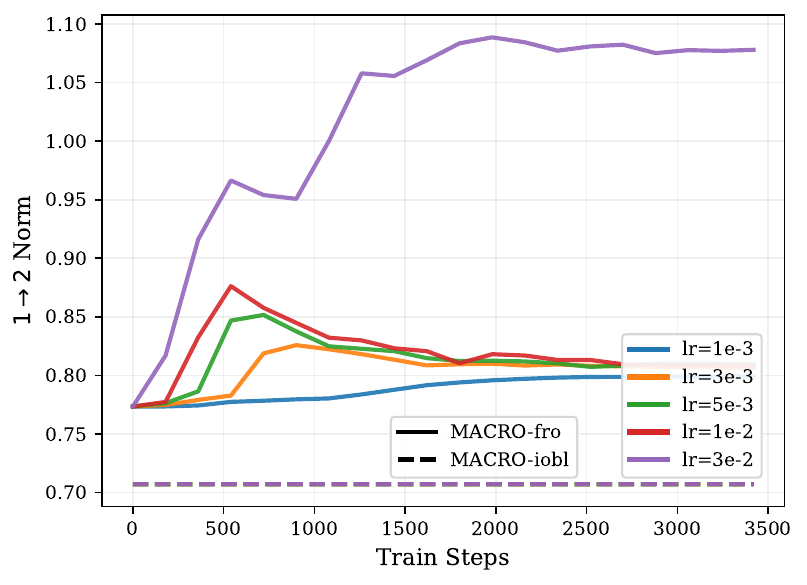}
        \vspace{0.1cm}
        {\small (d) \textbf{Maximum $1 \to 2$ Norm dynamics.} We analyze the $1 \to 2$ norm under the two constraints. For most learning rates, the trajectories match closely. However, at the learning rate $\eta = 3 \times 10^{-2}$, the Frobenius constraint allows the norm to increase, whereas the Input Oblique constraint strictly limits this growth.}
    \end{minipage}
    \vspace{0.2cm}
    \caption{Ablation study on geometric constraints.}
    \label{fig:frobvsbblique_ablation}
\end{figure}

% \clearpage
% \input{checklist.tex}

\end{document}